\documentclass[twoside,11pt]{article}
\usepackage{jmlr2e}
\usepackage[latin1]{inputenc}
\usepackage[T1]{fontenc}

\jmlrheading{14}{2013}{1091--1143}{}{}{Markus Thom and G\"{u}nther Palm}

\ShortHeadings{Sparse Activity and Sparse Connectivity in Supervised Learning}{Thom and Palm}

\usepackage{subscript}
\usepackage[ruled, linesnumbered]{algorithm2e}
\usepackage{booktabs}
\usepackage{enumitem}
\setlist[enumerate]{label=(\alph*)}
\newenvironment{proofof}[1][Proof]{\par\noindent{\bf #1\ }}{\hfill\BlackBox\\[2mm]}

\usepackage{amsmath}
\usepackage{amssymb}
\usepackage{nicefrac}
\usepackage{braket}
\usepackage{pxfonts}

\usepackage{graphicx}
\usepackage{url}
\usepackage{pslatex}

\DeclareMathOperator{\proj}{proj}
\DeclareMathOperator{\sgn}{sgn}
\DeclareMathOperator{\hada}{\circ}
\DeclareMathOperator{\hadadiff}{\circledcirc}
\DeclareMathOperator{\diag}{diag}
\DeclareMathOperator{\insphere}{in}
\DeclareMathOperator{\SOAE}{SOAE}
\DeclareMathOperator{\out}{out}
\newcommand{\tvect}[3]{\left(#1,\ #2,\ #3\right)\transp}
\newcommand{\tdvect}[2]{\tvect{#1}{\dots}{#2}}
\newcommand{\tdvectbig}[2]{\big(#1,\ \dots,\ #2\big)\transp}
\newcommand{\intervalcc}[2]{\left[#1,\ #2\right]}
\newcommand{\intervaloo}[2]{\left(#1,\ #2\right)}
\newcommand{\N}{\mathbb{N}}
\newcommand{\R}{\mathbb{R}}
\newcommand{\0}{\mathcal{O}}
\newcommand{\transp}{^T}
\newcommand{\comp}{^C}
\newcommand{\norm}[1]{\left\Vert#1\right\Vert}
\newcommand{\bnorm}[1]{\big\Vert#1\big\Vert}
\newcommand{\abs}[1]{\left\vert #1 \right\vert}
\newcommand{\splx}{\triangle}
\newcommand{\scp}[2]{\left<#1,\ #2\right>}
\newcommand{\bscp}[2]{\big<#1,\ #2\big>}
\newcommand{\discint}[2]{\{#1,\dotsc,#2\}}
\newcommand{\inint}[2]{\in\discint{#1}{#2}}
\newcommand{\charusing}[2]{\overset{\text{\footnotesize #2}}{#1}}
\newcommand{\equsing}[1]{\charusing{=}{#1}}
\newcommand{\smallsum}{\scalebox{1.2}{$\sum$}}
\renewcommand{\P}{\wp}
\renewcommand{\emptyset}{\varnothing}

\begin{document}
\title{Sparse Activity and Sparse Connectivity in Supervised Learning}

\author{\name Markus Thom \email markus.thom@uni-ulm.de\\
       \addr driveU / Institute of Measurement, Control and Microtechnology\\
       Ulm University\\
       89081 Ulm, Germany
       \AND
       \name G\"{u}nther Palm \email guenther.palm@uni-ulm.de\\
       \addr Institute of Neural Information Processing\\
       Ulm University\\
       89081 Ulm, Germany}

\editor{Aapo Hyv\"{a}rinen}

\maketitle

\begin{abstract}%
Sparseness is a useful regularizer for learning in a wide range of applications, in particular in neural networks.
This paper proposes a model targeted at classification tasks, where sparse activity and sparse connectivity are used to enhance classification capabilities.
The tool for achieving this is a sparseness-enforcing projection operator which finds the closest vector with a pre-defined sparseness for any given vector.
In the theoretical part of this paper, a comprehensive theory for such a projection is developed.
In conclusion, it is shown that the projection is differentiable almost everywhere and can thus be implemented as a smooth neuronal transfer function.
The entire model can hence be tuned end-to-end using gradient-based methods.
Experiments on the MNIST database of handwritten digits show that classification performance can be boosted by sparse activity or sparse connectivity.
With a combination of both, performance can be significantly better compared to classical non-sparse approaches.
\end{abstract}


\section{Introduction}
\label{sect:introduction}
Sparseness is a concept of efficiency in neural networks, and exists in two variants in that context \citep{Laughlin2003}.
The \emph{sparse activity} property means that only a small fraction of neurons is active at any time.
The \emph{sparse connectivity} property means that each neuron is connected to only a limited number of other neurons.
Both properties have been observed in mammalian brains \citep{Hubel1959,Olshausen2004,Mason1991,Markram1997} and have inspired a variety of machine learning algorithms.
A notable result was achieved through the sparse coding model of \citet{Olshausen1996}.
Given small patches from images of natural scenes, the model is able to produce Gabor-like filters, resembling properties of simple cells found in mammalian primary visual cortex \citep{Hubel1959,Vinje2000}.
Another example is the optimal brain damage method of \citet{LeCun1990}, which can be used to prune synaptic connections in a neural network, making connectivity sparse.
Although only a small fraction of possible connections remains after pruning, this is sufficient to achieve equivalent classification results.
Since then, numerous approaches on how to measure sparseness have been proposed, see \citet{Hurley2009} for an overview, and how to achieve sparse solutions of classical machine learning problems.

The $L_0$ pseudo-norm is a natural sparseness measure.
Its computation consists of counting the number of non-vanishing entries in a vector.
Using it rather than other sparseness measures has been shown to induce biologically more plausible properties \citep{Rehn2007}.
However, finding of optimal solutions subject to the $L_0$ pseudo-norm turns out to be NP-hard \citep{Natarajan1995,Weston2003}.
Analytical properties of this counting measure are very poor, for it is non-continuous, rendering the localization of approximate solutions difficult.
The Manhattan norm of a vector is a convex relaxation of the $L_0$ pseudo-norm \citep{Donoho2006a}, and has been employed in a vast range of applications.
This sparseness measure has the significant disadvantage of not being scale-invariant, so that an intuitive notion of sparseness cannot be derived from it.

\subsection{Hoyer's Normalized Sparseness Measure}
\label{sect:intro_hoyer}
A normalized sparseness measure $\sigma$ based on the ratio of the $L_1$ or Manhattan norm and the $L_2$ or Euclidean norm of a vector has been proposed by \citet{Hoyer2004},
\begin{displaymath}
  \sigma\colon\R^n\setminus\set{0}\to\intervalcc{0}{1}\text{,}\qquad x\mapsto\frac{\sqrt{n} - \nicefrac{\norm{x}_1}{\norm{x}_2}}{\sqrt{n}-1}\text{,}
\end{displaymath}
where higher values indicate more sparse vectors.
$\sigma$ is well-defined because $\norm{x}_2 \leq \norm{x}_1 \leq \sqrt{n}\norm{x}_2$ holds for all $x\in\R^n$ \citep{Laub2004}.
As $\sigma(\alpha x) = \sigma(x)$ for all $\alpha \neq 0$ and all $x\in\R^n\setminus\set{0}$, $\sigma$ is also scale-invariant.
As composition of differentiable functions, $\sigma$ is differentiable on its entire domain.

This sparseness measure fulfills all criteria of \citet{Hurley2009} except for Dalton's fourth law, which states that the sparseness of a vector should be identical to the sparseness of the vector resulting from multiple concatenation of the original vector.
This property, however, is not crucial for a proper sparseness measure.
For example, sparseness of connectivity in a biological brain increases quickly with its volume, so that connectivity in a human brain is about $170$~times more sparse than in a rat brain \citep{Karbowski2003}.
It follows that $\sigma$ features all desirable properties of a proper sparseness measure.

A sparseness-enforcing projection operator, suitable for projected gradient descent algorithms, was proposed by \citet{Hoyer2004} for optimization with respect to $\sigma$.
For a pre-defined target degree of sparseness $\sigma^*\in\intervaloo{0}{1}$, the operator finds the closest vector of a given scale that has sparseness $\sigma^*$ given an arbitrary vector.
This can be expressed formally as Euclidean projection onto parameterizations of the sets
\begin{displaymath}
  S^{(\lambda_1,\lambda_2)} := \Set{s\in\R^n | \norm{s}_1 = \lambda_1\text{ and }\norm{s}_2 = \lambda_2}\text{ and }
  S_{\geq 0}^{(\lambda_1,\lambda_2)} := S^{(\lambda_1,\lambda_2)} \cap \R_{\geq 0}^n\text{.}
\end{displaymath}
The first set is for achieving unrestricted projections, whereas the latter set is useful in situations where only non-negative solutions are feasible, for example in non-negative matrix factorization problems.
The constants $\lambda_1, \lambda_2 > 0$ are target norms and can be chosen such that all points in these sets achieve a sparseness of $\sigma^*$.
For example, if $\lambda_2$ was set to unity for yielding normalized projections, then $\lambda_1$ can be easily derived from the definition of $\sigma$.

Hoyer's original algorithm for computation of such a projection is an alternating projection onto a hyperplane representing the $L_1$ norm constraint, a hypersphere representing the $L_2$ norm constraint, and the non-negative orthant.
A slightly modified version of this algorithm has been proved to be correct by \citet{Theis2005} in the special case when exactly one negative entry emerges that is zeroed out in the orthant projection.
However, there is still no mathematically satisfactory proof for the general case.

\subsection{Contributions of this Paper}
This paper improves upon previous work in the following ways.
Section~\ref{sect:projalgorithms} proposes a simple algorithm for carrying out sparseness-enforcing projections with respect to Hoyer's sparseness measure.
Further, an improved algorithm is proposed and compared with Hoyer's original algorithm.
Because the projection itself is differentiable, it is the ideal tool for achieving sparseness in gradient-based learning.
This is exploited in Section~\ref{sect:soae}, where the sparseness projection is used to obtain a classifier that features both sparse activity and sparse connectivity in a natural way.
The benefit of these two key properties is demonstrated on a real-world classification problem, proving that sparseness acts as regularizer and improves classification results.
The final sections give an overview of related concepts and conclude this paper.

On the theoretical side, a first rigorous and mathematically satisfactory analysis of the properties of the sparseness-enforcing projection is provided.
This is lengthy and technical and therefore deferred into several appendixes.
Appendix~\ref{sect:notation} fixes the notation and gives an introduction to general projections.
In Appendix~\ref{sect:proj_symm}, certain symmetries of subsets of the Euclidean space and their effect on projections onto such sets is studied.
The problem of finding projections onto sets where Hoyer's sparseness measure attains a constant value is addressed in Appendix~\ref{sect:projfuncproof}.
Ultimately, the algorithms proposed in Section~\ref{sect:projalgorithms} are proved to be correct.
Appendix~\ref{sect:analytical_properties} investigates analytical properties of the sparseness projection and concludes with an efficient algorithm that computes its gradient.
The gradients for optimization of the parameters of the architecture proposed in Section~\ref{sect:soae} are collected in the final Appendix~\ref{sect:soae_gradients}.

\section{Algorithms for the Sparseness-Enforcing~Projection~Operator}
\label{sect:projalgorithms}

The projection onto a set is a fundamental concept, for example see \cite{Deutsch2001}:
\begin{definition}
\label{dfn:projection}
Let $x\in\R^n$ and $\emptyset\neq M\subseteq\R^n$.
Then every point in
\begin{displaymath}
  \proj_M(x) := \set{y\in M | \norm{y - x}_2 \leq \norm{z - x}_2\text{ for all }z\in M}
\end{displaymath}
is called \emph{Euclidean projection} of $x$ onto $M$.
When there is exactly one point $y$ in $\proj_M(x)$, then $y = \proj_M(x)$ is used as an abbreviation.
\end{definition}
Because $\R^n$ is finite-dimensional, $\proj_M(x)$ is nonempty for all $x\in\R^n$ if and only if $M$ is closed, and $\proj_M(x)$ is a singleton for all $x\in\R^n$ if and only if $M$ is closed and convex \citep{Deutsch2001}.
In the literature, the elements from $\proj_M(x)$ are also called \emph{best approximations} to $x$ from $M$.

Projections onto sets that fulfill certain symmetries are of special interest in this paper and are formalized and discussed in Appendix~\ref{sect:proj_symm} in greater detail.
It is notable that projections onto a \emph{permutation-invariant} set $M$, that is a set where membership is stable upon coordinate permutation, are \emph{order-preserving}.
This is proved in Lemma~\ref{lem:proj_props}\ref{lem:proj_props_a}.
As a consequence, when a vector is sorted in ascending or descending order, then its projection onto $M$ is sorted accordingly.
If $M$ is \emph{reflection-invariant}, that is when the signs of arbitrary coordinates can be swapped without violating membership in $M$, then the projection onto $M$ is \emph{orthant-preserving}, as shown in Lemma~\ref{lem:proj_props}\ref{lem:proj_props_b}.
This means that a point and its projection onto $M$ are located in the same orthant.
By exploiting this property, projections onto $M$ can be yielded by recording and discarding the signs of the coordinates of the argument, projecting onto $M\cap\R_{\geq 0}^n$, and finally restoring the signs of the coordinates of the result using the signs of the argument.
This is formalized in Lemma~\ref{lem:proj_nonneg_sols}.

As an example for these concepts, consider the set $Z := \set{x\in\R^n | \norm{x}_0 = \kappa}$ of all vectors with exactly $\kappa\in\N$ non-vanishing entries.
$Z$ is clearly both permutation-invariant and reflection-invariant.
Therefore, the projection with respect to an $L_0$ pseudo-norm constraint must be both order-preserving and orthant-preserving.
In fact, the projection onto $Z$ consists simply of zeroing out all entries but the $\kappa$ that are greatest in absolute value \citep{Blumensath2009}.
This trivially fulfills the aforementioned properties of order-preservation and orthant-preservation.

Permutation-invariance and reflection-invariance are closed under intersection and union operations.
Therefore, the unrestricted target set $S^{(\lambda_1,\lambda_2)}$ for the $\sigma$ projection is permutation-invariant and reflection-invariant.
It is hence enough to handle projections onto $S_{\geq 0}^{(\lambda_1,\lambda_2)}$ in the first place, as projections onto the unrestricted target set can easily be recovered.

In the remainder of this section, let $n\in\N$ be the problem dimensionality and let $\lambda_1,\lambda_2>0$ be the fixed target norms, which must fulfill $\lambda_2 \leq \lambda_1 \leq \sqrt{n}\lambda_2$ to avoid the existence of only trivial solutions.
In the applications of the sparseness projection in this paper, $\lambda_2$ is always set to unity to achieve normalized projections, and $\lambda_1$ is adjusted as explained in Section~\ref{sect:intro_hoyer} to achieve the target degree of sparseness $\sigma^*$.
The related problem of finding the best approximation to a point $x$ regardless of the concrete scaling, that is computing projections onto $\Set{s\in\R^n\setminus\set{0} | \sigma(s) = \sigma^*}$, can be solved by projecting $x$ onto $S^{(\lambda_1,\lambda_2)}$ and rescaling the result $p$ such as to minimize $\norm{x - \alpha p}_2$ under variation of $\alpha\in\R$, which yields $\alpha = \nicefrac{\scp{x}{p}}{\norm{p}_2^2}$.
This method is justified theoretically by Remark~\ref{rem:projfunc_scaling}.

\subsection{Alternating Projections}
\label{sect:alternating_projections}
First note that the target set can be written as an intersection of simpler sets.
Let $e_1,\dotsc,e_n\in\R^n$ be the canonical basis of the $n$-dimensional Euclidean space $\R^n$.
Further, let $e := \sum_{i=1}^ne_i\in\R^n$ be the vector where all entries are identical to unity.
Then $H := \set{a\in\R^n | e\transp a = \lambda_1}$ denotes the target hyperplane where the coordinates of all points sum up to $\lambda_1$.
In the non-negative orthant $\R_{\geq 0}^n$, this is equivalent to the $L_1$ norm constraint.
Further, define $K := \set{q\in\R^n | \norm{q}_2 = \lambda_2}$ as the target hypersphere of all points satisfying the $L_2$ norm constraint.
This yields the following factorization:
\begin{displaymath}
  S_{\geq 0}^{(\lambda_1,\lambda_2)} = \R_{\geq 0}^n \cap H \cap K =: D\text{.}
\end{displaymath}
For computation of projections onto an intersection of a finite number of closed and convex sets, it is enough to perform alternating projections onto the members of the intersection \citep{Deutsch2001}.
As $K$ is clearly non-convex, this general approach has to be altered to work in this specific setup.

First, consider $L := H\cap K$, which denotes the intersection of the $L_1$ norm target hyperplane and the $L_2$ norm target hypersphere.
$L$ essentially possesses the structure of a hypercircle, that is, all points in $L$ lie also in $H$ and there is a central point $m\in H$ and a real number $\rho\geq 0$ such that all points in $L$ have squared distance $\rho$ from $m$.
It will be shown in Appendix~\ref{sect:projfuncproof} that $m = \nicefrac{\lambda_1}{n}\cdot e\in\R^n$ and $\rho = \lambda_2^2 - \nicefrac{\lambda_1^2}{n}$.
The intersection of the non-negative orthant with the $L_1$ norm hyperplane, $C := \R_{\geq 0}^n\cap H$, is a scaled canonical simplex.
Its barycenter coincides with the barycenter $m$ of $L$.
Finally, for an index set $I\subseteq\discint{1}{n}$ let $L_I := \set{a\in L | a_i = 0\text{ for all }i\not\in I}$ denote the subset of points from $L$, where all coordinates with index not in $I$ vanish.
Its barycenter is given by $m_I = \nicefrac{\lambda_1}{d}\cdot\sum_{i\in I}e_i\in\R^n$.
With these preparations, a simple algorithm can be proposed; it computes the sparseness-enforcing projection with respect to a constraint induced by Hoyer's sparseness measure $\sigma$.
\begin{algorithm}[t]
  \caption{Proposed algorithm for computing the sparseness-enforcing projection operator for Hoyer's sparseness measure $\sigma$.}
  \label{alg:projfunc}
  \SetAlgoLined
  \KwIn{$x\in\R^n$ and $\lambda_1,\lambda_2\in\R_{> 0}$ with $\lambda_2 \leq \lambda_1 \leq \sqrt{n}\lambda_2$.}
  \KwOut{$s\in\proj_D(x)$ where $D = S_{\geq 0}^{(\lambda_1,\lambda_2)}$.}
  \BlankLine

  \tcp{Project onto target hyperplane $H$ and target hypercircle $L$.}
  $r := \proj_H(x)$\nllabel{algl:projH}\;
  $s \in \proj_L(r)$\nllabel{algl:projL}\;

  \tcp{Perform alternating projections until feasible solution is found.}
  \While{$s\not\in\R_{\geq 0}^n$}
  {\nllabel{algl:projwhile}%
    \tcp{Project onto scaled canonical simplex $C$.}
    $r := \proj_C(s)$\nllabel{algl:projC}\;
    \tcp{Project onto $L$ keeping already vanished coordinates at zero.}
    $s \in \proj_{L_I}(r)$ where $I := \set{i\inint{1}{n} | r_i \neq 0}$\nllabel{algl:projLI}\;
  }
\end{algorithm}
\begin{theorem}
\label{thm:projfunc}
For every $x\in\R^n$, Algorithm~\ref{alg:projfunc} computes an element from $\proj_D(x)$.
If $r\neq m$ after line~\ref{algl:projH} and $r\neq m_I$ after line~\ref{algl:projC} in all iterations, then $\proj_D(x)$ is a singleton.
\end{theorem}
As already pointed out, the idea of Algorithm~\ref{alg:projfunc} is that projections onto $D$ can be computed by alternating projections onto the geometric structures just defined.
The rigorous proof of correctness from Appendix~\ref{sect:projfuncproof} proceeds by showing that the set of solutions is not tampered by projection onto the intermediate structures $H$, $C$, $L$ and $L_I$.
Because of the non-convexity of $L$ and $L_I$, the relation between these sets and the simplex $C$ is non-trivial and needs long arguments to be described further, see especially Lemma~\ref{lem:subsplx_proj} and Corollary~\ref{cor:subsplx_proj}.

The projection onto the hyperplane $H$ is straightforward and discussed in Section~\ref{sect:l1nrmcnstrnt}.
As $L$ is essentially a hypersphere embedded in a subspace $H$ of $\R^n$, projections of points from $H$ onto $L$ are achieved by shifting and scaling, see Section~\ref{sect:l2nrmcnstrnt}.
The alternating projection onto $H$ and $L$ in the beginning of Algorithm~\ref{alg:projfunc} make the result of the projection onto $D$ invariant to positive scaling and arbitrary shifting of the argument, as shown in Corollary~\ref{cor:affine_invariance_projection}.
This is especially useful in practice, alleviating the need for certain pre-processing methods.
The formula for projections onto $L$ can be generalized for projections onto $L_I$ for an index set $I\subseteq\discint{1}{n}$, by keeping already vanished coordinates at zero, see Section~\ref{sect:selfsim_rec}.

Projections onto the simplex $C$ are more involved and discussed at length in Section~\ref{sect:splx_geom}.
The most relevant result is that if $x\in\R^n\setminus C$, then there exists a separator $\hat{t}\in\R$ such that $p := \proj_C(x) = \max\left(x - \hat{t}\cdot e,\ 0\right)$, where the maximum is taken element-wise \citep{Chen2011}.
In the cases considered in this paper it is always $\hat{t} \geq 0$ as shown in Lemma~\ref{lem:splx}.
This implies that all entries in $x$ that are less than $\hat{t}$ do not survive the projection, and hence the $L_0$ pseudo-norm of $x$ is strictly greater than that of $p$.
The simplex projection therefore enhances sparseness.

\begin{algorithm}[p]
  \caption{Computation of information for performing projections onto $C$, which is a scaled canonical simplex. This is an adapted version of the algorithm of \cite{Chen2011}.}
  \label{alg:projsplx_general}
  \SetAlgoLined
  \KwIn{$x\in\R^n\setminus C$ and $\lambda_1\in\R_{> 0}$.}
  \KwOut{$(\hat{t},\ d)\in\R\times\N$ such that $\proj_C(x) = \max\left(x - \hat{t}\cdot e,\ 0\right)$ and $\norm{\proj_C(x)}_0 = d$.}
  \BlankLine

  \tcp{Sort the input vector in descending order.}
  Let $\tau\in S_n$ such that $x_{\tau(1)} \geq \dots \geq x_{\tau(n)}$ and $y := P_\tau x\in\R^n$\nllabel{algl:projsplx_sort}\;
  \BlankLine

  \tcp{Find the only feasible separator $\hat{t}$.}
  $s := 0$\;
  \For{$i := 1$ \KwTo $n - 1$}
  {%
    $s := s + y_i$;\enspace
    $t := \frac{s - \lambda_1}{i}$\nllabel{algl:projsplx_general_tone}\;
    \lIf{$t \geq y_{i + 1}$}
    {%
      \KwRet $(t,\ i)$\;
    }
  }
  $s := s + y_n$;\enspace
  $t := \frac{s - \lambda_1}{n}$\nllabel{algl:projsplx_general_ttwo};\enspace
  \KwRet $(t,\ n)$\;
\end{algorithm}

\begin{algorithm}[p]
  \caption{Explicit and optimized variant of Algorithm~\ref{alg:projfunc}.}
  \label{alg:projfunc_explicit}
  \SetAlgoLined
  \KwIn{$x\in\R^n$ and $\lambda_1,\lambda_2\in\R_{> 0}$ with $\lambda_2 \leq \lambda_1 \leq \sqrt{n}\lambda_2$.}
  \KwOut{$s\in\proj_D(x)$ where $D = S_{\geq 0}^{(\lambda_1,\lambda_2)}$.}
  \BlankLine

  \SetKwBlock{KwProcProjL}{procedure {\tt proj\_L$(y\in\R^d)$}}{end}
  \KwProcProjL
  {%
    $\rho := \lambda_2^2 - \nicefrac{\lambda_1^2}{d}$\nllabel{algl:projfunc_explicit_projL_rho}\tcp*{Compute squared radius of $L_I$ (Lemma~\ref{lem:intersection_sphere_plane}).}
    $\varphi := \sum_{i=1}^d y_i^2 - \nicefrac{\lambda_1^2}{d}$\nllabel{algl:projfunc_explicit_projL_phi}\tcp*{$\varphi := \norm{y - m_I}_2^2$ (Remark~\ref{rem:point_m}).}
    \uIf{$\varphi = 0$}
    {%
      $\tdvect{y_1}{y_{d-1}} := \nicefrac{\lambda_1}{d} + \nicefrac{\sqrt{\rho}}{\sqrt{d(d-1)}}$\tcp*{$y$ equals the barycenter of $L_I$,}
      $y_d := \nicefrac{\lambda_1}{d} - \nicefrac{\sqrt{\rho(d-1)}}{\sqrt{d}}$\tcp*{pick a sorted projection (Remark~\ref{rem:projmontoL}).}
    }
    \lElse
    {%
      $y := \nicefrac{\lambda_1}{d}\cdot e + \sqrt{\nicefrac{\rho}{\varphi}}\cdot\left(y - \nicefrac{\lambda_1}{d}\cdot e\right)$\nllabel{algl:projfunc_explicit_projL_r}\tcp*{Pick unique projection (Lemma~\ref{lem:sphere}).}
    }
  }
  \BlankLine

  \tcp{Beginning of main body.}
  Let $\tau\in S_n$ such that $x_{\tau(1)} \geq \dots \geq x_{\tau(n)}$ and $y := P_\tau x\in\R^n$\tcp*{Sort the input vector.}
  $y := y + \nicefrac{1}{n}\cdot\left(\lambda_1 - \sum_{i=1}^n y_i\right)e$\nllabel{algl:projfunc_explicit_projH}\tcp*{Project onto $H$ (Lemma~\ref{lem:plane}).}
  $\text{\tt proj\_L}\left(y_1,\dotsc,y_n\right)$\nllabel{algl:projfunc_explicit_projL}\tcp*{Project in-place onto $L$.}
  \BlankLine

  \tcp{Perform alternating projections until feasible solution is found.}
  $d := n$\nllabel{algl:projfunc_explicit_d_init}\tcp*{Store current number of relevant entries of $y$.}
  \While{$\tdvect{y_1}{y_d}\not\in\R_{\geq 0}^d$}
  {\nllabel{algl:projfunc_explicit_while}%
    $(\hat{t},\ d) := \text{\tt proj\_C}\left(y_1,\dotsc,y_d\right)$\nllabel{algl:projfunc_explicit_projCcall}\tcp*{This is carried out by Algorithm~\ref{alg:projsplx_general}.}
    $\tdvect{y_1}{y_d} := \tdvect{y_1}{y_d} - \hat{t}$\tcp*{Project onto $C$ (Proposition~\ref{prop:projsplx}).}
    $\text{\tt proj\_L}\left(y_1,\dotsc,y_d\right)$\tcp*{Project onto $L_I$ where $I = \discint{1}{d}$ (Lemma~\ref{lem:splxrcsn}).}
  }
  \BlankLine

  \tcp{Undo sorting permutation and set remaining entries to zero.}
  $s\in\set{0}^n$;\enspace%
  \lFor{$i := 1$ \KwTo $d$}
  {%
    $s_{\tau(i)} := y_i$\;
  }
\end{algorithm}

The separator $\hat{t}$ and the number of nonzero entries in the projection onto $C$ can be computed with Algorithm~\ref{alg:projsplx_general}, which is an adapted version of the algorithm of \citet{Chen2011}.
In line~\ref{algl:projsplx_sort}, $S_n$ denotes the symmetric group and $P_\tau$ denotes the permutation matrix associated with a permutation $\tau\in S_n$.
The algorithm works by sorting its argument $x$ and then determining $\hat{t}$ as the mean value of the largest entries of $x$ minus the target $L_1$ norm $\lambda_1$.
The number of relevant entries for computation of $\hat{t}$ is equal to the $L_0$ pseudo-norm of the projection and is found by trying all feasible values, starting with the largest ones.
The computational complexity of Algorithm~\ref{alg:projsplx_general} is dominated by sorting the input vector and is thus quasilinear.

\subsection{Optimized Variant}
\label{sect:projfunc_optimized_variant}
Because of the permutation-invariance of the sets involved in the projections, it is enough to sort the vector that is to be projected onto $D$ once.
This guarantees that the working vector that emerges from subsequent projections is sorted also.
No additional sorting has then to be carried out when using Algorithm~\ref{alg:projsplx_general} for projections onto $C$.
This additionally has the side effect that the non-vanishing entries of the working vector are always concentrated in its first entries.
Hence all relevant information can always be stored in a small unit-stride array, to which access is more efficient than to a large sparse array.
Further, the index set $I$ of non-vanishing entries in the working vector is always of the form $I = \discint{1}{d}$, where $d$ is the number of nonzero entries.

Algorithm~\ref{alg:projfunc_explicit} is a variant of Algorithm~\ref{alg:projfunc} where these optimizations were applied, and where the explicit formulas for the intermediate projections were used.
The following result, which is proved in Appendix~\ref{sect:projfuncproof}, states that both algorithms always compute the same result:
\begin{theorem}
\label{thm:projfunc_improved}
Algorithm~\ref{alg:projfunc} is equivalent to Algorithm~\ref{alg:projfunc_explicit}.
\end{theorem}
Projections onto $C$ increase the amount of vanishing entries in the working vector, which is of finite dimension $n$.
Hence, at most $n$ alternating projections are carried out, and the algorithm terminates in finite time.
Further, the complexity of each iteration is at most linear in the $L_0$ pseudo-norm of the working vector.
The theoretic overall computational complexity is thus at most quadratic in problem dimensionality $n$.

\subsection{Comparison with Hoyer's Original Algorithm}
\label{sect:comp_hoyer}
The original algorithm for the sparseness-enforcing projection operator proposed by \cite{Hoyer2004} is hard to understand, and correctness has been proved by \cite{Theis2005} in a special case only.
A simple alternative has been proposed with Algorithm~\ref{alg:projfunc} in this paper.
Based on the symmetries induced by Hoyer's sparseness measure $\sigma$ and by exploiting the projection onto a simplex, an improved method was given in Algorithm~\ref{alg:projfunc_explicit}.

The improved algorithm proposed in this paper always requires at most the same number of iterations of alternating projections as the original algorithm.
The original algorithm uses a projection onto the non-negative orthant $\R_{\geq 0}^n$ to achieve vanishing coordinates in the working vector.
This operation can be written as $\proj_{\R_{\geq 0}^n}(x) = \max(x,\ 0)$.
In the improved algorithm, a simplex projection is used for this purpose, expressed formally as $\proj_C(x) = \max\left(x - \hat{t}\cdot e,\ 0\right)$ with $\hat{t}\in\R$ chosen accordingly.
Due to the theoretical results on simplex geometry from Section~\ref{sect:splx_geom} and their application in Lemma~\ref{lem:splx} in Section~\ref{sect:selfsim_rec}, the number $\hat{t}$ is always non-negative.
Therefore, at least the same amount of entries is set to zero in the simplex projection compared to the projection onto the non-negative orthant, see also Corollary~\ref{cor:splx}.
Hence with induction for the number of non-vanishing entries in the working vector, the number of iterations the proposed algorithm needs to terminate is bounded by the number of iterations the original method needs to terminate given the same input.

The experimental determination of an estimate of the number of iterations required was carried out as follows.
Random vectors with sparseness $0.15$ were sampled and their sparse projections were computed using the respective algorithms, to gain the best normalized approximations with a target sparseness degree of $\sigma^* := 0.90$.
For both algorithms the very same vectors were used as input.
During the run-time of the algorithms, the number of iterations that were necessary to compute the result were counted.
Additionally, the number of nonzero entries in the working vector was recorded in each iteration.
This was done for different dimensionalities, and for each dimensionality $1000$~vectors were sampled.

\begin{figure}[t]
  \centering
  \includegraphics[page=1]{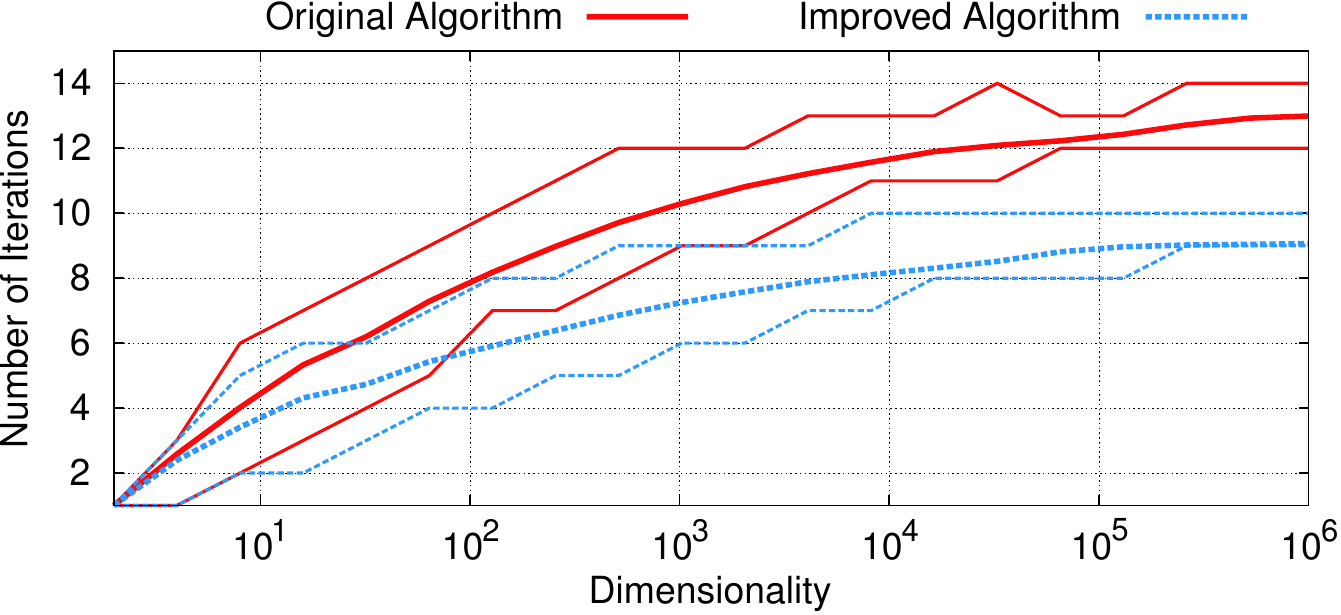}
  \caption{Comparison of the number of iterations of the original algorithm for the projection onto $D$ with the improved version as proposed in this paper. The sparseness-enforcing projection with target sparseness $0.90$ was carried out for input vectors of sparseness $0.15$. The thick lines indicate the mean number of iterations required for the projection, and the thin lines indicate the minimum and maximum number of iterations, respectively. Even for input vectors with a million entries, less than $14$~iterations are required to find the projection. With the improved algorithm, this reduces to at most $10$~iterations.}
  \label{fig:alg_iterations}
\end{figure}
\begin{figure}[p]
  \centering
  \includegraphics[page=2]{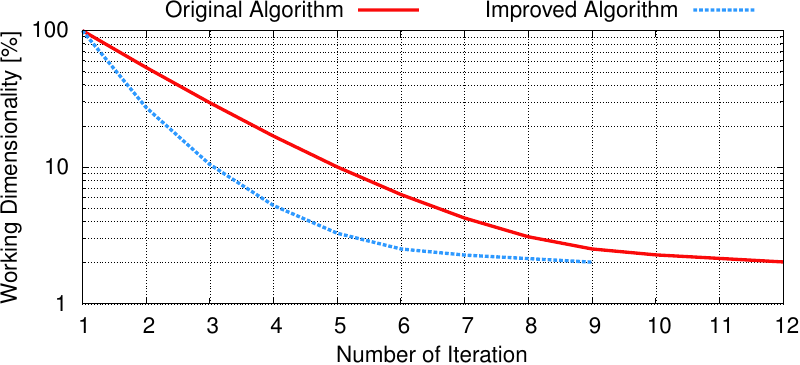}
  \caption{Comparison of the number of non-vanishing entries in the working vectors of the original algorithm and the improved algorithm during run-time. The algorithms were run with input vectors of dimensionality $1000$ and initial sparseness $0.15$ to compute projections with sparseness $0.90$. Standard deviations were always less than $1\%$; they were omitted in the plot to avoid clutter. The algorithm proposed in this paper reduces dimensionality more quickly and terminates earlier than the original algorithm.}
  \label{fig:alg_iteration_dims}
\end{figure}
\begin{figure}[p]
  \centering
  \includegraphics[page=3]{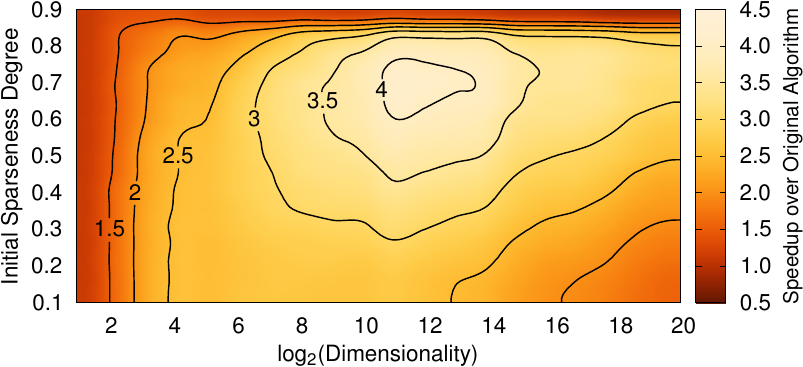}
  \caption{Ratio of the computation time of the original algorithm and the improved algorithm for a variety of input dimensionality and initial vector sparseness. Numbers greater than one indicate parameterizations where the proposed algorithm is more efficient than the original one. There is a large region where the speedup is decent.}
  \label{fig:alg_timings}
\end{figure}

Figure~\ref{fig:alg_iterations} shows statistics on the number of iterations the algorithms needed to terminate.
As was already observed by \cite{Hoyer2004}, the number of required iterations grows very slowly with problem dimensionality.
For $n = 10^6$, only between~$12$ and $14$~iterations were needed with the original algorithm to compute the result.
With Algorithm~\ref{alg:projfunc_explicit}, this can be improved to requiring~$9$ to~$10$ iterations, which amounts to roughly~$30\%$ less iterations.
Due to the small slope in the number of required iterations, it can be conjectured that this quantity is at most logarithmic in problem dimensionality $n$.
If this applies, the complexity of Algorithm~\ref{alg:projfunc_explicit} is at most quasilinear.
Because the input vector is sorted in the beginning, it is also not possible to fall below this complexity class.

The progress of working dimensionality reduction for problem dimensionality $n = 1000$ is depicted in Figure~\ref{fig:alg_iteration_dims}, averaged over the $1000$~input vectors from the experiment.
After the first iteration, that is after projecting onto $H$ and $L$, the working dimensionality still matches the input dimensionality.
Starting with the second iteration, dimensions are discarded by projecting onto $\R_{\geq 0}^n$ in the original algorithm and onto $C$ in the improved variant, which yields vanishing entries in the working vectors.
With the original algorithm, in the mean $54\%$ of all entries are nonzero after the second iteration, while with the improved algorithm only $27\%$ of the original $1000$~dimensions remain in the mean.
This trend continues in subsequent iterations such that the final working dimensionality is reached more quickly with the algorithm proposed in this paper.
Although using Algorithm~\ref{alg:projsplx_general} to perform the simplex projection is more expensive than just setting negative entries to zero in the orthant projection, the overhead quickly amortizes because of the boost in dimensionality reduction.

For determination of the relative speedup incorporated with both the simplex projection and the access to unit-stride arrays due to the permutation-invariance, both algorithms were implemented as C++~programs using an optimized implementation of the BLAS~library for carrying out the vector operations.
The employed processor was an Intel Core i7-990X.
For a range of different dimensionalities, a set of vectors with varying initial sparseness were sampled.
The number of the vectors for every pair of dimensionality and initial sparseness was chosen such that the processing time of the algorithms was several orders of magnitudes greater than the latency time of the operation system.
Then the absolute time needed for the algorithms to compute the projections with a target sparseness of~$0.90$ were measured, and their ratio was taken to compute the relative speedup.
The results of this experiment are depicted in Figure~\ref{fig:alg_timings}.
It is evident that the maximum speedup is achieved for vectors with a dimensionality between $2^9$ and $2^{15}$, and an initial sparseness greater than $0.40$.
For low initial sparseness, as is achieved by randomly sampled vectors, a speedup of about $2.5$ can be achieved for a broad spectrum of dimensionality between $2^4$ and $2^{13}$.

The improvements to the original algorithm are thus not only theoretical, but also noticeable in practice.
The speedup is especially useful when the projection is used as a neuronal transfer function in a classifier as proposed in Section~\ref{sect:soae}, because then the computational complexity of the prediction of class membership of unknown samples can be reduced.

\subsection{Function Definition and Differentiability}
\label{sect:projfunc_differentiability}
It is clear from Theorem~\ref{thm:projfunc} that the projection onto $D$ is unique almost everywhere.
Therefore the set $R := \Set{x\in\R^n | \abs{\proj_D(x)} \neq 1}$ is a null set.
However, $R\neq\emptyset$ as for example the projection is not unique for vectors where all entries are identical.
In other words, for $x := \xi e\in\R^n$ for some $\xi\in\R$ follows $\proj_H(x) = m$ and $\proj_L(m) = L$.
If $n = 2$ a possible solution is given by $\left(\alpha, \beta\right)\transp\in\proj_D(x)$ with $\alpha$ and $\beta$ given as stated in Remark~\ref{rem:projmontoL}, as in this case $\alpha$ and $\beta$ are positive.
Additionally, another solution is given by $\left(\beta, \alpha\right)\transp\in\proj_D(x)$ which is unequal to the other solution because of $\alpha\neq\beta$.
A similar argument can be used to show non-uniqueness for all $n\geq 2$.
As $R$ is merely a small set, non-uniqueness is not an issue in practical applications.

The sparseness-enforcing projection operator that is restricted to non-negative solutions can thus be cast almost everywhere as a function
\begin{displaymath}
  \pi_{\geq 0}\colon\R^n\setminus R\to D\text{,}\qquad x\mapsto\proj_D(x)\text{.}
\end{displaymath}
Exploiting reflection-invariance implies that the unrestricted variant of the projection
\begin{displaymath}
  \pi\colon\R^n\setminus R\to S^{(\lambda_1,\lambda_2)}\text{,}\qquad x\mapsto s\hada\pi_{\geq 0}\left(\abs{x}\right)\text{,}
\end{displaymath}
is well-defined, where $s\in\set{\pm 1}^n$ is given as described in Lemma~\ref{lem:proj_nonneg_sols}.
Note that computation of $\pi_{\geq 0}$ is a crucial prerequisite to computation of the unrestricted variant $\pi$.
It will be used exclusively in Section~\ref{sect:soae} because non-negativity is not necessary in the application proposed there.

If $\pi$ or $\pi_{\geq 0}$ is employed in an objective function that is to be optimized, the information whether these functions are differentiable is crucial for selecting an optimization strategy.
As an example, consider once more projections onto $Z := \set{x\in\R^n | \norm{x}_0 = \kappa}$ where $\kappa\in\N$ is a constant.
It was already mentioned in Section~\ref{sect:projalgorithms} that the projection onto $Z$ consists simply of zeroing out the elements that are smallest in absolute value.
Let $x\in\R^n$ be a point and let $\tau\in S_n$ be a permutation such that $\abs{x_{\tau(1)}} \geq \dots \geq \abs{x_{\tau(n)}}$.
Clearly, if $\abs{x_{\tau(\kappa)}} \neq \abs{x_{\tau(\kappa + 1)}}$ then $\proj_Z(x) = y$ where $y_i = x_i$ for $i\in\discint{\tau(1)}{\tau(\kappa)}$ and $y_i = 0$ for $i\in\discint{\tau(\kappa + 1)}{\tau(n)}$.
Moreover, when $\abs{x_{\tau(\kappa)}} \neq \abs{x_{\tau(\kappa + 1)}}$ then there exists a neighborhood $U$ of $x$ such that $\proj_Z(s) = \sum_{i=1}^\kappa s_{\tau(i)}e_{\tau(i)}$ for all $s\in U$.
With this closed-form expression, $s\mapsto\proj_Z(s)$ is differentiable in $x$ with gradient $\nicefrac{\partial\proj_Z(x)}{\partial x} = \diag\left(\sum_{i=1}^\kappa e_{\tau(i)}\right)$, that is the identity matrix where the entries on the diagonal belonging to small absolute values of $x$ have been zeroed out.
If the requirement on $x$ is not fulfilled, then a small distortion of $x$ is sufficient to find a point in which the projection onto $Z$ is differentiable.

In contrast to the $L_0$ projection, differentiability of $\pi$ and $\pi_{\geq 0}$ is non-trivial.
A full-length discussion is given in Appendix~\ref{sect:analytical_properties}, and concludes that both $\pi$ and $\pi_{\geq 0}$ are differentiable almost everywhere.
It is more efficient when only the product of the gradient with an arbitrary vector needs to be computed, see Corollary~\ref{cor:projfuncgraddgemv}.
Such an expression emerges in a natural way by application of the chain rule to an objective function where the sparseness-enforcing projection is used.
In practice this weaker form is thus mostly no restriction and preferable for efficiency reasons over the more general complete gradient as given in Theorem~\ref{thm:projfuncblockgrad}.

The derivative of $\pi_{\geq 0}$ is obtained by exploiting the structure of Algorithm~\ref{alg:projfunc}.
Because the projection onto $D$ is essentially a composition of projections onto $H$, $C$, $L$ and $L_I$, the overall gradient can be computed using the chain rule.
The gradients of the intermediate projections are simple expressions and can be combined to yield one matrix for each iteration of alternating projections.
Since these iteration gradients are basically sums of dyadic products, their product with an arbitrary vector can be computed by primitive vector operations.
With matrix product associativity, this process can be repeated to efficiently compute the product of the gradient of $\pi_{\geq 0}$ with an arbitrary vector.
For this, it is sufficient to record some intermediate quantities during execution of Algorithm~\ref{alg:projfunc_explicit}, which does not add any major overhead to the algorithm itself.
The gradient of the unrestricted variant $\pi$ can be deduced in a straightforward way from the gradient of $\pi_{\geq 0}$ because of their close relationship.

\section{Sparse Activity and Sparse Connectivity in Supervised Learning}
\label{sect:soae}
The sparseness-enforcing projection operator can be cast almost everywhere as vector-valued function $\pi$, which is differentiable almost everywhere, see Section~\ref{sect:projfunc_differentiability}.
This section proposes a hybrid of an auto-encoder network and a two-layer neural network, where the sparseness projection is employed as a neuronal transfer function.
The proposed model is called supervised online auto-encoder~(SOAE) and is intended for classification by means of a neural network that features sparse activity and sparse connectivity.
Because of the analytical properties of the sparseness-enforcing projection operator, the model can be optimized end-to-end using gradient-based methods.

\subsection{Architecture}
Figure~\ref{fig:soae_architecture} depicts the data flow in the proposed model.
There is one module for reconstruction capabilities and one module for classification capabilities.
The \emph{reconstruction module}, depicted on the left of Figure~\ref{fig:soae_architecture}, operates by converting an input sample $x\in\R^d$ into an \emph{internal representation} $h\in\R^n$, and then computing an approximation $\tilde{x}\in\R^d$ to the original input sample.
In doing so, the product $u\in\R^n$ of the input sample with a \emph{matrix of bases} $W\in\R^{d\times n}$ is computed, and a transfer function $f\colon\R^n\to\R^n$ is applied.
For sparse activity, $f$ can be chosen to be the sparseness-enforcing projection operator $\pi$ or the projection with respect to the $L_0$ pseudo-norm.
This guarantees that the internal representation is sparsely populated and close to $u$.
The reconstruction is achieved like in a linear generative model, by multiplication of the matrix of bases with the internal representation.
Hence the same matrix $W$ is used for both encoding and decoding, rendering the reconstruction module symmetric, or in other words with tied weights.
This approach is similar to principal component analysis \citep{Hotelling1933}, restricted Boltzmann machines for deep auto-encoder networks \citep{Hinton2006} and to sparse encoding symmetric machine \citep{Ranzato2008}.

By enforcing $W$ to be sparsely populated, the sparse connectivity property holds as well.
More formally, the aim is that $\sigma(We_i) = \sigma_W$ holds for all $i\inint{1}{n}$, where $\sigma_W\in\intervaloo{0}{1}$ is the target degree of connectivity sparseness and $We_i$ is the $i$-th column of $W$.
This condition was adopted from non-negative matrix factorization with sparseness constraints \citep{Hoyer2004}.
In the context of neural networks, the synaptic weights of individual neurons are stored in the columns of the weight matrix $W$.
The interpretation of this formal sparseness constraint is then that each neuron is only allowed to be sparsely connected with the input layer.

\begin{figure}[t]
  \centering
  \includegraphics[page=4]{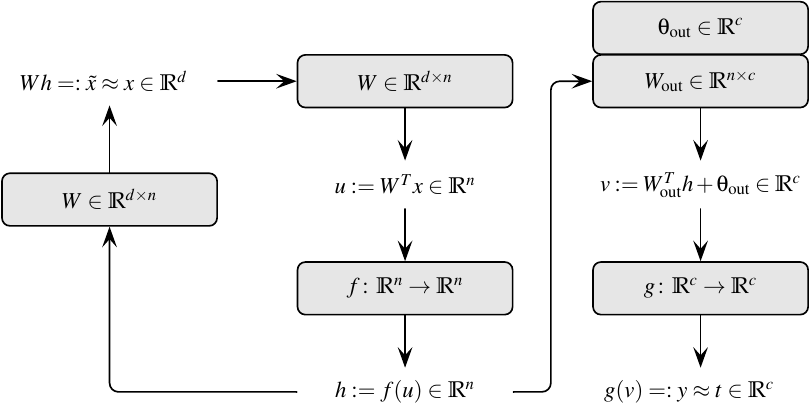}
  \caption{Architecture and data flow of supervised online auto-encoder~(SOAE). The circle on the left, mapping from an input sample $x$ to its approximation $\tilde{x}$ comprises the reconstruction module. The classification module consists of the mapping from $x$ to the classification decision $y$. The matrix of bases $W$ shall be sparsely populated to account for the sparse connectivity property. If the transfer function $f$ is set to the sparseness projection, the internal representation $h$ will be sparsely populated, fulfilling the sparse activity property.}
  \label{fig:soae_architecture}
\end{figure}

The \emph{classification module} is shown on the right-hand side of Figure~\ref{fig:soae_architecture}.
It computes a \emph{classification decision} $y\in\R^c$ by feeding $h$ through a one-layer neural network.
The network output $y$ is yielded through computation of the product with a matrix of weights $W_{\out}\in\R^{n\times c}$, addition of a threshold vector $\theta_{\out}\in\R^c$ and application of a transfer function $g\colon\R^c\to\R^c$.
This module shares the inference of the internal representation with the reconstruction module, which can also be considered a one-layer neural network.
Therefore the entire processing path from $x$ to $y$ forms a two-layer neural network \citep{Rumelhart1986}, where $W$ stores the synaptic weights of the hidden layer, and $W_{\out}$ and $\theta_{\out}$ are the parameters of the output layer.

The input sample $x$ shall be approximated by $\tilde{x}$, and the target vector for classification $t\in\R^c$ shall be approximated by $y$.
This is achieved by optimization of the parameters of SOAE, that is the quantities $W$, $W_{\out}$ and $\theta_{\out}$.
The goodness of the approximation $x\approx\tilde{x}$ is estimated using a differentiable similarity measure $s_R\colon\R^d\times\R^d\to\R$, and the approximation $y\approx t$ is assessed by another similarity measure $s_C\colon\R^c\times\R^c\to\R$.
For minimizing the deviation in both approximations, the objective function
\begin{displaymath}
  E_{\SOAE}\left(W,\ W_{\out},\ \theta_{\out}\right) := \left(1 - \alpha\right)\cdot s_R\left(\tilde{x},\ x\right) + \alpha\cdot s_C\left(y,\ t\right)
\end{displaymath}
shall be optimized, where $\alpha\in\intervalcc{0}{1}$ controls the trade-off between reconstruction and classification capabilities.
To incorporate sparse connectivity, feasible solutions are restricted to fulfill $\sigma(We_i) = \sigma_W$ for all $i\inint{1}{n}$.
If $\alpha = 0$, then SOAE is identical to a symmetric auto-encoder network with sparse activity and sparse connectivity.
In the case of $\alpha = 1$, SOAE forms a two-layer neural network for classification with a sparsely connected hidden layer and where the activity in the hidden layer is sparse.
The parameter $\alpha$ can also be used to blend continuously between these two extremes.
Note that $\tilde{x}$ only depends on $W$ but not on $W_{\out}$ or $\theta_{\out}$, but $y$ depends on $W$, $W_{\out}$ and $\theta_{\out}$.
Hence $W_{\out}$ and $\theta_{\out}$ are only relevant when $\alpha > 0$, whereas $W$ is essential for all choices of $\alpha$.

An appropriate choice for $s_R$ is the correlation coefficient \citep[see for example][]{Rodgers1988},
because it is normed to values in the interval $\intervalcc{-1}{1}$, invariant to affine-linear transformations, and differentiable.
If $f$ is set to $\pi$, then a model that is invariant to the concrete scaling and shifting of the occurring quantities can be yielded.
This follows because $\pi$ is also invariant to such transformations, see Corollary~\ref{cor:affine_invariance_projection}.
The similarity measure for classification capabilities $s_C$ is chosen to be the cross-entropy error function \citep{Bishop1995}, which was shown empirically by \citet{Simard2003} to induce better classification capabilities than the mean squared error function.
The softmax transfer function \citep{Bishop1995} is used as transfer function $g$ of the output layer.
It provides a natural pairing together with the cross-entropy error function \citep{Dunne1997} and supports multi-class classification.

\subsection{Learning Algorithm}
\label{sect:soae_learning_algorithm}
The proposed optimization algorithm for minimization of the objective function $E_{\SOAE}$ is projected gradient descent \citep{Bertsekas1999}.
Here, each update to the degrees of freedom is followed by application of the sparseness projection to the columns of $W$ to enforce sparse connectivity.
There are theoretical results on the convergence of projected gradient methods when projections are carried out onto convex sets \citep{Bertsekas1999}, but here the target set for projection is non-convex.
Nevertheless, the experiments described below show that projected gradient descent is an adequate heuristic in the situation of the SOAE~framework to tune the network parameters.
For completeness, the gradients of $E_{\SOAE}$ with respect to the network parameters are given in Appendix~\ref{sect:soae_gradients}.
Update steps are carried out after every presentation of a pair of an input sample and associated target vector.
This online learning procedure results in faster learning and improves generalization capabilities over batch learning \citep{Wilson2003,Bottou2004}.

A learning set with samples from $\R^d$ and associated target vectors from $\set{0,1}^c$ as one-of-$c$-codes is input to the algorithm.
The dimensionality of the internal representation $n$ and the target degree of sparseness with respect to the connectivity $\sigma_W\in\intervaloo{0}{1}$ are parameters of the algorithm.
Sparseness of connectivity increases for larger $\sigma_W$, as Hoyer's sparseness measure is employed in the definition of the set of feasible solutions.

Two possible choices for the hidden layer's transfer function $f$ to achieve sparse activity were discussed in this paper.
One possibility is to carry out the projection with respect to the $L_0$ pseudo-norm.
The more sophisticated method is to use the unrestricted sparseness-enforcing projection operator $\pi$ with respect to Hoyer's sparseness measure $\sigma$, which can be carried out by Algorithm~\ref{alg:projfunc_explicit}.
In both cases, a target degree for sparse activity is a parameter of the learning algorithm.
In case of the $L_0$ projection, this sparseness degree is denoted by $\kappa\inint{1}{n}$, and sparseness increases with smaller values of it.
For the $\sigma$ projection, $\sigma_H\in\intervaloo{0}{1}$ is used, where larger values indicate more sparse activity.

Initialization of the columns of $W$ is achieved by selecting a random subset of the learning set, similar to the initialization of radial basis function networks \citep{Bishop1995}.
This ensures significant activity of the hidden layer from the very start, resulting in strong gradients and therefore reducing training time.
The parameters of the output layer, that is $W_{\out}$ and $\theta_{\out}$, are initialized by sampling from a zero-mean Gaussian distribution with a standard deviation of $\nicefrac{1}{100}$.

In every epoch, a randomly selected subset of samples and associated target vectors from the learning set is used for stochastic gradient descent to update $W$, $W_{\out}$ and $\theta_{\out}$.
The results from Appendix~\ref{sect:soae_gradients} can be used to efficiently compute the gradient of the objective function.
There, the gradient for the transfer function $f$ only emerges as a product with a vector.
The gradient for the $L_0$ projection is trivial and was given as an example in Section~\ref{sect:projfunc_differentiability}.
If $f$ is Hoyer's sparseness-enforcing projection operator, it is possible to exploit that only the product of the gradient with a vector is needed.
In this case, it is more efficient to compute the result of the multiplication implicitly using Corollary~\ref{cor:projfuncgraddgemv} and thus avoid the computation of the entire gradient of $\pi$.

After every epoch, a sparseness projection is applied to the columns of $W$.
This guarantees that $\sigma(We_i) = \sigma_W$ holds for all $i\inint{1}{n}$, and therefore the sparse connectivity property is fulfilled.
The trade-off variable $\alpha$ which controls the weight of the reconstruction and the classification term is adjusted according to $\alpha(\nu) := 1 - \exp\left(-\nicefrac{\nu}{100}\right)$, where $\nu\in\N$ denotes the number of the current epoch.
Thus $\alpha$ starts at zero, increases slowly and asymptotically reaches one.
The emphasis at the beginning of the optimization is thus on reconstruction capabilities.
Subsequently, classification capabilities are incorporated slowly, and in the final phase of training classification capabilities exclusively are optimized.
This continuous variant of unsupervised pre-training \citep{Hinton2006} leads to parameters in the vicinity of a good minimizer for classification capabilities before classification is preferred over reconstruction through the trade-off parameter $\alpha$.
Compared to the choice $\alpha\equiv 1$ this strategy helps to stabilize the trajectory in parameter space and makes the objective function values settle down more quickly, such that the termination criterion is satisfied earlier.

\subsection{Description of Experiments}
\label{sect:description_experiments}
To assess the classification capabilities and the impact of sparse activity and sparse connectivity, the MNIST database of handwritten digits \citep{MNIST} was employed.
It is a popular benchmark data set for classification algorithms, and numerous results with respect to this data set are reported in the literature.
The database consists of $70\;000$ samples, divided into a learning set of $60\;000$ samples and an evaluation set of $10\;000$ samples.
Each sample represents a digit of size $28\times 28$ pixels and has a class label from $\discint{0}{9}$ associated with it.
Therefore the input and output dimensionalities are $d := 28^2 = 784$ and $c := 10$, respectively.
The classification error is given in percent of all $10\;000$ evaluation samples, hence $0.01\%$ corresponds to a single misclassified digit.

For generation of the original data set, the placement of the digits has been achieved based on their barycenter \citep{MNIST}.
Because of sampling and rounding errors, the localization uncertainty can hence be assumed to be less than one pixel in both directions.
To account for this uncertainty, the learning set was augmented by jittering each sample in each of eight possible directions by one pixel, yielding $540\;000$ samples for learning in total.
The evaluation set was left unchanged to yield results that can be compared to the literature.
As noted by \citet{Hinton2006}, the learning problem is no more permutation-invariant due to the jittering, as information on the neighborhood of the pixels is implicitly incorporated in the learning set.

However, classification results improve dramatically when such prior knowledge is used.
This was demonstrated by \citet{Schoelkopf1997} using the virtual support vector method, which improved a support vector machine with polynomial kernel of degree five from an error of $1.4\%$ to $1.0\%$ by jittering the support vectors by one pixel in four principal directions.
This result was extended by \citet{DeCoste2002}, where a support vector machine with a polynomial kernel of degree nine was improved from an error of $1.22\%$ to $0.68\%$ by jittering in all possible eight directions.
Further improvements can be achieved by generating artificial training samples using elastic distortions \citep{Simard2003}.
This reduced the error of a two-layer neural network with $800$~hidden units to $0.7\%$, compared to the $1.1\%$ error yielded when training on samples created by affine distortions.
Very big and very deep neural networks possess a large number of adaptable weights.
In conjunction with elastic and affine distortions such neural networks can yield errors as low as $0.35\%$ \citep{Ciresan2010a}.
The current record error of $0.23\%$ is held by an approach that combines distorted samples with a committee of convolutional neural networks \citep{Ciresan2012a}.
This is an architecture that has been optimized exclusively for input data that represents images, that is where the neighborhood of the pixels is hard-wired in the classifier.
To allow for a plain evaluation that does not depend on additional parameters for creating artificial samples, the jittered learning set with $540\;000$ samples is used throughout this paper.

The experimental methodology was as follows.
The number of hidden units was chosen to be $n := 1000$ in all experiments that are described below.
This is an increased number compared to the $800$~hidden units employed by \citet{Simard2003}, but promises to yield better results when an adequate number of learning samples is used.
As all tested learning algorithms are essentially gradient descent methods, an initial step size had to be chosen.
For each candidate step size, five runs of a two-fold cross validation were carried out on the learning set.
Then, for each step size the median of the ten resulting classification errors was computed.
The winning step size was then determined to be the one that achieved a minimum median of classification errors.

In every epoch, $21\;600$ samples were randomly chosen from the learning set and presented to the network.
This number of samples was chosen as it is $\nicefrac{1}{25}$-th of the jittered learning set.
The step size was multiplicatively annealed using a factor of $0.999$ after every epoch.
Optimization was terminated once the relative change in the objective function became very small and no more significant progress on the learning set could be observed.
The resulting classifiers were then applied to the evaluation set, and misclassifications were counted.

\subsection{Experimental Results}
\label{sect:experimental_results}
Two variants of the supervised online auto-encoder architecture as proposed in this section were trained on the augmented learning set.
In both variants, the target degree of sparse connectivity was set to $\sigma_W := 0.75$.
This choice was made because $96\%$~of all samples in the learning set possess a sparseness which is less than $0.75$.
Therefore, the resulting bases are forced to be truly sparsely connected compared to the sparseness of the digits.

The first variant is denoted by SOAE-$\sigma$.
Here, the sparseness-enforcing projection operator $\pi$ was used as transfer function $f$ in the hidden layer.
Target degrees of sparse activity $\sigma_H$ with respect to Hoyer's sparseness measure $\sigma$ were chosen from the interval $\intervalcc{0.20}{0.95}$ in steps of size $0.05$.
This variant was then trained on the jittered learning set using the method described in Section~\ref{sect:soae_learning_algorithm}.
For every value of $\sigma_H$, the resulting sparseness of activity was measured after training using the $L_0$ pseudo-norm.
For this, each sample of the learning set was presented to the networks, and the number of active units in the hidden layer was counted.
Figure~\ref{fig:soae-sigma-l0-connection} shows the resulting mean value and standard deviation of sparse activity.
If $\sigma_H = 0.20$ is chosen, then in the mean about $800$~of the total $1000$~hidden units are active upon presentation of a sample from the learning set.
For $\sigma_H = 0.80$ only one hundred units are active at any one time, and for $\sigma_H = 0.95$ there are only eleven active units.
The standard deviation of the activity decreases when sparseness increases, hence the mapping from $\sigma_H$ to the resulting number of active units becomes more accurate.

\begin{figure}[p]
  \centering
  \includegraphics[page=5]{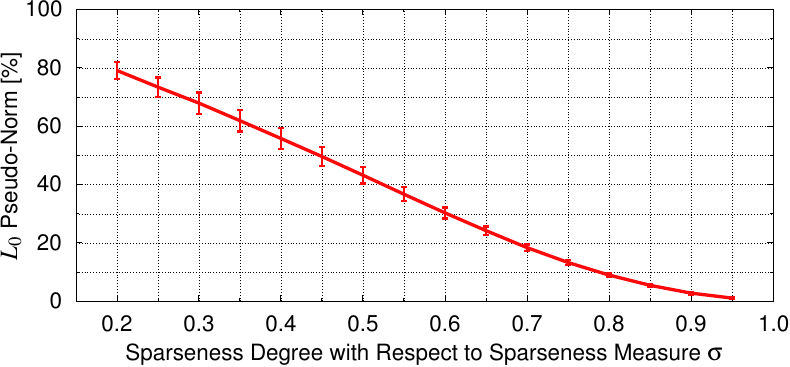}
  \caption{Resulting amount of nonzero entries in an internal representation $h$ with $1000$~entries, depending on the target degree of sparseness for activity $\sigma_H$ with respect to Hoyer's sparseness measure $\sigma$. For low values of $\sigma_H$, about $80\%$ of the entries are nonzero, whereas for very high sparseness degrees only $1\%$ of the entries do not vanish. The error bars indicate $\pm$~one standard deviation distance from the mean value. Standard deviation shrinks with increasing sparseness degree, making the mapping more accurate.}
  \label{fig:soae-sigma-l0-connection}
\end{figure}

The second variant, denoted SOAE-$L_0$, differs from SOAE-$\sigma$ in that the projection with respect to the $L_0$ pseudo-norm as transfer function $f$ was used.
The target sparseness of activity is given by a parameter $\kappa\inint{1}{n}$, which controls the exact number of units that are allowed to be active at any one time.
For the experiments, the values for $\kappa$ were chosen to match the mean activities from the SOAE-$\sigma$ experiments.
This way the results of both variants can be compared based on a unified value of activity sparseness.
The results are depicted in Figure~\ref{fig:soae-results-mnist}.
Usage of the $\sigma$ projection consequently outperforms the $L_0$ projection for all sparseness degrees.
Even for high sparseness of activity, that is when only about ten percent of the units are allowed to be active at any one time, good classification capabilities can be obtained with SOAE-$\sigma$.
For $\kappa\in\intervalcc{242}{558}$, the classification results of SOAE-$L_0$ reach an optimum.
SOAE-$\sigma$ is more robust, as classification capabilities first begin to collapse when sparseness is below $5\%$, whereas SOAE-$L_0$ starts to degenerate when sparseness falls below $20\%$.
For $\sigma_H\in\intervalcc{0.45}{0.85}$, roughly translating to between $5\%$ and $50\%$ activity, about equal classification performance is achieved using SOAE-$\sigma$.

\begin{figure}[p]
  \centering
  \includegraphics[page=6]{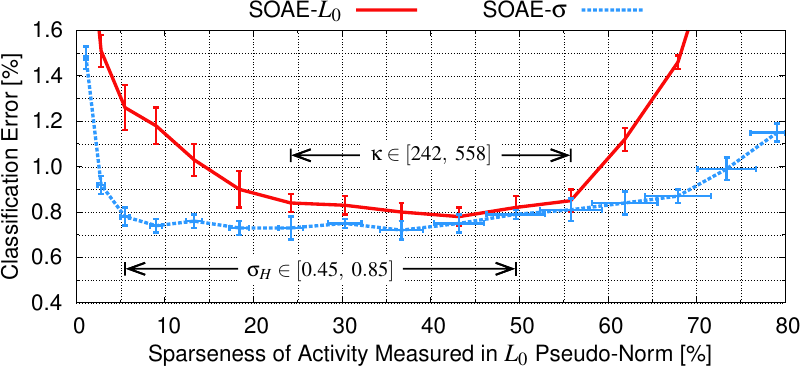}
  \caption{Resulting classification error on the MNIST evaluation set for the supervised online auto-encoder network, in dependence of sparseness of activity in the hidden layer. The projection onto an $L_0$ pseudo-norm constraint for variant SOAE-$L_0$ and the projection onto a constraint determined by sparseness measure $\sigma$ for variant SOAE-$\sigma$ were used as transfer functions. The error bars indicate $\pm$~one standard deviation difference from the mean.}
  \label{fig:soae-results-mnist}
\end{figure}

It can thus be concluded that using the sparseness-enforcing projection operator as described in this paper yields better results than when the simple $L_0$ projection is used to achieve sparse activity.
To assess the benefit more precisely and to investigate the effect of individual factors, several comparative experiments have been carried out.
A summary of these experiments and their outcome is given in Table~\ref{tbl:soae_comparatative_results}.
The variants SOAE-$\sigma$ and SOAE-$L_0$ denote the entirety of the respective experiments where sparseness of activity lies in the intervals described above, that is $\sigma_H\in\intervalcc{0.45}{0.85}$ and $\kappa\in\intervalcc{242}{558}$, respectively.
Using these intervals, SOAE-$\sigma$ and SOAE-$L_0$ achieved a median error of $0.75\%$ and $0.82\%$ on the evaluation set, respectively.
Variant SOAE-$\sigma$-conn is essentially equal to SOAE-$\sigma$, except for sparse connectivity not being incorporated.
Sparseness of activity here was also chosen to be $\sigma_H\in\intervalcc{0.45}{0.85}$, which resulted in about equal classification results over the entire range.
Dropping of sparse connectivity increases misclassifications, for the median error of SOAE-$\sigma$-conn is $0.81\%$ and thereby greater than the median error of SOAE-$\sigma$.

\begin{table}[t]
  \centering
  \renewcommand{\arraystretch}{1.075}
  \begin{tabular}{ccccc}
    \toprule
    Approach & Sparse & Sparse & Result $(W,\ p)$ of & Evaluation\\
             & Connectivity & Activity & Shapiro-Wilk Test & Error [\%]\\\midrule
    SOAE-$\sigma$ & $\sigma_W = 0.75$ & $\sigma_H\in\intervalcc{0.45}{0.85}$ & $(0.9802,\ 0.60)$ & $0.75\pm 0.04$\\
    SOAE-$L_0$ & $\sigma_W = 0.75$ & $\kappa\in\intervalcc{242}{558}$ & $(0.9786,\ 0.53)$ & $0.82\pm 0.05$\\
    SOAE-$\sigma$-conn & none & $\sigma_H\in\intervalcc{0.45}{0.85}$ & $(0.9747,\ 0.40)$ & $0.81\pm 0.04$\\
    SMLP-SCFC & $\sigma_W = 0.75$ & none & $(0.9770,\ 0.47)$ & $0.81\pm 0.05$\\
    MLP-OBD & $\gamma = 12.5\%$ & none & $(0.9807,\ 0.62)$ & $0.89\pm 0.04$\\
    MLP-random & none & none & $(0.9798,\ 0.58)$ & $0.88\pm 0.03$\\
    MLP-samples & none & none & $(0.9773,\ 0.49)$ & $0.91\pm 0.05$\\
    MLP-SCFC & none & none & $(0.9794,\ 0.57)$ & $0.91\pm 0.06$\\\bottomrule
  \end{tabular}
  \caption{Overview of comparative experiments. The second and third columns indicate whether sparse connectivity or sparse activity was incorporated, respectively. The fourth column reports the result of a statistical test for normality, which is interpreted in Section~\ref{sect:statistical_analysis}. The final column gives the median $\pm$ one standard deviation of the achieved classification error on the MNIST evaluation set. The results for each experiment were trimmed to gain a sample of size $47$, allowing for statistical robust estimates.}
  \label{tbl:soae_comparatative_results}
\end{table}

The other five approaches included in the comparison are multi-layer perceptrons~(MLPs) with the same topology and dynamics as the classification module of supervised online auto-encoder, with two exceptions.
First, the transfer function of the hidden layer $f$ was set to a hyperbolic tangent, thus not including explicit sparse activity.
Second, in all but one experiment sparse connectivity was either not incorporated, or achieved through other means than by performing a $\sigma$ projection after each learning epoch.
Besides the variation in sparseness of connectivity, the experiments differ in the initialization of the network parameters.

For each variant, $55$ runs were carried out and the resulting classifiers were applied to the evaluation set to compute the classification error.
Then, the best four and the worst four results were discarded and not included in further analysis.
Hence a random sample of size $47$ was achieved, where $15\%$ of the original data were trimmed away.
This procedure was also applied to the results of SOAE-$\sigma$, SOAE-$\sigma$-conn, and SOAE-$L_0$, to obtain a total of eight random samples of equal size for comparison with another.

The most basic variant, denoted the baseline in this discussion, is MLP-random, where all network parameters were initialized randomly.
This achieved a median error of $0.88\%$ on the evaluation set, being considerably worse than SOAE-$\sigma$.
For variant MLP-samples, the hidden layer was initialized by replication of $n$ randomly chosen samples from the learning set.
This did decrease the overall learning time.
However, the median classification error was slightly worse with $0.91\%$ compared to MLP-random.

For variant MLP-SCFC, the network parameters were initialized in an unsupervised manner using the sparse coding for fast classification~(SCFC) algorithm \citep{Thom2011d}.
This method is a precursor to the SOAE proposed in this paper.
It also features sparse connectivity and sparse activity but differs in some essential parts.
First, sparseness of activity is achieved through a latent variable that stores the optimal sparse code words of all samples simultaneously.
Using this matrix of code words, the activity of individual units was enforced to be sparse over time on the entire learning set.
SOAE achieves sparseness over space, as for each sample only a pre-defined fraction of units is allowed to be active at any one time.
A second difference is that sparse activity is achieved only indirectly by approximation of the latent matrix of code words with a feed-forward representation.
With SOAE, sparseness of activity is guaranteed by construction.
MLP-SCFC achieved a median classification error of $0.91\%$ on the MNIST evaluation set, rendering it slightly worse than MLP-random and equivalent to MLP-samples.

The first experiment that incorporates only sparse connectivity is SMLP-SCFC.
Initialization was done as for MLP-SCFC, but during training sparseness of connectivity was yielded by application of the sparseness-enforcing projection operator to the weights of the hidden layer after every learning epoch.
Hence the sparseness gained from unsupervised initialization was retained.
MLP-SCFC features sparse connectivity only after initialization, but loses this property when training proceeds.
With this slight modification, the median error of SMLP-SCFC decreases to $0.81\%$, which is significantly better than the baseline result.

The effect of better generalization due to sparse connectivity has also been observed by \citet{LeCun1990} in the context of convolutional neural networks.
It can be explained by the bias-variance decomposition of the generalization error \citep{Geman1992}.
When the effective number of the degrees of freedom is constrained, overfitting will be less likely and hence classifiers produce better results on average.
The same argument can be applied to SOAE-$\sigma$, where additional sparse activity further improves classification results.

The last variant is called MLP-OBD.
Here, the optimal brain damage~(OBD) algorithm \citep{LeCun1990} was used to prune synaptic connections in the hidden layer that are irrelevant for the computation of the classification decision of the network.
The parameters of the network were first initialized randomly and then optimized on the learning set.
Then the impact for each synaptic connection on the objective function was estimated using the Taylor series of the objective function, where a diagonal approximation of the Hessian was employed and terms of cubic or higher order were neglected.
Using this information, the number of connections was halved by setting the weight of connections with low impact to zero.
The network was then retrained with weights of removed connections kept at zero.
This procedure was repeated until a target percentage $\gamma$ of active synaptic connections in the hidden layer was achieved.
For the results reported here, $\gamma = 12.5\%$ was chosen as this reflects the sparse connectivity $\sigma_W = 0.75$ of the other approaches best.
MLP-OBD achieved a median classification error of $0.89\%$, which is comparable to the baseline result.

\subsection{Statistical Analysis and Conclusions}
\label{sect:statistical_analysis}
A statistical analysis was carried out to assess the significance of the differences in the performance of the eight algorithms.
The procedure follows the proposals of \citet{Pizarro2002} and \citet{Demsar2006} for hypothesis testing, and is concluded by effect size estimation as proposed by \citet{Grissom1994} and \citet{Acion2006}.
For each algorithm, a sample of size $47$ was available, allowing for robust analysis results.

First, all results were tested for normality using the test developed by \citet{Shapiro1965}.
The resulting test statistics $W$ and $p$-values are given in Table~\ref{tbl:soae_comparatative_results}.
As all $p$-values are large, it cannot be rejected that the samples came from normally distributed populations.
Thus normality is assumed in the remainder of this discussion.
Next, the test proposed by \citet{Levene1960} was applied to determine whether equality of variances of the groups holds.
This resulted in a test statistic $F = 2.7979$ with $7$ and $368$ degrees of freedom, and therefore a $p$-value of $0.0075$.
Hence the hypothesis that all group variances are equal can be rejected with very high significance.
Consequently, parametric omnibus and post-hoc tests cannot be applied, as they require the groups to have equal variance.

As an alternative, the nonparametric test by \citet{Kruskal1952} which is based on rank information was employed to test whether all algorithms produced classifiers with equal classification errors in the mean.
The test statistic was $H = 214.44$ with $7$ degrees of freedom, and the $p$-value was less than $10^{-15}$.
There is hence a statistically significant difference in the mean classification results.
To locate this deviation, a critical difference for comparing the mean ranks of the algorithms was computed.
A Tukey-Kramer type modification applied to Dunn's procedure yields this critical difference, which is less conservative than Nemenyi's procedure for the Kruskal-Wallis test \citep{Hochberg1987}.
Note that this approach is nevertheless similar to the post-hoc procedure proposed by \citet{Demsar2006} for paired observations, such that the diagrams proposed there can be adapted to the case for unpaired observations.
The result is depicted in Figure~\ref{fig:demsar}, where the critical difference for statistical significance at the $\alpha = 0.01$ level is given.
This test induces a highly significant partitioning of the eight algorithms, namely three groups $A$, $B$ and $C$ given by
\begin{gather*}
  A := \set{\text{SOAE-$\sigma$}}\text{, }\;
  B := \set{\text{SOAE-$\sigma$-conn},\ \text{SOAE-$L_0$},\ \text{SMLP-SCFC}}\text{,}\\
  \text{and }C := \set{\text{MLP-OBD},\ \text{MLP-random},\ \text{MLP-samples},\ \text{MLP-SCFC}}\text{.}
\end{gather*}
This partition in turn induces an equivalence relation.
Statistical equivalence is hence unambiguous and well-defined at $\alpha = 0.01$.
Moreover, the $p$-value for this partition is $0.007$.
If the significance level $\alpha$ would have been set lower than this, then groups $A$ and $B$ would blend together.

\begin{figure}[t]
  \centering
  \includegraphics[page=7]{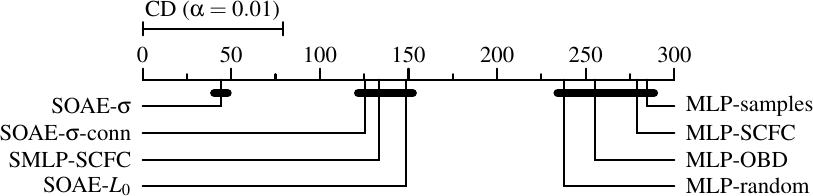}
  \caption{Diagram for multiple comparison of algorithms following \citet{Demsar2006}. For each algorithm, the mean rank was computed during the Kruskal-Wallis test. Then, a critical difference~(CD) was computed at the $\alpha = 0.01$ significance level. Two algorithms produce classification results that are statistically not equal if the difference between their mean ranks is greater than the critical difference. This induced three groups of algorithms that produced statistically equivalent results, which are marked with black bars.}
  \label{fig:demsar}
\end{figure}

To assess the benefit when an algorithm from one group is chosen over an algorithm from another group, the probability of superior experiment outcome was estimated \citep{Grissom1994,Acion2006}.
For this, the classification errors were pooled with respect to membership in the three groups.
It was then tested whether these pooled results still come from normal distributions.
As group $A$ is a singleton, this is trivially fulfilled with the result from Table~\ref{tbl:soae_comparatative_results}.
For group $B$, the Shapiro-Wilk test statistic was $W = 0.9845$ and the $p$-value was $0.11$.
Group $C$ achieved a test statistic of $W = 0.9882$ and a $p$-value of $0.12$.
If a standard significance level of $\alpha = 0.01$ is chosen, then $B$ and $C$ can be assumed to be normally distributed also.

Let $E_G$ be the random variable modeling the classification results of the algorithms from group $G\in\set{A, B, C}$.
It is assumed that $E_G$ is normally distributed with unknown mean and unknown variance for all $G$.
Then $E_G - E_{\tilde{G}}$ is clearly normally distributed also for two groups $G,\tilde{G}\in\set{A, B, C}$.
Therefore, the probability $P(E_G < E_{\tilde{G}})$ that one algorithm produces a better classifier than another could be computed from the Gaussian error function if the group means and variances were known.
However, using Rao-Blackwell theory a minimum variance unbiased estimator $\hat{R}_2$ of this probability can be computed easily \citep{Downton1973}.
Evaluation of the expression for $\hat{R}_2$ shows that $P(E_A < E_B)$ can be estimated by $0.87$, $P(E_B < E_C)$ can be estimated by $0.88$, and $P(E_A < E_C)$ can be estimated by $0.99$.
Therefore, the effect of choosing SOAE-$\sigma$ over any of the seven other algorithms is dramatic \citep{Grissom1994}.

These results can be interpreted as follows.
When neither sparse activity nor sparse connectivity is incorporated, then the worst classification results are obtained regardless of the initialization of the network parameters.
The exception is MLP-OBD which incorporates sparse connectivity, although, as its name says, in a destructive way.
Once a synaptic connection has been removed, it cannot be recovered, as the measure for relevance of \citet{LeCun1990} vanishes for synaptic connections of zero strength.
The statistics for SMLP-SCFC shows that when sparse connectivity is obtained using the sparseness-enforcing projection operator, then superior results can be achieved.
Because of the nature of projected gradient descent, it is possible here to restore deleted connections if it helps to decrease the classification error during learning.
For SOAE-$\sigma$-conn only sparse activity was used, and classification results were statistically equivalent to SMLP-SCFC.

Therefore, using either sparse activity or sparse connectivity improves classification capabilities.
When both are used, then results improve even more as variant SOAE-$\sigma$ shows.
This does not hold for SOAE-$L_0$ however, where the $L_0$ projection was used as transfer function.
As Hoyer's sparseness measure $\sigma$ and the according projection possess desirable analytical properties, they can be considered smooth approximations to the $L_0$ pseudo-norm.
It is this smoothness which seems to produce this benefit in practice.

\section{Related Work}
This section reviews work related with the contents of this paper.
First, the theoretical foundations of the sparseness-enforcing projection operator are discussed.
Next, its application as neuronal transfer function to achieve sparse activity in a classification scenario is put in context with alternative approaches, and possible advantages of sparse connectivity are described.

\subsection{Sparseness-Enforcing Projection Operator}
\label{sect:relwork_projfunc}
The first major part of this paper dealt with improvements to the work of \citet{Hoyer2004} and \citet{Theis2005}.
Here, an algorithm for the sparseness-enforcing projection with respect to Hoyer's sparseness measure $\sigma$ was proposed.
The technical proof of correctness is given in Appendix~\ref{sect:projfuncproof}.
The set that should be projected onto is an intersection of a simplex $C$ and a hypercircle $L$, which is a hypersphere lying in a hyperplane.
The overall procedure can be described as performing alternating projections onto $C$ and certain subsets of $L$.
This approach is common for handling projections onto intersections of individual sets.
For example, \citet{Neumann1950} proposed essentially the same idea when the investigated sets are closed subspaces, and has shown that this converges to a solution.
A similar approach can be carried out for intersections of closed, convex cones \citep{Dykstra1983}, which can be generalized to translated cones that can be used to approximate any convex set \citep{Dykstra1987}.
For these alternating methods, it is only necessary to know how projections onto individual members of the intersection can be achieved.

Although these methods exhibit great generality, they have two severe drawbacks in the scenario of this paper.
First, the target set for projection must be an intersection of convex sets.
The scaled canonical simplex $C$ is clearly convex, but the hypercircle $L$ is non-convex if it contains more than one point.
The condition that generates $L$ cannot easily be weakened to achieve convexity.
If the original hypersphere were replaced with a closed ball, then $L$ would be convex.
But this changes the meaning of the problem dramatically, as now virtually any sparseness below the original target degree of sparseness can be obtained.
This is because when the target $L_1$ norm $\lambda_1$ is fixed, the sparseness measure $\sigma$ decreases whenever the target $L_2$ norm decreases.
In geometric terms, the method proposed in this paper performs a projection from within a circle onto its boundary to increase the sparseness of the working vector.
This argument is given in more detail in Figure~\ref{fig:simplex_proj_rec} and the proof of Lemma~\ref{lem:splx}\ref{lem:splx_f}.

The second drawback of the general methods for projecting onto intersections is that a solution is only achieved asymptotically, even when the convexity requirements are fulfilled.
Due to the special structure of $C$ and $L$, the number of alternating projections that have to be carried out to find a solution using Algorithm~\ref{alg:projfunc_explicit} is bounded from above by the problem dimensionality.
Thus an exact projection is always found in finite time.
Furthermore, the solution is guaranteed to be found in time that is at most quadratic in problem dimensionality.

A crucial point is the computation of the projection onto $C$ and certain subsets of $L$.
Due to the nature of the $L_2$ norm, the latter is straightforward.
For the former, efficient algorithms have been proposed recently \citep{Duchi2008,Chen2011}.
When only independent solutions are required, the projection of a point $x$ onto a scaled canonical simplex of $L_1$ norm $\lambda_1$ can also be carried out in linear time \citep{Liu2009a}, without having to sort the vector that is to be projected.
This can be achieved by showing that the separator $\hat{t}$ for performing the simplex projection is the unique zero of the monotonically decreasing function $t \mapsto \norm{\max\left(\abs{x} - t\cdot e,\ 0\right)}_1 - \lambda_1$.
The zero of this function can be found efficiently using the bisection method, and exploiting the special structure of the occurring expressions \citep{Liu2009a}.

In the context of this paper an explicit closed-form expression for $\hat{t}$ is preferable as it permits additional insight into the properties of the projected point.
The major part in proving the correctness of Algorithm~\ref{alg:projfunc} is the interconnection between $C$ and $L$, that is that the final solution has zero entries at the according positions in the working vector and thus a chain monotonically decreasing in $L_0$ pseudo-norm is achieved.
This result is established through Lemma~\ref{lem:subsplx_proj}, which characterizes projections onto certain faces of a simplex, Corollary~\ref{cor:subsplx_proj} and their application in Lemma~\ref{lem:splx}.

Analysis of the theoretical properties of the sparseness-enforcing projection is concluded with its differentiability in Appendix~\ref{sect:analytical_properties}.
The idea is to exploit the finiteness of the projection sequence and to apply the chain rule of differential calculus.
It is necessary to show that the projection chain is robust in a neighborhood of the argument.
This reduces analysis to individual projection steps which have already been studied in the literature.
For example, the projection onto a closed, convex set is guaranteed to be differentiable almost everywhere \citep{Hiriart-Urruty1982}.
Here non-convexity of $L$ is not an issue, as the only critical point is its barycenter.
For the simplex $C$, a characterization of critical points is given with Lemma~\ref{lem:splx_proj_analytics} and Lemma~\ref{lem:splx_proj_critical}, and it is shown that the expression for the projection onto $C$ is invariant to local changes.
An explicit expression for construction of the gradient of the sparseness-enforcing projection operator is given in Theorem~\ref{thm:projfuncblockgrad}.
In Corollary~\ref{cor:projfuncgraddgemv} it is shown that the computation of the product of the gradient with an arbitrary vector can be achieved efficiently by exploiting sparseness and the special structure of the gradient.

Similar approaches for sparseness projections are discussed in the following.
The iterative hard thresholding algorithm is a gradient descent algorithm, where a projection onto an $L_0$ pseudo-norm constraint is performed \citep{Blumensath2009}.
Its application lies in compressed sensing, where a linear generative model is used to infer a sparse representation for a given observation.
Sparseness here acts as regularizer which is necessary because observations are sampled below the Nyquist rate.
In spite of the simplicity of the method, it can be shown that it achieves a good approximation to the optimal solution of this NP-hard problem \citep{Blumensath2009}.

Closely related with the work of this paper is the generalization of Hoyer's sparseness measure by \citet{Theis2006}.
Here, the $L_1$ norm constraint is replaced with a generalized $L_p$ pseudo-norm constraint, such that the sparseness measure becomes $\sigma_p(x) := \nicefrac{\norm{x}_p}{\norm{x}_2}$.
For $p = 1$, Hoyer's sparseness measure up to a constant normalization is obtained.
When $p$ converges decreasingly to zero, then $\sigma_p(x)^p$ converges point-wise to the $L_0$ pseudo-norm.
Hence for small values of $p$ a more natural sparseness measure is obtained.
\citet{Theis2006} also proposed an extension of Hoyer's projection algorithm.
It is essentially von Neumann's alternating projection method, where closed subspaces have been replaced by "spheres" that are induced by $L_p$ pseudo-norms.
Note that these sets are non-convex when $p < 1$, such that convergence is not guaranteed.
Further, no closed-form solution for the projection onto an "$L_p$-sphere" is known for $p\not\in\set{1, 2, \infty}$, such that numerical methods have to be employed.

A problem where similar projections are employed is to minimize a convex function subject to group sparseness \citep[see for example][]{Friedman2010}.
In this context, mixed norm balls are of particular interest \citep{Sra2012}.
For a matrix $X\in\R^{n\times g}$, the mixed $L_{p,q}$ norm is defined as the $L_p$ norm of the $L_q$ norms of the columns of $X$, that is $\norm{X}_{p, q} := \bnorm{\tdvectbig{\norm{Xe_1}_q}{\norm{Xe_g}_q}}_p$.
Here, $X$ can be interpreted to be a data point with entries partitioned into $g$ groups.
When $p = 1$, then the projection onto a simplex can be generalized directly for $q = 2$ \citep{Berg2008} and for $q = \infty$ \citep{Quattoni2009}.
The case when $p = 1$ and $q\geq 1$ is more difficult, but can be solved as well \citep{Liu2010,Sra2012}.

The last problem discussed here is the elastic net criterion \citep{Zou2005}, which is a constraint on the sum of an $L_1$ norm and an $L_2$ norm.
The feasible set can be written as the convex set $N := \set{s\in\R^n | \lambda_1 \norm{s}_1 + \lambda_2\norm{s}_2^2 \leq 1}$, where $\lambda_1,\lambda_2 \geq 0$ control the shape of $N$.
Note that in $N$ only the sum of two norms is considered, whereas the non-convex set $S^{(\lambda_1,\lambda_2)}$ consists of the intersection of two different constraints.
Therefore, the elastic net induces a different notion of sparseness than Hoyer's sparseness measure $\sigma$ does.
As is the case for mixed norm balls, the projection onto a simplex can be generalized to achieve projections onto $N$ \citep{Mairal2010}.

\subsection{Supervised Online Auto-Encoder}
The sparseness-enforcing projection operator $\pi$ with respect to Hoyer's sparseness measure $\sigma$ and the projection onto an $L_0$ pseudo-norm constraint are differentiable almost everywhere.
Thus they are suitable for gradient-based optimization algorithms.
In Section~\ref{sect:soae}, they were used as transfer functions in a hybrid of an auto-encoder network and a two-layer neural network to infer a sparse internal representation.
This representation was subsequently employed to approximate the input sample and to compute a classification decision.
In addition, the matrix of bases which was used to compute the internal representation was enforced to be sparsely populated by application of the sparseness projection after each learning epoch.
Hence the supervised online auto-encoder proposed in this paper features both sparse activity and sparse connectivity.

These two key properties have also been investigated and exploited in the context of auto-associative memories for binary inputs.
If the entries of the training patterns are sparsely populated, the weight matrix of the memory will be sparsely populated as well after training if Hebbian-like learning rules are used \citep{Kohonen1972}.
The assumption of sparsely coded inputs also results in increased completion capacity and noise resistance of the associative memory \citep{Palm1980}.
If the input data is not sparse inherently, feature detectors can perform a sparsification prior to the actual processing through the memory \citep{Baum1988}.

A purely generative model that also possesses these two key properties is non-negative matrix factorization with sparseness constraints \citep{Hoyer2004}.
This is an extension to plain non-negative matrix factorization \citep{Paatero1994} which was shown to achieve sparse connectivity on certain data sets \citep{Lee1999}.
However, there are data sets on which this does not work \citep{Li2001,Hoyer2004}.
Although Hoyer's model makes sparseness easily controllable by explicit constraints, it is not inherently suited to classification tasks.
An extension intended to incorporate class membership information to increase discriminative capabilities was proposed by \citet{Heiler2006}.
In their approach, an additional constraint was added ensuring that every internal representation is close to the mean of all internal representations that belong to the same class.
In other words, the method can be interpreted as supervised clustering, with the number of clusters equal to the number of classes.
However, there is no guarantee that a distribution of internal representations exists such that both the reproduction error is minimized and the internal representations can be arranged in such a pattern.
Unfortunately, \citet{Heiler2006}~used only a subset of a small data set for handwritten digit recognition to evaluate their approach.

A precursor to the supervised online auto-encoder was proposed by \citet{Thom2011d}.
There, inference of sparse internal representations was achieved by fitting a one-layer neural network to approximate a latent variable of optimal sparse representations.
The transfer function used for this approximation was a hyperbolic tangent raised to an odd power greater or equal to three.
This resulted in a depression of activities with small magnitude, favoring sparseness of the result.
Similar techniques to achieve a shrinkage-like effect for increasing sparseness of activity in a neural network were used by \citet{Gregor2010} and \citet{Glorot2011}.
Information processing is here purely local, that is a scalar function is evaluated entrywise on a vector, and thus no information is interchanged among individual entries.

The use of non-local shrinkage to reduce Gaussian noise in sparse coding has already been described by \citet{Hyvaerinen1999a}.
Here, a maximum likelihood estimate with only weak assumptions yields a shrinkage operation, which can be conceived as projection onto a scaled canonical simplex.
In the use case of object recognition, a hard shrinkage was also employed to de-noise filter responses \citep{Mutch2006}.
Whenever a best approximation from a permutation-invariant set is used, a shrinkage-like operation must be employed.
Using a projection operator as neural transfer function is hence a natural extension of these ideas.
When the projection is sufficiently smooth, the entire model can be tuned end-to-end using gradient methods to achieve an auto-encoder or a classifier.

The second building block from \citet{Thom2011d} that was incorporated into supervised online auto-encoder is the architectural concept for classification.
It is well-known that two layers in a neural network are sufficient to approximate any continuous function on a compactum with arbitrary precision \citep{Cybenko1989,Funahashi1989,Hornik1989}.
Similar architectures have also been proposed for classification in combination with sparse coding of the inputs.
However, sparse connectivity was not considered in this context.
\citet{Bradley2009} used the Kullback-Leibler divergence as implicit sparseness penalty term and combined this with the backpropagation algorithm to yield a classifier that achieved a $1.30\%$ error rate on the MNIST evaluation set.
The Kullback-Leibler divergence was chosen to replace the usual $L_1$ norm penalty term, as it is smoother than the latter and therefore sparsely coded internal representations are more stable subject to subtle changes of the input.
A related technique is supervised dictionary learning by \citet{Mairal2009b}, where the objective function is an additive combination of a classification error term, a term for the reproduction error, and an $L_1$ norm constraint.
Inference of sparse internal representations is achieved through solving an optimization problem.
Such procedures are time-consuming and greatly increase the computational complexity of classification.
With this approach, a classification error of $1.05\%$ on the MNIST evaluation set was achieved.
These two approaches used the original MNIST learning set without jittering the digits and can thus be considered permutation-invariant.
Augmentation of the learning set with virtual samples would have contributed to improve classification performance, as demonstrated by \citet{Schoelkopf1997}.

Finally consider once more the sparse connectivity property, which is mostly neglected in the literature in favor of sparse activity.
It was shown in this paper that sparse connectivity helps to improve generalization capabilities.
In practice, this property can also be used to reduce the computational complexity of classification by one order of magnitude \citep{Thom2011e}.
This results from exploiting sparseness and using sparse matrix-vector multiplication algorithms to infer the internal representation, which is the major computational burden in class membership prediction.
It was shown in this paper and by \citet{Thom2011e} that a small number of nonzero entries in the weight matrix of the hidden layer is sufficient for achieving good classification results.
Furthermore, the additional savings in required storage capacity and bandwidth allow using platforms with modest computational power for practical implementations.
Sparseness is therefore an elementary concept of efficiency in artificial processing systems.

\section{Conclusions}
Without sparseness in their brains, higher mammals probably would not have developed to viable life-forms.
This important concept of efficiency was discovered by neuroscientists, and practical benefit was obtained by the engineers of artificial information processing systems.
This paper studied Hoyer's sparseness measure $\sigma$, and in particular the projection of arbitrary vectors onto sets where $\sigma$ attains a constant value.
A simple yet efficient algorithm for computing this sparseness-enforcing projection operator was proposed in this paper, and its correctness was proved.
In addition, it was demonstrated that the proposed algorithm is superior in run-time to Hoyer's original algorithm.
The analysis of the theoretical properties of this projection was concluded by showing it is differentiable almost everywhere.

As projections onto $\sigma$ constraints are well-understood, they constitute the ideal tool for building systems that can benefit from sparseness constraints.
An original use case was introduced in this paper.
Here, the $\sigma$ projection was implemented as neuronal transfer function, yielding a differentiable closed-form expression for inference of sparse code words.
Besides this sparse activity, the connectivity in this system was also forced to be sparse by performing the $\sigma$ projection after the presentation of learning examples.
Because of its smoothness, the entire system can be optimized end-to-end by gradient-based methods, yielding a classification architecture exhibiting true sparse information processing.

This supervised online auto-encoder was applied on a benchmark data set for pattern recognition.
Because sparseness constraints reduce the amount of feasible solutions, it is not clear in the first place whether the same performance can be achieved at all.
However, when the target degree of sparseness of the activity is in a reasonable range, classification results are not only equivalent but superior to classical non-sparse approaches.
This result is supported by statistical evaluation showing that this performance increase is not merely coincidental, but statistically significant.
Therefore, sparseness can be seen as regularizer that offers the potential to improve artificial systems in the same way it seems to improve biological systems.

\acks{The authors wish to thank Patrik O. Hoyer and Xiaojing Ye for sharing the source code of their algorithms. The authors are also grateful to the anonymous reviewers for their valuable comments and feedback. This work was supported by Daimler AG, Germany.}

\appendix
\section{Notation and Prerequisites}
\label{sect:notation}
This appendix fixes the notation and provides prerequisites for the following appendices.
$\N$ denotes the natural numbers including zero, $\R$ the real numbers and $\R_{\geq 0}$ the non-negative real numbers.
$\R^n$ is the $n$-dimensional Euclidean space with canonical basis $e_1,\dotsc,e_n\in\R^n$, and $e := \sum_{i=1}^ne_i\in\R^n$ denotes the vector where all entries are identical to unity.
For all other vectors, a subscript denotes the corresponding entry of the vector, that is $x_i = e_i\transp x$ for $x\in\R^n$.
The amount of nonzero entries in a vector is given by the $L_0$ pseudo-norm, $\norm{\cdot}_0$.
$\norm{\cdot}_1$ and $\norm{\cdot}_2$ denote the Manhattan norm and Euclidean norm, respectively.
$\scp{\cdot}{\cdot}$ denotes the canonical dot product in the Euclidean space.
Given a vector $x$, $\diag(x)$ denotes the square matrix with $x$ on its main diagonal and zero entries at all other positions, and $a\hada b = \diag(a)b$ denotes the Hadamard product or entrywise product for vectors $a$ and $b$.
When $A$ and $B$ are square matrices, then $\diag(A, B)$ denotes the block diagonal matrix with the blocks given by $A$ and $B$.
$S_n$ is the symmetric group, and $P_\tau$ denotes the permutation matrix for $\tau\in S_n$.
For a set $M\subseteq U$, $M\comp$ denotes its complement in the universal set $U$, where $U\in\set{\R^n, \discint{1}{n}}$ is clear from the context.
The power set of $M$ is denoted by $\P(M)$.
If $M\subseteq\R^n$, then $\partial M$ denotes its boundary in the topological sense.
The sign function is denoted by $\sgn(\cdot)$.
A list of symbols that are frequently used throughout the paper is given in Table~\ref{tab:symbols}.
\begin{table}[t]
  \centering
  \renewcommand{\arraystretch}{1.185}
  \begin{tabular}{l l}
    \toprule
    Symbol and Definition & Meaning\\\midrule
    $\sigma$ (see Section~\ref{sect:intro_hoyer}) & Sparseness measure by \citet{Hoyer2004}\\
    $\pi$ and $\pi_{\geq 0}$ (see Section~\ref{sect:projfunc_differentiability}) & Sparseness projection cast as function\\\midrule
    $n\in\N$ & Problem dimensionality\\
    $e_1,\dotsc,e_n\in\R^n$ & Canonical basis of $\R^n$\\
    $e := \sum_{i=1}^ne_i\in\R^n$ & Vector where all entries are one\\
    $\lambda_1\in\R_{> 0}$ & Target $L_1$ or Manhattan norm\\
    $\lambda_2\in\R_{> 0}$ & Target $L_2$ or Euclidean norm\\
    $S^{(\lambda_1,\lambda_2)} \subseteq \R^n$ (see Section~\ref{sect:intro_hoyer}) & Target set for sparseness projection\\
    $S_{\geq 0}^{(\lambda_1,\lambda_2)} := S^{(\lambda_1,\lambda_2)} \cap \R_{\geq 0}^n$ & Target set for non-negative sparseness projection\\\midrule
    $D := S_{\geq 0}^{(\lambda_1,\lambda_2)}$ & Short for the non-negative target set\\
    $H := \set{a\in\R^n | e\transp a = \lambda_1}$ & Target hyperplane\\
    $K := \set{q\in\R^n | \norm{q}_2 = \lambda_2}$ & Target hypersphere\\
    $L := H\cap K$ & Target hypercircle\\
    $C := \R_{\geq 0}^n\cap H$ & Scaled canonical simplex\\
    $m := \nicefrac{\lambda_1}{n}\cdot e\in\R^n$ & Barycenter of $L$ and $C$\\
    $\rho := \lambda_2^2 - \nicefrac{\lambda_1^2}{n}\in\R$ & Squared radius of $L$\\\midrule
    $I\subseteq\discint{1}{n}$ & Index set of nonzero entries\\
    $d := \abs{I}\in\N$ & Working dimensionality\\
    $L_I := \set{a\in L | a_i = 0\text{ for all }i\not\in I}$ & Points in $L$ where certain coordinates vanish\\
    $C_I := \set{c\in C | c_i = 0\text{ for all }i\not\in I}$ & Face of simplex $C$\\
    $m_I := \nicefrac{\lambda_1}{d}\cdot\sum_{i\in I}e_i\in\R^n$ & Barycenter of $L_I$ and $C_I$\\
    $\rho_I := \lambda_2^2 - \nicefrac{\lambda_1^2}{d}\in\R$ & Squared radius of $L_I$\\\bottomrule
  \end{tabular}
  \caption{A list of symbols used frequently in this paper and their meaning.}
  \label{tab:symbols}
\end{table}

The important concept of the projection onto a set was given in Definition~\ref{dfn:projection}.
The following basic statement will be used extensively in this paper and follows from $\scp{x}{x} = \norm{x}_2^2$ for all $x\in\R^n$ and the fact that the scalar product is a symmetric bilinear form \citep{Laub2004}:
\begin{proposition}
\label{prop:sqnrm}
Let $a,b\in\R^n$.
Then $\norm{a\pm b}_2^2 = \norm{a}_2^2 + \norm{b}_2^2 \pm 2\scp{a}{b}$.
Further it is $\norm{a - b}_2^2 = \norm{a - p}_2^2 + \norm{p - b}_2^2 + 2\scp{a - p}{p - b}$ for all $p\in\R^n$.
\end{proposition}
As an example, note that the outcome of the sparseness-enforcing projection operator depends only on the target sparseness degree up to scaling:
\begin{remark}
\label{rem:projfunc_scaling}
Let $\lambda_1, \lambda_2 > 0$ and $\tilde{\lambda}_1, \tilde{\lambda}_2 > 0$ be pairs of target norms such that $\nicefrac{\lambda_1}{\lambda_2} = \nicefrac{\tilde{\lambda}_1}{\tilde{\lambda}_2}$.
Then
\begin{displaymath}
  \proj_{S^{(\lambda_1,\lambda_2)}}(x) = \nicefrac{\tilde{\lambda}_2}{\lambda_2}\cdot\proj_{S^{(\tilde{\lambda}_1,\tilde{\lambda}_2)}}(x)\text{ for all }x\in\R^n\text{.}
\end{displaymath}
\end{remark}
\begin{proof}
It is sufficient to show only one inclusion.
Let $x\in\R^n$ be arbitrary, $p\in\proj_{S^{(\lambda_1,\lambda_2)}}(x)$ and $\tilde{r}\in S^{(\tilde{\lambda}_1,\tilde{\lambda}_2)}$.
Define $\tilde{p} := \nicefrac{\tilde{\lambda}_1}{\lambda_1}\cdot p = \nicefrac{\tilde{\lambda}_2}{\lambda_2}\cdot p\in\R^n$, then $\norm{\tilde{p}}_1 = \abs{\nicefrac{\tilde{\lambda}_1}{\lambda_1}}\cdot\norm{p}_1 = \tilde{\lambda}_1$ and analogously $\norm{\tilde{p}}_2 = \tilde{\lambda}_2$, hence $\tilde{p}\in S^{(\tilde{\lambda}_1,\tilde{\lambda}_2)}$.
For the claim to hold it has now to be shown that $\norm{\tilde{p} - x}_2 \leq \norm{\tilde{r} - x}_2$.
Write $r := \nicefrac{\lambda_2}{\tilde{\lambda}_2}\cdot\tilde{r}\in\R^n$, which in fact lies in $S^{(\lambda_1,\lambda_2)}$.
So $\norm{p - x}_2 \leq \norm{r - x}_2$ by definition of $p$,
and with Proposition~\ref{prop:sqnrm} follows
$\norm{\tilde{r} - x}_2^2 - \norm{\tilde{p} - x}_2^2
= \norm{\tilde{r}}_2^2 + \norm{x}_2^2 - 2\scp{\tilde{r}}{x} - \norm{\tilde{p}}_2^2 - \norm{x}_2^2 + 2\scp{\tilde{p}}{x}
= 2\scp{\tilde{p} - \tilde{r}}{x}
= \nicefrac{\tilde{\lambda}_2}{\lambda_2}\cdot 2\scp{p - r}{x}
= \nicefrac{\tilde{\lambda}_2}{\lambda_2}\cdot\big(\norm{r - x}_2^2 - \norm{p - x}_2^2\big)
\geq 0$.
\end{proof}
Hence only the ratio of the target $L_1$ norm to the target $L_2$ norm is important and not their actual scale.
This argument can be generalized to projections onto any scale-invariant set and therefore naturally holds also for $S_{\geq 0}^{(\lambda_1,\lambda_2)}$.

\section{Projections onto Symmetric Sets}
\label{sect:proj_symm}
This appendix investigates certain symmetries of sets and their effect on projections onto such sets.
A great variety of sparseness measures fulfills certain symmetries as vector entries are equally weighted, see \citet{Hurley2009}.
This means that no entry is preferred over another, and for negative entries usually the absolute value or the squared value is taken, such that the signs of the entries are ignored.
Consider the following definition of symmetries that are to be analyzed:
\begin{definition}
Let $\emptyset\neq M\subseteq\R^n$.
Then $M$ is called \emph{permutation-invariant} if and only if $P_\tau x\in M$ for all $x\in M$ and all permutations $\tau\in S_n$.
Further, $M$ is called \emph{reflection-invariant} if and only if $b\hada x\in M$ for all $x\in M$ and all $b\in\set{\pm 1}^n$.
\end{definition}
In other words, a subset $M$ of the Euclidean space is permutation-invariant if set membership is invariant to permutation of individual coordinates.
$M$ is reflection-invariant if single entries can be negated without violating set membership.
This is equivalent to $x - 2\sum_{i\in I}x_ie_i\in M$ for all $x\in M$ and all index sets $I\subseteq\discint{1}{n}$, which is a condition that is technically easier to handle.
The following observation states that these symmetries are closed under common set operations:
\begin{remark}
\label{rem:symmetries_closed}
Let $\emptyset\neq A,B\subseteq\R^n$.
When $A$ and $B$ are permutation-invariant or reflection-invariant, then so are $A\cup B$, $A\cap B$ and $A^C$.
\end{remark}
The proof is obvious by elementary set algebra.
Now consider the following general properties of functions mapping to power sets:
\begin{definition}
Let $\emptyset\neq M\subseteq\R^n$, let $\P(M)$ be its power set and let $f\colon \R^n\to\P(M)$ be a function.
$f$ is called \emph{order-preserving} if and only if $x_i > x_j$ implies $p_i \geq p_j$ for all $x\in\R^n$, for all $p\in f(x)$ and for all $i,j\inint{1}{n}$.
$f$ is called \emph{absolutely order-preserving} if and only if from $\abs{x_i} > \abs{x_j}$ follows $\abs{p_i} \geq \abs{p_j}$ for all $x\in\R^n$, for all $p\in f(x)$ and for all $i,j\inint{1}{n}$. 
$f$ is called \emph{orthant-preserving} if and only if $\sgn(x_i) = \sgn(p_i)$ or $x_i = 0$ or $p_i = 0$ for all $x\in M$ and all $p\in f(x)$.
\end{definition}
Hence, a function $f$ is order-preserving if the relative order of entries of its arguments does not change upon function evaluation.
Thus if the entries of $x$ are sorted in ascending or descending order, then so are the entries of every vector in $f(x)$.
Orthant-preservation denotes the fact that $x$ and every vector from $f(x)$ are located in the same orthant.
The link between set symmetries and projection properties is established by the following result.
A weaker form of its statements has been described by \citet{Duchi2008} in the special case of a projection onto a simplex.
\begin{lemma}
\label{lem:proj_props}
Let $\emptyset\neq M\subseteq\R^n$ and $p\colon\R^n\to\P(M)$, $x\mapsto\proj_M(x)$.
Then the following holds:
\begin{enumerate}
\item \label{lem:proj_props_a}
When $M$ is permutation-invariant, then $p$ is order-preserving.

\item \label{lem:proj_props_b}
When $M$ is reflection-invariant, then $p$ is orthant-preserving.

\end{enumerate}
\end{lemma}
\begin{proof}
\ref{lem:proj_props_a}
Let $x\in\R^n$ and $p\in\proj_M(x)$.
Let $i,j\inint{1}{n}$ with $x_i > x_j$.
Assume that $p_i < p_j$.
Let $\tau := (i, j)\in S_n$ and $q := P_\tau p$, then $q\in M$ because of $M$ being permutation-invariant.
Consider $d := \norm{p - x}_2^2 - \norm{q - x}_2^2$.
Because $\tau$ is a single transposition, application of Proposition~\ref{prop:sqnrm} yields $d = 2(p_j - p_i)(x_i - x_j)$.
By requirement $d > 0$, which contradicts the minimality of $p$ as being a projection of $x$ onto $M$.
Hence $p_i \geq p_j$ must hold.

\ref{lem:proj_props_b}
Let $x\in\R^n$ and $p\in\proj_M(x)$.
Define $I := \set{i\inint{1}{n} | \sgn(x_i) \neq \sgn(p_i)}$.
The claim holds trivially if $I = \emptyset$.
Assume $I\neq\emptyset$ and define $q := p - 2\sum_{i\in I}p_ie_i$.
It follows $q\in M$ because $M$ is reflection-invariant.
Proposition~\ref{prop:sqnrm} implies $\norm{q}_2^2 = \norm{p}_2^2$, and clearly $\scp{q}{x} = \scp{p}{x} - 2\sum_{i\in I}p_ix_i$.
Therefore application of Proposition~\ref{prop:sqnrm} yields $d := \norm{p - x}_2^2 - \norm{q - x}_2^2 = -4\sum_{i\in I}p_ix_i$.
By the definition of $I$ one obtains $p_ix_i\in\set{-1, 0}$.
Hence would there be an index $i\in I$ with $p_i\neq 0$ and $x_i\neq 0$, then $d > 0$, but $\norm{p - x}_2^2 > \norm{q - x}_2^2$ would contradict the minimality of $p$.
Therefore $I = \set{i\inint{1}{n} | p_i = 0\text{ or } x_i = 0}$, and the claim follows.
\end{proof}
When the projection onto a permutation-invariant set is unique, then equal entries of the argument cause equal entries in the projection:
\begin{remark}
Let $\emptyset\neq M\subseteq\R^n$ be permutation-invariant and $x\in\R^n$.
When $p = \proj_M(x)$ is unique, then $p_i = p_j$ follows for all $i,j\inint{1}{n}$ with $x_i = x_j$.
\end{remark}
\begin{proof}
Let $x\in\R^n$, $p = \proj_M(x)$ and $i,j\inint{1}{n}$ with $x_i = x_j$.
Assume $p_i\neq p_j$ would hold and let $\tau := (i, j)\in S_n$ and $q := P_\tau p \neq p$.
With the permutation-invariance of $M$ follows $q\in M$, and $\norm{q - x}_2 = \norm{p - x}_2$  with $x_i = x_j$.
Hence $q\in\proj_M(x)$, so $q = p$ with the uniqueness of the projection, which contradicts $q\neq p$.
Therefore, $p_i = p_j$.
\end{proof}
The next result shows how solutions to a projection onto reflection-invariant sets can be turned into non-negative solutions and vice-versa.
Its second part was already observed by \citet{Hoyer2004}, in the special case of the sparseness-enforcing projection operator, and by \citet{Duchi2008}, when the connection between projections onto a simplex and onto an $L_1$ ball was studied.
Both did not provide a proof, but in the latter work a hint to a possible proof was given.
With Lemma~\ref{lem:proj_nonneg_sols} it suffices to consider non-negative solutions for projections onto reflection-invariant sets.
\begin{lemma}
\label{lem:proj_nonneg_sols}
Let $\emptyset\neq A\subseteq\R^n$ be reflection-invariant, $B := A\cap\R_{\geq 0}^n$ and $p,x\in\R^n$.
Then:
\begin{enumerate}
\item \label{lem:proj_nonneg_sols_a}
If $p\in\proj_A(x)$, then $\abs{p}\in\proj_B(\abs{x})$.

\item \label{lem:proj_nonneg_sols_b}
If $p\in\proj_B(\abs{x})$, then $s\hada p\in\proj_A(x)$ where $s\in\set{\pm 1}^n$ is given by $s_i := 1$ if $x_i\geq 0$ and $s_i := -1$ otherwise for all $i\inint{1}{n}$.
\end{enumerate}
\end{lemma}
\begin{proof}
First note that if $q\in\proj_A(x)$, then $\sgn(x_i) = \sgn(q_i)$ or $x_i = 0$ or $q_i = 0$ with Lemma~\ref{lem:proj_props}\ref{lem:proj_props_b}.
Hence for all $i\inint{1}{n}$ follows $\left(\abs{q_i} - \abs{x_i}\right)^2 = \left(q_i\cdot\sgn(q_i) - x_i\cdot\sgn(x_i)\right)^2 = \left(q_i - x_i\right)^2$, and therefore $\norm{q - x}_2^2 = \norm{\abs{q} - \abs{x}}_2^2$.
Furthermore, $\abs{q}\in A$ because of $A$ reflection-invariant and $\abs{q}\in\R_{\geq 0}^n$, so $\abs{q}\in B$.

\ref{lem:proj_nonneg_sols_a}
Let $p\in\proj_A(x)$ and $q\in B$, then $\abs{p}\in B$ and it has to be shown that $\norm{\abs{p} - \abs{x}}_2 \leq \norm{q - \abs{x}}_2$.
Define $I := \set{i\inint{1}{n} | x_i < 0}$ and $\tilde{q} := q - 2\sum_{i\in I}q_i e_i$, that is the signs of entries in $I$ are flipped.
Clearly $\tilde{q}\in A$, so in conjunction with the remark at the beginning of the proof follows $\norm{\abs{p} - \abs{x}}_2^2 = \norm{p - x}_2^2 \leq \norm{\tilde{q} - x}_2^2$.
For $i\not\in I$ one obtains $x_i\geq 0$ and $\tilde{q_i} = q_i$, hence $\tilde{q_i} - x_i = q_i - \abs{x_i}$.
For $i\in I$ follows $x_i < 0$ and $\tilde{q_i} = -q_i$, hence $\tilde{q_i} - x_i = -\left(q_i - \abs{x_i}\right)$.
This yields $\left(\tilde{q_i} - x_i\right)^2 = \left(q_i - \abs{x_i}\right)^2$ for all $i\inint{1}{n}$, thus $\norm{\tilde{q} - x}_2^2 = \norm{q - \abs{x}}_2^2$, and the claim follows.

\ref{lem:proj_nonneg_sols_b}
Let $p\in\proj_B(\abs{x})$.
If $i\inint{1}{n}$ with $x_i\geq 0$, then clearly $s_ip_i - x_i = p_i - \abs{x_i}$.
For $i\inint{1}{n}$ with $x_i < 0$ follows $s_ip_i - x_i = -\left(p_i - \abs{x_i}\right)$.
Therefore, $\norm{s\hada p - x}_2^2 = \norm{p - \abs{x}}_2^2$.
Let $q\in\proj_A(x)$, then the remark at the beginning of the proof yields $\norm{q - x}_2^2 = \norm{\abs{q} - \abs{x}}_2^2$ and $\abs{q}\in B$.
$p\in\proj_B(\abs{x})$ yields $\norm{p - \abs{x}}_2^2 \leq \norm{\abs{q} - \abs{x}}_2^2$, and the claim follows.
\end{proof}
Using this result immediately yields a condition for projections to be absolutely order-preserving:
\begin{lemma}
Let $\emptyset\neq M\subseteq\R^n$ be both permutation-invariant and reflection-invariant.
Then the function $p\colon\R^n\to\P(M)$, $x\mapsto\proj_M(x)$, is absolutely order-preserving.
\end{lemma}
\begin{proof}
Let $x\in\R^n$, $p\in\proj_M(x)$, and $i,j\inint{1}{n}$ with $\abs{x_i} > \abs{x_j}$.
Define $L := M\cap\R_{\geq 0}^n$, which is permutation-invariant with Remark~\ref{rem:symmetries_closed}.
Lemma~\ref{lem:proj_nonneg_sols} implies that $\abs{p}\in\proj_L(\abs{x})$, and with Lemma~\ref{lem:proj_props} follows $\abs{p_i}\geq\abs{p_j}$.
\end{proof}
The application of these elementary results to projections onto sets on which functions achieve constant values is straightforward.
Examples were given in Section~\ref{sect:projalgorithms} with the sets $Z$ and $S^{(\lambda_1,\lambda_2)}$.

\section{Proof of Correctness of Algorithm~\ref{alg:projfunc} and Algorithm~\ref{alg:projfunc_explicit}}
\label{sect:projfuncproof}
The purpose of this appendix is to rigorously prove correctness of Algorithm~\ref{alg:projfunc} and Algorithm~\ref{alg:projfunc_explicit}, that is that they compute projections onto $S_{\geq 0}^{(\lambda_1,\lambda_2)}$.
Projections onto $S^{(\lambda_1,\lambda_2)}$ can then be inferred easily as explained in Appendix~\ref{sect:proj_symm}.

\subsection{Geometric Structures and First Considerations}
\label{sect:geomstruct}
The aim is to compute projections onto $D$, which is the intersection of the non-negative orthant $\R_{\geq 0}^n$, the target hyperplane $H$ and the target hypersphere $K$, see Section~\ref{sect:alternating_projections}.
Further, the intersection of $H$ and $K$ yields a hypercircle $L$, and the intersection of $\R_{\geq 0}^n$ and $H$ yields a scaled canonical simplex $C$.
The structure of $H$ and $L$ will be analyzed in Section~\ref{sect:l1nrmcnstrnt} and Section~\ref{sect:l2nrmcnstrnt}, respectively.
The properties of $C$ are discussed in Section~\ref{sect:splx_geom} and Section~\ref{sect:selfsim_rec}.
These results will then be used in Section~\ref{sect:proof_projfunc_thms} to prove Theorem~\ref{thm:projfunc} and Theorem~\ref{thm:projfunc_improved}.

For the analysis of subsets where certain coordinates vanish, it is useful to define the following quantities for an index set $I\subseteq\discint{1}{n}$ with cardinality $d := \abs{I}$.
The corresponding face of $C$ is denoted by $C_I := \set{c\in C | c_i = 0\text{ for all }i\not\in I}$ and has barycenter $m_I := \nicefrac{\lambda_1}{d}\cdot\sum_{i\in I}e_i\in C_I$.
Further, $L_I := \set{a\in L | a_i = 0\text{ for all }i\not\in I}$ denotes the hypercircle with according vanishing entries, and $\rho_I := \lambda_2^2 - \nicefrac{\lambda_1^2}{d}$ is the squared radius of $L_I$.
Note that $m_I$ is also the barycenter of $L_I$.

With these definitions the intermediate goal is now to prove that projections onto $D$ can be computed by alternating projections onto the geometric structures defined earlier.
The idea is to show that the set of solutions is not tampered by alternating projections onto $H$, $C$, $L$ and $L_I$.

\subsubsection{\textit{L}\textsubscript{1} Norm Constraint---Target Hyperplane}
\label{sect:l1nrmcnstrnt}
First, the projection onto the target hyperplane $H$ is considered.
Lemma~\ref{lem:plane} is an elaborated version of a result from \cite{Theis2005}, which is included here for completeness.
Using its statements, it can be assumed that the considered point lies on $H$ without modification of the solution set of the projection onto the target set $D$.
\begin{lemma}
\label{lem:plane}
Let $x\in\R^n$.
Then the following holds:
\begin{enumerate}
\item \label{lem:plane_a}
$\proj_H(x) = x + \nicefrac{1}{n}\cdot\left(\lambda_1 - e\transp x\right)e$.

\item \label{lem:plane_b}
Let $r := \proj_H(x)$. Then $\proj_D(x) = \proj_D(r)$.
\end{enumerate}
\end{lemma}
\begin{proof}
\ref{lem:plane_a}
This is essentially a projection onto a hyperplane, yielding a unique result.

\ref{lem:plane_b}
With~\ref{lem:plane_a} follows $r - x = \nicefrac{1}{n}\cdot\left(\lambda_1 - e\transp x\right)e$.
Hence $\scp{h}{r - x} = \nicefrac{\lambda_1}{n}\cdot\left(\lambda_1 - e\transp x\right)$ holds for arbitrary $h\in H$.
This expression is independent of the entries of $h$, which yields $\scp{a - b}{r - x} = 0$ for every $a,b\in H$.

Now let $p\in\proj_D(x)$, that is $\norm{p - x}_2 \leq \norm{q - x}_2$ for all $q\in D$.
Let $q\in D$ be arbitrary.
With $D\subseteq H$ follows $\scp{q - p}{r - x} = 0$, and thus Proposition~\ref{prop:sqnrm} yields $\norm{q - r}_2^2 - \norm{p - r}_2^2 = \norm{q - x}_2^2 - \norm{p - x}_2^2 + 2\scp{q - p}{x - r} = \norm{q - x}_2^2 - \norm{p - x}_2^2 \geq 0$, hence $\norm{p - r}_2^2 \leq \norm{q - r}_2^2$, so $p\in\proj_D(r)$.

For the converse let $p\in\proj_D(r)$.
Analogously $\norm{q - x}_2^2 - \norm{p - x}_2^2 = \norm{q - r}_2^2 - \norm{p - r}_2^2 \geq 0$, hence $p\in\proj_D(x)$.
\end{proof}
Therefore, the barycenter $m$ is the projection of the origin onto $H$.
The next remark gathers additional information on the norm of $m$ and dot products with this point.
\begin{remark}
\label{rem:point_m}
It is $\norm{m}_2^2 = \nicefrac{\lambda_1^2}{n}$.
Further, $\scp{m}{h} = \nicefrac{\lambda_1^2}{n}$ for all $h\in H$, and thus $\norm{h - m}_2^2 = \norm{h}_2^2 - \nicefrac{\lambda_1^2}{n}$ with Proposition~\ref{prop:sqnrm}.
\end{remark}

\subsubsection{\textit{L}\textsubscript{1} and \textit{L}\textsubscript{2} Norm Constraint---Target Hypersphere}
\label{sect:l2nrmcnstrnt}
After the projection onto the target hyperplane $H$ has been carried out, consider now the joint constraint of $H$ and the target hypersphere $K$.
First note that $L = H\cap K$ is a hypercircle, that is a hypersphere in the subspace $H$, with intrinsic dimensionality reduced by one:
\begin{lemma}
\label{lem:intersection_sphere_plane}
Consider $L = H\cap K$ and $\rho = \lambda_2^2 - \nicefrac{\lambda_1^2}{n}$.
Then the following holds:
\begin{enumerate}
\item \label{lem:intersection_sphere_plane_a}
$L = \tilde{L} := \set{q\in H | \norm{q - m}_2^2 = \rho}$.

\item \label{lem:intersection_sphere_plane_b}
$L\neq\emptyset$ if and only if $\lambda_2 \geq \nicefrac{\lambda_1}{\sqrt{n}}$.
\end{enumerate}
\end{lemma}
\begin{proof}
\ref{lem:intersection_sphere_plane_a}
Follows immediately from Remark~\ref{rem:point_m}.

\ref{lem:intersection_sphere_plane_b}
$L$ is nonempty if and only if $\rho\geq 0$ using~\ref{lem:intersection_sphere_plane_a},
and $\rho = \left(\lambda_2 + \nicefrac{\lambda_1}{\sqrt{n}}\right)\left(\lambda_2 - \nicefrac{\lambda_1}{\sqrt{n}}\right)$.
Hence, with $\lambda_1,\lambda_2 > 0$ one obtains $\rho\geq 0$ if and only if $\lambda_2 - \nicefrac{\lambda_1}{\sqrt{n}} \geq 0$.
\end{proof}
Hence $L\neq\emptyset$ by the requirement that $\lambda_2 \leq \lambda_1 \leq \sqrt{n}\lambda_2$.
Further, the following observation follows immediately from Proposition~\ref{prop:sqnrm} and $\norm{a}_2 = \norm{b}_2 = \lambda_2$ for all $a,b\in L$:
\begin{remark}
\label{rem:scpinL}
For all $a,b\in L$ it is $\norm{a - b}_2^2 = 2\left(\lambda_2^2 - \scp{a}{b}\right)$, hence $\scp{a}{b} = \lambda_2^2 - \nicefrac{1}{2}\cdot\norm{a - b}_2^2$.
\end{remark}
Therefore, on $L$ the dot product is equal to the Euclidean norm up to an additive constant.
Next consider projections onto $L$ and note that the solution set with respect to $D$ is not changed by this operation.
The major arguments for this result have been taken over from \cite{Theis2005}.
Here, the statements from Lemma~\ref{lem:intersection_sphere_plane} have been incorporated and the resulting quadratic equation was solved explicitly, simplifying the original version of \cite{Theis2005}.
\begin{lemma}
\label{lem:sphere}
Let $r\in H$ with $r\neq m$. Let $s := m + \delta(r - m)$ where $\delta := \nicefrac{\sqrt{\rho}}{\norm{r - m}_2}$.
Then:
\begin{enumerate}
\item \label{lem:sphere_a}
$\delta > 0$, $s\in L$, and $\norm{q - r}_2^2 - \norm{s - r}_2^2 = \nicefrac{1}{\delta}\cdot\norm{q - s}_2^2$ for all $q\in L$.

\item \label{lem:sphere_b}
$s = \proj_L(r)$.

\item \label{lem:sphere_c}
$\proj_D(r) = \proj_D(s)$.
\end{enumerate}
\end{lemma}
\begin{proof}
\ref{lem:sphere_a}
First note that $\delta > 0$ because of $r\neq m$.
Clearly $s = (1 - \delta)m + \delta r$.
Further, $s\in H$ because of $e\transp s = \lambda_1$, and $s\in L$ because of $\norm{s - m}_2^2 = \rho$ and Lemma~\ref{lem:intersection_sphere_plane}.

Let $q\in L$ be arbitrary.
One obtains $\norm{q}_2 = \norm{s}_2$ with $q,s\in K$ and therefore application of Proposition~\ref{prop:sqnrm} yields $\norm{q - r}_2^2 - \norm{s - r}_2^2 = 2\scp{s - q}{r}$.
With Remark~\ref{rem:point_m} follows $\scp{m}{r} = \nicefrac{\lambda_1^2}{n}$ and $\norm{r}_2^2 = \norm{r - m}_2^2 + \nicefrac{\lambda_1^2}{n}$.
Hence, $\scp{s}{r} = (1 - \delta)\scp{m}{r} + \delta\scp{r}{r} = \nicefrac{\lambda_1^2}{n} + \delta\norm{r - m}_2^2$.
On the other hand, from $s - m = \delta(r - m)$ and hence $r = \left(1 - \nicefrac{1}{\delta}\right)m + \nicefrac{1}{\delta}\cdot s$, and using Remark~\ref{rem:scpinL} it follows that $\scp{q}{r} = \left(1 - \nicefrac{1}{\delta}\right)\nicefrac{\lambda_1^2}{n} + \nicefrac{1}{\delta}\cdot\big(\lambda_2^2 - \nicefrac{1}{2}\cdot\norm{q - s}_2^2\big)$.
Therefore with $\delta\norm{r - m}_2^2 = \nicefrac{\rho}{\delta}$ one obtains
\begin{displaymath}
  \scp{s - q}{r}
  = \delta\norm{r - m}_2^2 + \nicefrac{1}{\delta}\cdot\big(\nicefrac{\lambda_1^2}{n} - \lambda_2^2 + \nicefrac{1}{2}\cdot\norm{q - s}_2^2\big)
  = \tfrac{1}{2\delta}\norm{q - s}_2^2\text{,}
\end{displaymath}
and the claim follows directly by substitution.

\ref{lem:sphere_b}
Let $q\in L$, then~\ref{lem:sphere_a} implies $\norm{q - r}_2^2 - \norm{s - r}_2^2 = \nicefrac{1}{\delta}\cdot\norm{q - s}_2^2 \geq 0$ with equality if and only if $q = s$ because of $\norm{\cdot}_2$ being positive definite. Thus $s$ is the unique projection of $r$ onto $L$.

\ref{lem:sphere_c}
With~\ref{lem:sphere_a} follows $\norm{q - s}_2^2 - \norm{p - s}_2^2 = \delta\big(\norm{q - r}_2^2 - \norm{p - r}_2^2\big)$ for all $p,q\in D$ because of $D\subseteq L$.
For $\proj_D(r) \subseteq \proj_D(s)$, let $p\in\proj_D(r)$ and $q\in D$.
By definition $\norm{p - r}_2^2 \leq \norm{q - r}_2^2$, and thus $\norm{q - s}_2^2 - \norm{p - s}_2^2 \geq 0$ with $\delta > 0$, hence $p\in\proj_D(s)$.
For the converse, let $p\in\proj_D(s)$ and $q\in D$.
Similarly, $\norm{q - r}_2^2 - \norm{p - r}_2^2 = \nicefrac{1}{\delta}\cdot\big(\norm{q - s}_2^2 - \norm{p - s}_2^2\big) \geq 0$, thus $p\in\proj_D(r)$.
\end{proof}
Lemma~\ref{lem:sphere} does not hold when $r = m$, which forms a null set.
In practice, however, this can occur when the input vector $x$ for Algorithm~\ref{alg:projfunc} is poorly chosen, for example if all entries are equal.
In this case, $\proj_L(r) = L$, hence any point from $L$ can be chosen for further processing.
\begin{remark}
\label{rem:projmontoL}
One possibility in the case $r = m$ would be to choose the point $s := \alpha\sum_{i = 1}^{n-1}e_i + \beta e_n$ where $\alpha,\beta\in\R$ for $s\in\proj_L(r)$, that is forcing the last entry to be unequal to the other ones.
For satisfying $s\in L$, set $\alpha := \nicefrac{\lambda_1}{n} + \nicefrac{\sqrt{\rho}}{\sqrt{n(n-1)}}$ and $\beta := \lambda_1 - \alpha(n-1) = \nicefrac{\lambda_1}{n} - \nicefrac{\sqrt{\rho(n-1)}}{\sqrt{n}}$.
This yields $\alpha - \beta = \sqrt{\rho}\left(\nicefrac{1}{\sqrt{n(n-1)}} + \nicefrac{\sqrt{n-1}}{\sqrt{n}}\right) > 0$, hence $\alpha\neq \beta$.
This choice has the convenient side effect of $s$ being sorted in descending order.
\end{remark}
Combining these properties of $H$ and $L$, it can now be shown that projections onto $D$ are invariant to affine-linear transformations with positive scaling:
\begin{corollary}
\label{cor:affine_invariance_projection}
Let $\alpha > 0$, $\beta\in\R$ and $x\in\R^n$.
Then $\proj_D(\alpha x + \beta e) = \proj_D(x)$.
\end{corollary}
\begin{proof}
Let $\alpha > 0$, $\beta\in\R$ and $x\in\R^n$.
With Lemma~\ref{lem:plane} and Lemma~\ref{lem:sphere} it is enough to show that $\proj_L(\proj_H(\alpha x + \beta e)) = \proj_L(\proj_H(x))$.
Let $\tilde{x} := \alpha x + \beta e$, $\tilde{r} := \proj_H(\tilde{x})$ and $\tilde{s} := \proj_L(\tilde{r})$.
Lemma~\ref{lem:plane} and $e\transp e = n$ yield $\tilde{r} = \left(\alpha x + \beta e\right) + \nicefrac{1}{n}\cdot\left(\lambda_1 - \alpha e\transp x - \beta e\transp e\right)e = \alpha x + \nicefrac{1}{n}\cdot\left(\lambda_1 - \alpha e\transp x\right)e$.
Hence $\tilde{r}$ is independent of $\beta$.
Lemma~\ref{lem:sphere} yields $\tilde{s} = m + \tilde{\delta}\left(\tilde{r} - m\right)$, where $\tilde{\delta} := \nicefrac{\sqrt{\rho}}{\norm{\tilde{r} - m}_2}$.
Application of Proposition~\ref{prop:sqnrm} yields
\begin{align*}
  \norm{\tilde{r}}_2^2
  &= \norm{\alpha x}_2^2 + \norm{\nicefrac{1}{n}\cdot\left(\lambda_1 - \alpha e\transp x\right)e}_2^2 + 2\scp{\alpha x}{\nicefrac{1}{n}\cdot\left(\lambda_1 - \alpha e\transp x\right)e}\\
  &= \alpha^2\norm{x}_2^2 + \nicefrac{1}{n}\cdot\left(\lambda_1 - \alpha e\transp x\right)\left(\lambda_1 + \alpha e\transp x\right)
  = \alpha^2\norm{x}_2^2 + \nicefrac{1}{n}\cdot\left(\lambda_1^2 - \alpha^2 (e\transp x)^2\right)\text{,}
\end{align*}
and with Remark~\ref{rem:point_m} follows $\norm{\tilde{r} - m}_2^2 = \norm{\tilde{r}}_2^2 - \nicefrac{\lambda_1^2}{n} = \alpha^2\big(\norm{x}_2^2 - \nicefrac{1}{n}\cdot(e\transp x)^2\big)$.
Let $r := \proj_H(x)$ and $s := \proj_L(r)$, then Lemma~\ref{lem:plane} and Lemma~\ref{lem:sphere} imply $r = x + \nicefrac{1}{n}\cdot\left(\lambda_1 - e\transp x\right)e$ and $s = m + \delta\left(r - m\right)$, where $\delta := \nicefrac{\sqrt{\rho}}{\norm{r - m}_2}$.
Likewise $\norm{r - m}_2^2 = \norm{x}_2^2 - \nicefrac{1}{n}\cdot(e\transp x)^2$, and hence $\nicefrac{\delta}{\tilde{\delta}} = \nicefrac{\norm{\tilde{r} - m}_2}{\norm{r - m}_2} = \alpha$, where $\alpha > 0$ must hold.
This yields
\begin{displaymath}
  \tilde{s}
  = m + \tilde{\delta}\left(\tilde{r} - m\right)
  = m + \nicefrac{\delta}{\alpha}\cdot\left(\alpha x + \nicefrac{\lambda_1}{n}\cdot e - \nicefrac{\alpha}{n}\cdot e\transp x e - \nicefrac{\lambda_1}{n}\cdot e\right)
  = m + \delta\left(x - \nicefrac{1}{n}\cdot e\transp xe\right) = s\text{,}
\end{displaymath}
which shows that the projection is invariant.
\end{proof}
\begin{figure}[t]
  \centering
  \includegraphics[page=8]{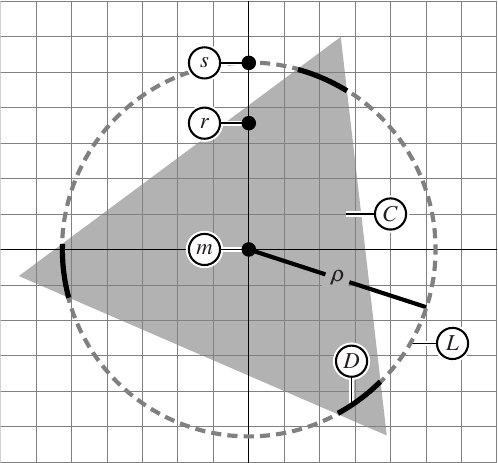}
  \caption{Sketch of the situation in Section~\ref{sect:geomstruct}, projected onto target hyperplane $H$. $r$ is the projection of the input point $x$ onto $H$. $s$ is the projection of $r$ onto the hypercircle $L$, which has squared radius $\rho$. The intersection of $H$ with the non-negative orthant is a simplex and denoted by $C$. The feasible set $D$ is the intersection of $C$ and $L$, and is marked with solid black lines. With Lemma~\ref{lem:plane} and Lemma~\ref{lem:sphere} follows that $\proj_D(x) = \proj_D(r) = \proj_D(s)$, hence the next steps consist of projecting $s$ onto $C$ for finding the projection of $x$ onto $D$.}
  \label{fig:theplot3d-nonneg}
\end{figure}
Therefore, shifting and positive scaling of the argument of Algorithm~\ref{alg:projfunc} do not change the outcome.
An overview of the steps carried out this far is given in Figure~\ref{fig:theplot3d-nonneg}.
Consider a point $x\in\R^n$ and $s := \proj_L(\proj_H(x))$.
When $s\in\R_{\geq 0}^n$, then already $s\in D$ and hence $s\in\proj_D(x)$.
Therefore only situations in which $s\not\in\R_{\geq 0}^n$ holds are relevant in the remainder of this discussion.

\subsection{Simplex Geometry}
\label{sect:splx_geom}
The joint constraint of the target hyperplane $H$ with non-negativity yields simplex $C$.
The following definition is likewise to definitions from \citet{Chen2011} and \citet{Michelot1986}:
\begin{definition}
For $n\in\N$, $n\geq 1$, the set $\splx^n := \set{\alpha\in\R_{\geq 0}^n | e\transp\alpha = 1}$ is called \emph{canonical $n$-simplex}.
\end{definition}
It is clear that $C = \R_{\geq 0}^n\cap H = \set{\lambda_1\alpha | \alpha\in\splx^n}$ is a scaled canonical simplex.
Further, for an index set $I\subseteq\discint{1}{n}$ the set $C_I = \set{c\in C | c_i = 0\text{ for all }i\not\in I}$ is a face of the simplex, which intrinsically possesses the structure of a simplex itself---although of reduced intrinsic dimensionality.
Consider the following observation on the topology of $C$ embedded in the subspace $H$:
\begin{proposition}
\label{prop:splx_boundary}
Let $c = \lambda_1\alpha\in C$ with $\alpha\in\splx^n$.
Then $c\in\partial C$ in the metric space $\left(H,\ \norm{\cdot}_2\right)$ if and only if there is a $j\inint{1}{n}$ with $\alpha_j = 0$.
\end{proposition}
The proof is simple and omitted as it does not contribute to deeper insight.
Hence the faces $C_I$ are subsets of $\partial C$, which is the topological border of $C$ in $\left(H,\ \norm{\cdot}_2\right)$.
Using Proposition~\ref{prop:splx_boundary} a statement on the inradius of $C$ can be made, which in turn can be used to show that for $n = 2$ no simplex projection has to be carried out at all:
\begin{proposition}
\label{prop:splx_inradius}
The squared inradius of $C$ is $\rho_{\insphere} := \tfrac{\lambda_1^2}{n(n-1)}$.
It is $L\subseteq C$ for $n = 2$.
\end{proposition}
\begin{proof}
Because $C$ is closed and convex, it is enough to consider the distance between interior points and boundary points.
Hence the insphere radius of a point $p\in C$ can be computed as being the minimum distance to any of the boundary points.
With Proposition~\ref{prop:splx_boundary} these points can be characterized as points where at least one entry vanishes.
Using Lagrange multipliers it can be shown that $\min_{c\in\partial C}\norm{m - c}_2^2 = \rho_{\insphere}$.
Further, it can be shown that no point other than $m$ is center of a larger insphere.
This is achieved by constructing projections on certain faces of $C$, as is discussed in detail in the forthcoming Lemma~\ref{lem:subsplx_proj}.
When $n = 2$ and $\lambda_2 \leq \lambda_1$, which is fulfilled by requirement on $\lambda_1$ and $\lambda_2$, then $\rho = \lambda_2^2 - \nicefrac{\lambda_1^2}{n} \leq \nicefrac{\lambda_1^2}{2} = \rho_{\insphere}$, and $L\subseteq C$ follows with Lemma~\ref{lem:intersection_sphere_plane}.
\end{proof}
The projection within $H$ from outside a simplex is unique and must be located on its boundary:
\begin{remark}
\label{rem:splx_proj_boundary}
Let $s\in H\setminus C$.
Then $\proj_C(s)\in\partial C$ in $\left(H,\ \norm{\cdot}_2\right)$.
\end{remark}
The proof is obvious because $C$ is closed and convex.
By combination of Proposition~\ref{prop:splx_boundary} and Remark~\ref{rem:splx_proj_boundary}, it is now evident that the projection within $H$ onto $C$ yields vanishing entries.
After the first projection onto $H$, this subspace is never left throughout the arguments presented here, such that Remark~\ref{rem:splx_proj_boundary} always applies.
It has yet to be shown that the projection onto $D$ possesses zero entries in the same coordinates.
This way, a reduction of problem dimensionality can be achieved, and an iterative algorithm can be constructed to compute the projection onto $D$.
The algorithm is guaranteed to terminate at the latest when the problem dimensionality equals two with Proposition~\ref{prop:splx_inradius}.

\subsubsection{Projection onto a Simplex}
Quite a few methods have been proposed for carrying out projections onto canonical simplexes.
An iterative algorithm was developed by \cite{Michelot1986} which is very similar to Hoyer's original method for computation of the projection onto $D$.
A simpler and more effective algorithm has been developed by \cite{Duchi2008}.
Building upon this work, \cite{Chen2011} have proposed and rigorously proved correctness of a very similar algorithm, which is more explicit than that of \cite{Duchi2008}.
Their algorithm can be adapted to better suit the needs for the sparseness-enforcing projection.
This adapted version was given by Algorithm~\ref{alg:projsplx_general} in Section~\ref{sect:projalgorithms}.
The following note makes the adaptations explicit.
\begin{proposition}
\label{prop:projsplx}
Let $x\in\R^n\setminus C$ and $p := \proj_C(x)$.
Then the following holds:
\begin{enumerate}
\item \label{prop:projsplx_a}
There exists $\hat{t}\in\R$ such that $p = \max\left(x - \hat{t}\cdot e,\ 0\right)$, where the maximum is taken element-wise.
\item \label{prop:projsplx_b}
Algorithm~\ref{alg:projsplx_general} computes $\hat{t}$ such that~\ref{prop:projsplx_a} holds and the number of nonzero entries in $p$.
\end{enumerate}
\end{proposition}
\begin{proof}
The arguments from \cite{Chen2011} hold for projections onto $\splx^n$.
The case of the scaled canonical simplex can be recovered using $p = \lambda_1\cdot\proj_{\splx^n}\left(\nicefrac{x}{\lambda_1}\right)$.
Therefore lines~\ref{algl:projsplx_general_tone} and~\ref{algl:projsplx_general_ttwo} of Algorithm~\ref{alg:projsplx_general} can be adapted from $t := \frac{s - 1}{i}$ and $t := \frac{s - 1}{n}$ to $t := \frac{s - \lambda_1}{i}$ and $t := \frac{s - \lambda_1}{n}$, respectively.
The correct number of nonzero entries in $p$ follows immediately from its expression from~\ref{prop:projsplx_a}, the fact that $y$ is sorted in descending order and the termination criterion of Algorithm~\ref{alg:projsplx_general}.
\end{proof}
As already described in Section~\ref{sect:projfunc_optimized_variant}, symmetries can be exploited for projections onto $C$:
\begin{remark}
\label{rem:projsplx_sorting}
When $x$ is already sorted in descending order, then no sorting is needed at the beginning of Algorithm~\ref{alg:projsplx_general}.
The projection $p$ is then sorted also, because $C$ is permutation-invariant.
In this case, the nonzero entries of $p$ are located in the first $d := \norm{p}_0$ entries, while $p_{d+1}=\dots=p_n=0$.
\end{remark}
This fact is useful for optimizing access to the relevant entries of the working vector, which can then be stored contiguously in memory.

\subsubsection{Projection onto a Face of a Simplex}
The projection within $H$ onto $C$ yields zero entries in the working vector.
It still remains to be shown that the projection onto $D$ possesses zero entries at the same coordinates as the projection onto $C$.
If this holds true, then the dimensionality of the original problem can be reduced, and iterative arguments can be applied.
The main building block in the proof is the explicit construction of projections from within the simplex onto a certain face.
The next Lemma is fundamental for proving correctness of Algorithm~\ref{alg:projfunc}.
It describes the construction of the result of the projection onto a simplex face and poses a statement on its norm, which in turn is used to prove that the position of vanishing entries does not change upon projection.
\begin{lemma}
\label{lem:subsplx_proj}
Let $q\in C$ and let $\emptyset\neq I\subseteq\discint{1}{n}$ be an arbitrary index set.
Then there exists an $s\in C_I$ with $\norm{q - v}_2^2 = \norm{q - s}_2^2 + \norm{s - v}_2^2$ for all $v\in C_I$.
If additionally $\max_{j\in J}q_j \leq \min_{i\in I}q_i$ holds for $J := I\comp$, then $\norm{s}_2 \geq \norm{q}_2$ with equality if and only if $q_j = 0$ for all $j\in J$.

More precisely, let $h := \abs{J}$ and let $J = \discint{j_1}{j_h}$ such that $q_{j_1} \leq \dots \leq q_{j_h}$.
Consider the sequence $s^{(0)},\dotsc,s^{(h)}\in\R^n$ defined iteratively by $s^{(0)} := q$ and
\begin{displaymath}
  s^{(k)} := s^{(k-1)} - s_{j_k}^{(k-1)}e_{j_k} + \tfrac{1}{n - k}s_{j_k}^{(k-1)}\left(e - \smallsum_{i=1}^k e_{j_i}\right)
\end{displaymath}
for $k\inint{1}{h}$.
Write $s := s^{(h)}$.
Then the following holds:
\begin{enumerate}
\item \label{lem:subsplx_proj_a}
$s^{(k)} \in C_{\discint{j_1}{j_k}^C}$ for all $k\inint{1}{h}$.

\item \label{lem:subsplx_proj_b}
$\bscp{s^{(0)} - s^{(k)}}{s^{(k)} - s^{(k+1)}} = 0$ for all $k\inint{0}{h-1}$.

\item \label{lem:subsplx_proj_c}
$\bnorm{s^{(k)} - q}_2^2 = \sum_{i=1}^{k}\bnorm{s^{(i)} - s^{(i-1)}}_2^2$ for all $k\inint{0}{h}$.

\item \label{lem:subsplx_proj_d}
$s^{(k-1)}_{j_k} = q_{j_k} + \tfrac{1}{n - k + 1}\sum_{i=1}^{k-1}q_{j_i}$ for all $k\inint{1}{h}$.

\item \label{lem:subsplx_proj_e}
$\bscp{s^{(0)} - s^{(k)}}{s^{(k)} - v} = 0$ for all $k\inint{0}{h}$ and for all $v\in C_I$.

\item \label{lem:subsplx_proj_f}
$\norm{q - v}_2^2 = \bnorm{q - s^{(k)}}_2^2 + \bnorm{s^{(k)} - v}_2^2$ for all $k\inint{0}{h}$ and for all $v\in C_I$.

\item \label{lem:subsplx_proj_g}
$s = \proj_{C_I}(q)$.
\end{enumerate}

\vspace{2ex}
If $\max_{j\in J}q_j \leq \min_{i\in I}q_i$, then the following holds as well:
\begin{enumerate}[start=8]
\item \label{lem:subsplx_proj_h}
$s^{(k)}_{j_1} \leq \dots \leq s^{(k)}_{j_h} \leq \min_{i\in I} s^{(k)}_i$ for all $k\inint{0}{h}$.

\item \label{lem:subsplx_proj_i}
$s^{(k-1)}_{j_k} \leq \tfrac{\lambda_1}{n - k + 1}$ for all $k\inint{1}{h}$.

\item \label{lem:subsplx_proj_j}
$\bnorm{s^{(k-1)}}_2 \leq \bnorm{s^{(k)}}_2$ for all $k\inint{1}{h}$, and hence $\norm{s}_2 \geq \norm{q}_2$.

\item \label{lem:subsplx_proj_k}
$\norm{s}_2 = \norm{q}_2$ if and only if $q_j = 0$ for all $j\in J$.
\end{enumerate}
\end{lemma}
\begin{proof}
In other words, $s^{(k)}$ is constructed from $s^{(k-1)}$ by setting entry $j_k$ to zero, and adjusting all remaining entries, but the ones previously set to zero, such that the $L_1$ norm is preserved.
This generates a finite series of points progressively approaching $C_I$, see~\ref{lem:subsplx_proj_a}, where the final point is from $C_I$.
As all relevant dot products vanish, see~\ref{lem:subsplx_proj_b} and~\ref{lem:subsplx_proj_e}, this is a process of orthogonal projections.
Hence the distance between points can be computed using the Pythagorean theorem, see~\ref{lem:subsplx_proj_c} and~\ref{lem:subsplx_proj_f}.
In~\ref{lem:subsplx_proj_g} it is then shown that $s$ is the unique projection of $q$ onto $C_I$.

If the entry $j_k$ in $s^{(k-1)}$ does not vanish, then the $L_2$ norm of the newly constructed point is greater than that of the original point, see~\ref{lem:subsplx_proj_j} and~\ref{lem:subsplx_proj_k}.
The entries with indices from $J$ must be sufficiently small for this non-decreasing norm property to hold, see~\ref{lem:subsplx_proj_h} and~\ref{lem:subsplx_proj_i}.
The magnitude of these entries, however, is strongly connected with the magnitudes of respective entries from the original point $q$, that is, the rank is preserved from one point to its successor.
Figure~\ref{fig:simplexface_proj} gives an example for $n = 3$ in which cases the non-decreasing norm property holds.

Let $a_k := \tfrac{1}{n-k}\left(e - \sum\nolimits_{i=1}^k e_{j_i}\right) - e_{j_k}\in\R^n$ for $k\inint{1}{h}$.
Then $s^{(k)} = s^{(k-1)} + s^{(k-1)}_{j_k}a_k$, and with induction follows $s^{(k)} = s^{(0)} + \sum_{i=1}^k s^{(i-1)}_{j_i}a_i$ for $k\inint{1}{h}$.

\ref{lem:subsplx_proj_a}
First note that $e\transp a_k = \tfrac{1}{n-k}\left(n - k\right) - 1 = 0$ and that $a_k\in\R_{\geq 0}^n$ for all $k\inint{1}{h}$.
It is now shown by induction that $s^{(k)}$ lies on the claimed face of $C$.
For $k = 1$, one obtains $s^{(1)}_{j_1} = 0$ and $s^{(1)}_i = q_i + \tfrac{1}{n-1}q_{j_1}$ for $i\neq j_1$.
Thus $s^{(1)}_i\geq 0$ for all $i\inint{1}{n}$ because of $q_i \geq 0$ for all $i\inint{1}{n}$.
Further $e\transp s^{(1)} = e\transp q + q_{j_1}e\transp a_1 = e\transp q = \lambda_1$, hence $s^{(1)}\in C_{\discint{1}{n}\setminus\set{j_1}}$.

For $k - 1\to k$, assume $s^{(k-1)}_i = 0$ for all $i\inint{j_1}{j_{k-1}}$, $s^{(k-1)}_i \geq 0$ for all $i\inint{1}{n}$ and $e\transp s^{(k-1)} = \lambda_1$.
Clearly, $s^{(k)}_i = s^{(k-1)}_i = 0$ holds for $i\inint{j_1}{j_{k-1}}$.
Furthermore, one obtains $s^{(k)}_{j_k} = s^{(k-1)}_{j_k} - s^{(k-1)}_{j_k} = 0$.
With $s^{(k-1)}\in\R_{\geq 0}^n$ and $a_k\in\R_{\geq 0}^n$ follows that $s^{(k)}\in\R_{\geq 0}^n$.
Finally it is $e\transp s^{(k)} = e\transp s^{(k-1)} + s^{(k-1)}_{j_k} e\transp a_k = e\transp s^{(k-1)} = \lambda_1$.
Hence $s^{(k)}\in C_{\discint{1}{n}\setminus\discint{j_1}{j_k}}$.

\ref{lem:subsplx_proj_b}
For $i\inint{1}{k}$ follows
\begin{align*}
  \scp{e - \smallsum_{\mu=1}^{i}e_{j_\mu}}{e - \smallsum_{\nu=1}^{k+1}e_{j_\nu}}
  &= \scp{e}{e} - \scp{e}{\smallsum_{\nu=1}^{k+1}e_{j_\nu}} -\scp{\smallsum_{\mu=1}^{i}e_{j_\mu}}{e} + \scp{\smallsum_{\mu=1}^{i}e_{j_\mu}}{\smallsum_{\nu=1}^{k+1}e_{j_\nu}}\\
  &= n - (k + 1) - i + i = n - k - 1\text{,}
\end{align*}
and therefore
\begin{align*}
  \scp{a_i}{a_{k+1}}
  &= \tfrac{1}{\left(n-i\right)\left(n-k-1\right)}\scp{e - \smallsum_{\mu=1}^{i}e_{j_\mu}}{e - \smallsum_{\nu=1}^{k+1}e_{j_\nu}}
     + \scp{e_{j_i}}{e_{j_{k+1}}}\\
  &\phantom{=} - \tfrac{1}{n-i}\scp{e - \smallsum_{\mu=1}^{i}e_{j_\mu}}{e_{j_{k+1}}}
     - \tfrac{1}{n-k-1}\scp{e_{j_i}}{e - \smallsum_{\nu=1}^{k+1}e_{j_\nu}}\\
  &= \tfrac{n-k-1}{\left(n-i\right)\left(n-k-1\right)} + 0 - \tfrac{1}{n-i} - \tfrac{0}{n-k-1}
  = 0\text{.}
\end{align*}
Thus
\begin{displaymath}
  \scp{s^{(0)} - s^{(k)}}{s^{(k)} - s^{(k+1)}}
  = \scp{\smallsum_{i=1}^{k}s^{(i-1)}_{j_i}a_i}{s^{(k)}_{j_{k+1}}a_{k+1}}
  = \smallsum_{i=1}^{k}s^{(i-1)}_{j_i}s^{(k)}_{j_{k+1}}\scp{a_i}{a_{k+1}}
  = 0\text{.}
\end{displaymath}

\ref{lem:subsplx_proj_c}
Follows by induction using Proposition~\ref{prop:sqnrm} and~\ref{lem:subsplx_proj_b}.

\ref{lem:subsplx_proj_d}
Clearly, $\scp{e_{j_k}}{a_i} = \tfrac{1}{n-i}$ for $i\inint{1}{k-1}$, hence with induction follows
\begin{align*}
  s^{(k-1)}_{j_k}
  &= \bscp{e_{j_k}}{s^{(k-1)}}
   = \bscp{e_{j_k}}{s^{(0)}} + \smallsum_{i=1}^{k-1} s^{(i-1)}_{j_i}\scp{e_{j_k}}{a_i}
  = q_{j_k} + \smallsum_{i=1}^{k-1}\tfrac{1}{n-i} s^{(i-1)}_{j_i}\\
  &\equsing{IH} q_{j_k} + \smallsum_{i=1}^{k-1}\tfrac{1}{n-i}q_{j_i} + \smallsum_{i=1}^{k-1}\smallsum_{\mu=1}^{i-1}\tfrac{1}{n-i}\tfrac{1}{n-i+1}q_{j_\mu}
  = q_{j_k} + \smallsum_{i=1}^{k-1}q_{j_i}\left[\tfrac{1}{n-i} + \smallsum_{\mu=i+1}^{k-1}\tfrac{1}{n-\mu}\tfrac{1}{n-\mu+1}\right]\text{.}
\end{align*}
Using $\sum\nolimits_{i=1}^{k-1}\sum\nolimits_{\mu=1}^{i-1} = \sum_{1\leq\mu < i\leq k-1} = \sum\nolimits_{\mu=1}^{k-1}\sum\nolimits_{i=\mu+1}^{k-1}$ the order of summation was changed after the induction step, and then the variables $i$ and $\mu$ were swapped.
For the claim to hold it is enough to show that $\tfrac{1}{n-i} + \sum\nolimits_{\mu=i+1}^{k-1}\tfrac{1}{n-\mu}\tfrac{1}{n-\mu+1} = \tfrac{1}{n - k + 1}$ for all $i\inint{1}{k-1}$, which follows by reverse induction.

\ref{lem:subsplx_proj_e}
Proof by induction. Let $v\in C_I$, that is $e\transp v = \lambda_1$ and $v_j = 0$ for all $j\in J$.
For $k = 0$, $s^{(0)} - s^{(k)} = 0$ and the claim follows.
For $k - 1\to k$, first note that
$\scp{a_k}{q} = \tfrac{1}{n-k}\left(\lambda_1 - \sum_{i=1}^{k}q_{j_i}\right) - q_{j_k}$,
$\bscp{a_k}{s^{(k-1)}} = \bscp{a_k}{s^{(0)}} + \sum_{i=1}^{k-1}s^{(i-1)}_{j_i}\scp{a_k}{a_i} = \scp{a_k}{q}$ because $\scp{a_k}{a_i} = 0$ for all $i\inint{1}{k-1}$ as shown in~\ref{lem:subsplx_proj_b},
thus $\bscp{a_k}{s^{(0)} - 2s^{(k-1)}} = -\scp{a_k}{q}$.
Furthermore $\scp{a_k}{v} = \tfrac{\lambda_1}{n-k}$ because $\scp{e_{j_i}}{v} = v_{j_i} = 0$ for all $i\inint{1}{h}$, and
$\scp{a_k}{a_k} = \norm{a_k}_2^2 = \frac{1}{\left(n - k\right)^2}\bnorm{e - \sum\nolimits_{i=1}^k e_{j_i}}_2^2 + \norm{e_{j_k}}_2^2 = 1 + \tfrac{1}{n - k}$.
Therefore,
\begin{align*}
  \bscp{s^{(0)} - s^{(k)}}{s^{(k)} - v}
  &= \bscp{s^{(0)} - s^{(k-1)} - s^{(k-1)}_{j_k}a_k}{s^{(k-1)} + s^{(k-1)}_{j_k}a_k - v}\\
  &= \bscp{s^{(0)} - s^{(k-1)}}{s^{(k-1)} - v} + s^{(k-1)}_{j_k}\bscp{a_k}{s^{(0)} - 2s^{(k-1)} - s^{(k-1)}_{j_k}a_k + v}\\
  &\equsing{IH} s^{(k-1)}_{j_k}\left(-\scp{a_k}{q} - s^{(k-1)}_{j_k}\scp{a_k}{a_k} + \scp{a_k}{v}\right)\\
  &= s^{(k-1)}_{j_k}\left(-\tfrac{\lambda_1}{n-k} + \tfrac{1}{n-k}\smallsum_{i=1}^{k}q_{j_i} +q_{j_k} - s^{(k-1)}_{j_k} \left(1 + \tfrac{1}{n-k}\right) + \tfrac{\lambda_1}{n-k}\right)\\
  &= s^{(k-1)}_{j_k}\left(q_{j_k}\left(1 + \tfrac{1}{n-k}\right) + \tfrac{1}{n-k}\smallsum_{i=1}^{k-1}q_{j_i} - s^{(k-1)}_{j_k} \left(1 + \tfrac{1}{n-k}\right)\right)
  = 0\text{,}
\end{align*}
where the final equality was yielded using the statement from~\ref{lem:subsplx_proj_d} and $\left(1 + \tfrac{1}{n-k}\right)\cdot\tfrac{1}{n-k+1} = \tfrac{1}{n-k}$.

\begin{figure}[t]
  \centering
  \includegraphics[page=9]{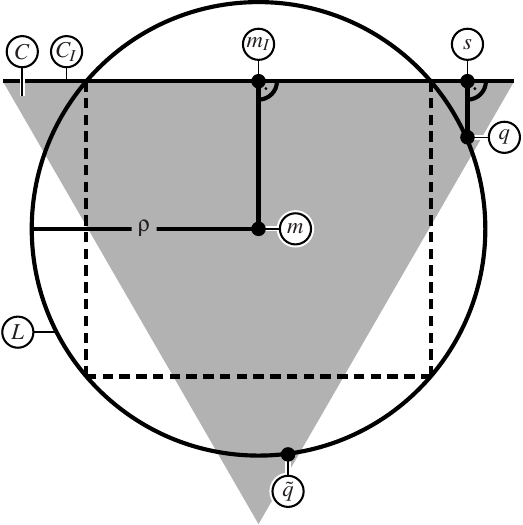}
  \caption{Situation of Lemma~\ref{lem:subsplx_proj}: The projection of a point $q$ from within the simplex $C$ onto one of its faces $C_I$ yields point $s$. When $q$ is sufficiently close to $C_I$, then $\norm{s}_2\geq\norm{q}_2$. This does not hold when the point projected onto $C_I$ is too far away. For example, the point $\tilde{q}$ which is located outside the dashed square with edge length $2\norm{m_I - m}_2$ would yield a projection with a smaller Euclidean norm.}
  \label{fig:simplexface_proj}
\end{figure}
\ref{lem:subsplx_proj_f}
Follows using Proposition~\ref{prop:sqnrm} and~\ref{lem:subsplx_proj_e}.

\ref{lem:subsplx_proj_g}
A unique Euclidean projection exists because $C_I$ is closed and convex.
$s\in C_I$ with~\ref{lem:subsplx_proj_a}, and $\norm{q - s}_2^2 \leq \norm{q - v}_2^2$ for all $v\in C_I$ with~\ref{lem:subsplx_proj_f}.
Therefore $s = \proj_{C_I}(q)$.

\ref{lem:subsplx_proj_h}
In the remainder of the proof assume that $\max_{j\in J}q_j \leq \min_{i\in I}q_i$ holds.
It is first shown by induction that $s^{(k)}_{j_1} \leq \dots \leq s^{(k)}_{j_h}$ for all $k\inint{0}{h}$.
For $k=0$ this is fulfilled as requirement on $q$ and by definition of $J$.
For $k - 1\to k$, let $\mu,\nu\inint{j_1}{j_h}$ with $\mu < \nu$.
Then with $\chi$ denoting the indicator function and with $A := \chi_{\set{j_k}}(\mu) - \chi_{\set{j_k}}(\nu) + \tfrac{1}{n-k}\left(\chi_{\discint{j_1}{j_k}^C}(\nu) - \chi_{\discint{j_1}{j_k}^C}(\mu)\right)$ follows $s^{(k)}_\nu - s^{(k)}_\mu = s^{(k-1)}_\nu - s^{(k-1)}_\mu + s^{(k-1)}_{j_k}A$.
Clearly, when $A\geq 0$ then the claim follows with the induction hypothesis and with $s^{(k-1)}_{j_k} \geq 0$ due to \ref{lem:subsplx_proj_a}.

First consider the case of $\mu\in\discint{j_1}{j_{k-1}}$. 
If $\nu\in\discint{j_1}{j_{k-1}}$ also, then $A = 0$.
If $\nu = j_k$, then $A = -1$, and hence $s^{(k)}_\nu - s^{(k)}_\mu = s^{(k-1)}_\nu - s^{(k-1)}_\mu - s^{(k-1)}_{\nu} = - s^{(k-1)}_\mu$ which however vanishes with \ref{lem:subsplx_proj_a}.
If $\nu\in\discint{j_{k+1}}{j_h}$, then $A = \tfrac{1}{n-k}\geq 0$.
If $\mu = j_k$, then $\nu\inint{j_{k+1}}{j_h}$, and then $A = 1 + \tfrac{1}{n-k}\geq 0$.
If $\mu\inint{j_{k+1}}{j_h}$, then $\nu\inint{j_{\mu+1}}{j_h}$, thus $A = 0$.
Hence the first claim is always fulfilled.

Next, it is shown that $\max_{j\in J}s^{(k)}_j \leq \min_{i\in I} s^{(k)}_i$ for all $k\inint{0}{h}$.
For $k = 0$ this is the requirement on $q$.
For $k - 1\to k$, let $i\in I$ and $j\in J$.
It is then $\chi_{\set{j_k}}(i) = 0$, $\chi_{\set{j_k}}(j)\in\set{0,1}$, $\chi_{\discint{j_1}{j_k}^C}(i) = 1$ and $\chi_{\discint{j_1}{j_k}^C}(j) = 0$, therefore
\begin{align*}
  s^{(k)}_i - s^{(k)}_j &=
  \phantom{-}s^{(k-1)}_i - s^{(k-1)}_{j_k}{\chi_{\set{j_k}}(i)} + \tfrac{1}{n-k}s^{(k-1)}_{j_k}{\chi_{\discint{j_1}{j_k}^C}(i)} \\
  &\phantom{=} -s^{(k-1)}_j + s^{(k-1)}_{j_k}\chi_{\set{j_k}}(j) - \tfrac{1}{n-k}s^{(k-1)}_{j_k}{\chi_{\discint{j_1}{j_k}^C}(j)}\\
  &= {s^{(k-1)}_i - s^{(k-1)}_j} + {s^{(k-1)}_{j_k}\left(\tfrac{1}{n-k} + \chi_{\set{j_k}}(j)\right)} \geq 0\text{,}
\end{align*}
where $s^{(k-1)}_i - s^{(k-1)}_j \geq 0$ holds by induction hypothesis.

\ref{lem:subsplx_proj_i}
With~\ref{lem:subsplx_proj_a} follows $\lambda_1 = e\transp s^{(k-1)}$ and $s^{(k-1)}_{j_i} = 0$ for all $i\inint{1}{k-1}$.
$s^{(k-1)}_i \geq s^{(k-1)}_{j_k}$ holds for all $i\in I$ with~\ref{lem:subsplx_proj_h},
and $s^{(k-1)}_{j_i} \geq s^{(k-1)}_{j_k}$ for all $i\inint{k+1}{h}$.
Therefore,
\begin{align*}
  \lambda_1
  &= \smallsum_{i=1}^ns^{(k-1)}_i
  = \smallsum_{i\in I}s^{(k-1)}_i + \smallsum_{i=1}^{k-1}s^{(k-1)}_{j_i} + s^{(k-1)}_{j_k} + \smallsum_{i=k+1}^hs^{(k-1)}_{j_i}\\
  &\geq \big((n-h) + 1 + (h - k)\big)s^{(k-1)}_{j_k} = \left(n - k + 1\right)s^{(k-1)}_{j_k}\text{,}
\end{align*}
and the claim follows because $n - k + 1 > 0$.

\ref{lem:subsplx_proj_j}
In~\ref{lem:subsplx_proj_e} it was shown that $\norm{a_k}_2^2 = 1 + \tfrac{1}{n - k}$.
Furthermore,
\begin{displaymath}
  \bscp{s^{(k-1)}}{a_k}
  = \tfrac{1}{n - k}\left(\lambda_1 - \smallsum_{i=1}^{k-1}s^{(k-1)}_{j_i} - s^{(k-1)}_{j_k}\right) - s^{(k-1)}_{j_k}
  = \tfrac{1}{n - k}\left(\lambda_1 - s^{(k-1)}_{j_k}\right) - s^{(k-1)}_{j_k}\text{,}
\end{displaymath}
because all $s^{(k-1)}_{j_i}$ vanish for $i\inint{1}{k-1}$ using~\ref{lem:subsplx_proj_a}.
Application of Proposition~\ref{prop:sqnrm} yields
\begin{align*}
  \bnorm{s^{(k)}}_2^2 - \bnorm{s^{(k-1)}}_2^2
  &= \bnorm{s^{(k-1)}_{j_k}a_k}_2^2 + 2s^{(k-1)}_{j_k}\bscp{s^{(k-1)}}{a_k}\\
  &= s^{(k-1)}_{j_k}\left[s^{(k-1)}_{j_k}\left(1 + \tfrac{1}{n - k}\right) + \tfrac{2}{n-k}\left(\lambda_1 - s^{(k-1)}_{j_k}\right) - 2s^{(k-1)}_{j_k}\right]\\
  &= s^{(k-1)}_{j_k}\left[\tfrac{2\lambda_1}{n - k} - s^{(k-1)}_{j_k}\left(1 + \tfrac{1}{n - k}\right)\right]\text{,}
\end{align*}
which is non-negative when $s^{(k-1)}_{j_k} \leq \left(1 + \tfrac{1}{n - k}\right)^{-1}\cdot\tfrac{2\lambda_1}{n - k} = \tfrac{2\lambda_1}{n - k + 1}$.
With~\ref{lem:subsplx_proj_i} this is always fulfilled, hence $\bnorm{s^{(k-1)}}_2 \leq \bnorm{s^{(k)}}_2$, and $\norm{s}_2 \geq \norm{q}_2$ follows immediately using a telescoping sum argument.

\ref{lem:subsplx_proj_k}
When $q_j = 0$ for all $j\in J$, then $s = q$ and the claim follows.
When there is a $j_k\inint{j_1}{j_h}$ with $q_{j_k}\neq 0$, let $k$ be minimal such that either $k = 1$ or $q_{j_{k-1}} = 0$, hence $s^{(k-1)}_{j_k} = q_{j_k}$.
With~\ref{lem:subsplx_proj_i} follows $0 < s^{(k-1)}_{j_k}  \leq \tfrac{\lambda_1}{n - k + 1} < \tfrac{2\lambda_1}{n - k + 1}$, and hence $\bnorm{s^{(k)}}_2^2 - \bnorm{s^{(k-1)}}_2^2 > 0$ with~\ref{lem:subsplx_proj_j}, thus $\bnorm{s}_2 > \bnorm{q}_2$.
\end{proof}
\begin{figure}[t]
  \centering
  \includegraphics[page=10]{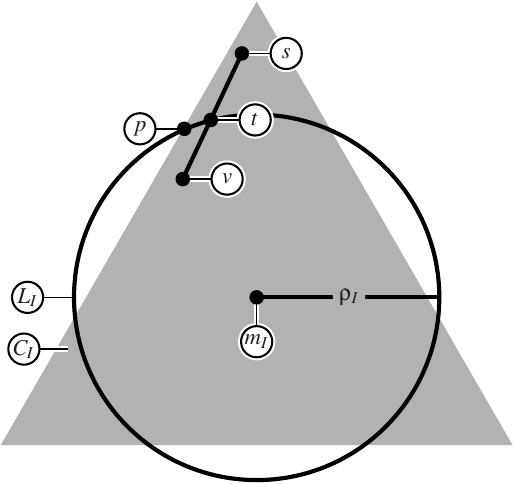}
  \caption{Sketch of the proof of Corollary~\ref{cor:subsplx_proj}: The projection of $v$ onto $D$ must be located on simplex face $C_I$. Assume there is a projection $q\not\in C_I$, then it must be sufficiently close to $C_I$ for application of Lemma~\ref{lem:subsplx_proj}. The projection $s$ of $q$ onto $C_I$ is located outside $L_I = \set{a\in L|a_i = 0\text{ for all }i\not\in I}$. Hence the intersection $t$ of the line between $v$ and $s$ and $L_I$ is in $C_I$ due to its convexity, and it is farther than the projection $p$ of $v$ onto $\tilde{D} = \set{a\in D|a_i = 0\text{ for all }i\not\in I}$. Therefore $q$ cannot be the projection of $v$ onto $D$.}
  \label{fig:simplexface_corr}
\end{figure}

\noindent
The application of Lemma~\ref{lem:subsplx_proj} then shows that the projection of a point from a face $C_I$ of $C$ onto $D$ must reside on the same face $C_I$, given the original point is located within a sphere with squared radius $\rho_I$ around $m_I$.
As will be shown in Lemma~\ref{lem:splx}, this is automatically fulfilled for projections from $L$ onto $C$.
\begin{corollary}
\label{cor:subsplx_proj}
Let $I\subseteq\discint{1}{n}$, let $v\in C_I$ with $\norm{v}_2 < \lambda_2$, and let $q\in\proj_D(v)$.
Then $q\in C_I$.
\end{corollary}
\begin{proof}
Let $J := \discint{1}{n}\setminus I$, and let $q\in\proj_D(v)$.
Assume there is at least one $j\in J$ with $q_j \neq 0$.
For showing $\max_{j\in J}q_j \leq \min_{i\in I}q_i$, assume there are $i\in I$ and $j\in J$ with $q_j > q_i$.
Then $v_i > 0$ and $v_j = 0$ because of $v\in C_I$.
$D$ is permutation-invariant using Remark~\ref{rem:symmetries_closed} as intersection of permutation-invariant sets.
Hence let $\tau := (i, j)\in S_n$ be the transposition swapping $i$ and $j$, and consider
\begin{displaymath}
  d := \norm{q - v}_2^2 - \norm{P_\tau q - v}_2^2 = 2\left(q_j - q_i\right)\left(v_i - v_j\right)\text{.}
\end{displaymath}
It is $d > 0$ because of $q_j - q_i > 0$ and $v_i - v_j = v_i > 0$.
Hence $\norm{P_\tau q - v}_2 < \norm{q - v}_2$ and $P_\tau q\in D$, which violates the minimality of $q$.
Therefore, $\max_{j\in J}q_j \leq \min_{i\in I}q_i$ must hold.

A drawing for the next arguments is given in Figure~\ref{fig:simplexface_corr}.
With Lemma~\ref{lem:subsplx_proj} there is an $s\in C_I$ such that $\norm{q - v}_2^2 = \norm{q - s}_2^2 + \norm{s - v}_2^2$ and $\norm{s}_2 > \lambda_2$.
Consider $f\colon\intervalcc{0}{1}\to\R$, $\beta\mapsto\norm{v + \beta\left(s - v\right)}_2$.
Clearly $f(0) = \norm{v}_2 < \lambda_2$ and $f(1) = \norm{s}_2 > \lambda_2$, hence with the intermediate value theorem there exists a $\beta^*\in\intervaloo{0}{1}$ with $f(\beta^*) = \lambda_2$.
Let $t := v + \beta^*\left(s - v\right)\in\R^n$, which lies in $C_I$ because of $v,s\in C_I$ and $C_I$ is convex.
By construction $\norm{t}_2 = \lambda_2$, hence $t\in\tilde{D} := \set{a\in D|a_i = 0\text{ for all }i\not\in I}$.
Clearly, $\norm{v - t}_2 + \norm{t - s}_2 = \abs{\beta^*}\cdot\norm{s-v}_2 + \abs{1 - \beta^*}\cdot\norm{v-s}_2 = \norm{s-v}_2$.
Let $p\in\proj_{\tilde{D}}(v)$, then $\norm{v - p}_2 \leq \norm{v - t}_2$.
Therefore,
\begin{displaymath}
  \norm{q - v}_2^2 = \norm{q - s}_2^2 + \norm{s - v}_2^2 = \norm{q - s}_2^2 + \left(\norm{v - t}_2 + \norm{t - s}_2\right)^2
  > \norm{v - t}_2^2 \geq \norm{v - p}_2^2\text{.}
\end{displaymath}
Because of $\tilde{D}\subseteq D$, $p\in D$ also, hence $\norm{q - v}_2 \leq \norm{p - v}_2$, which contradicts $\norm{q - v}_2 > \norm{v - p}_2$.
Hence, $q_j = 0$ for all $j\in J$ must hold, and thus $q\in C_I$.
\end{proof}
Because $C_I$ is isomorphic to a simplex, but with lower dimensionality than $C$, an algorithm can be constructed to compute the projection onto $D$, as discussed in the following.

\subsection{Self-Similarity of the Feasible Set}
\label{sect:selfsim_rec}
The next Lemma summarizes previous results and analyzes projections from $L$ onto $C$ in greater detail.
It shows that the solution set with respect to the projection onto $D$ is not tampered, and that all solutions have zeros  at the same positions as the projection onto $C$.
Figure~\ref{fig:simplex_proj_rec} provides orientation on the quantities discussed in Lemma~\ref{lem:splx}.

\begin{figure}[t]
  \centering
  \includegraphics[page=11]{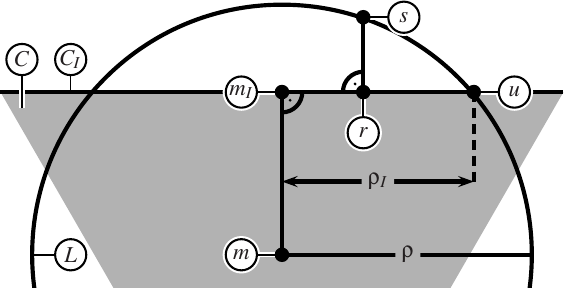}
  \caption{Situation of Lemma~\ref{lem:splx} and Lemma~\ref{lem:splxrcsn}: The point $s$ is projected onto simplex $C$ yielding $r$, which resides on one of its faces $C_I$. From there point $u$ can be constructed by projecting within $C_I$ onto the target hypersphere.}
  \label{fig:simplex_proj_rec}
\end{figure}

\begin{lemma}
\label{lem:splx}
Let $s\in L\setminus C$ and $r\in\proj_C(s)$.
Let $I := \set{i\inint{1}{n} | r_i\neq 0}$ and $d := \abs{I}$.
Then:
\begin{enumerate}
\item \label{lem:splx_a}
There exists $\hat{t}\in\R_{\geq 0}$ such that $r = \max\left(s - \hat{t}\cdot e,\ 0\right)$, with the maximum taken element-wise.

\item \label{lem:splx_b}
$I\neq \emptyset$ and $I \neq \discint{1}{n}$.

\item \label{lem:splx_c}
$s_i > \hat{t}$ and $s_i - r_i = \hat{t}$ for all $i\in I$.
$s_i \leq \hat{t}$ and $s_i - r_i = s_i$ for all $i\not\in I$.

\item \label{lem:splx_d}
$\scp{r}{s - r} = \lambda_1\hat{t}$.

\item \label{lem:splx_e}
$\scp{m_I}{s - r} = \lambda_1\hat{t}$, thus $\scp{m_I - r}{s - r} = 0$.

\item \label{lem:splx_f}
$\proj_D(r) \subseteq C_I$, hence from $q\in\proj_D(r)$ follows $q_i = 0$ for all $i\not\in I$.

\item \label{lem:splx_g}
Let $q\in\proj_D(r)$. Then $\scp{q - r}{s - r} = 0$, and thus $\norm{q - s}_2^2 = \norm{q - r}_2^2 + \norm{r - s}_2^2$.

\item \label{lem:splx_h}
$\proj_D(r) = \proj_D(s) \subseteq C_I$.
\end{enumerate}
\end{lemma}
\begin{proof}
\ref{lem:splx_a}
The existence of $\hat{t}\in\R$ such that $r = \max\left(s - \hat{t}\cdot e,\ 0\right)$ is guaranteed by Proposition~\ref{prop:projsplx}.
It remains to be shown that $\hat{t}\geq 0$.
Consider the index set $J := \set{j\inint{1}{n} | s_j\geq 0}$ of non-negative entries of $s$.
Because of $s\in L\setminus C$ one obtains $e\transp s = \lambda_1$ and there is an index $i$ with $s_i < 0$, hence $J\neq\discint{1}{n}$.
Therefore, $\lambda_1 = \sum_{j\in J}s_j + \sum_{j\not\in J}s_j < \sum_{j\in J}s_j$.
Now assume $\hat{t} < 0$, then $s_j - \hat{t} > 0$ for all $j\in J$, and hence $r_j = s_j - \hat{t}$ for all $j\in J$.
Using $r\in C$ yields
\begin{displaymath}
  \lambda_1
  = \smallsum_{j=1}^nr_j
  \geq \smallsum_{j\in J}r_j
  = \smallsum_{j\in J}\left(s_j - \hat{t}\right)
  = \smallsum_{j\in J}s_j - \abs{J}\cdot\hat{t}
  > \smallsum_{j\in J}s_j\text{,}
\end{displaymath}
which contradicts $\sum_{j\in J}s_j > \lambda_1$.
Hence $\hat{t}\geq 0$ must hold.

\ref{lem:splx_b}
Would $I = \emptyset$ hold, then $r = 0$, which is impossible because of $r\in C$.
$I = \discint{1}{n}$ would violate the existence of vanishing entries in $r$, as is guaranteed by Remark~\ref{rem:splx_proj_boundary}.

\ref{lem:splx_c}
Using~\ref{lem:splx_a}:
When $i\in I$, then $0 < r_i = s_i - \hat{t}$ and the claim follows.
When $i\not\in I$, then $r_i = 0$, hence $s_i - \hat{t} \leq 0$ and the claim follows.

\ref{lem:splx_d}
$e\transp r = e\transp s = \lambda_1$ because $r,s\in H$, hence using~\ref{lem:splx_c} yields
\begin{displaymath}
  0
  = \scp{e}{s - r}
  = \smallsum_{i\in I}\left(s_i - r_i\right) + \smallsum_{i\not\in I}\left(s_i - r_i\right)
  = \smallsum_{i\in I}\hat{t} + \smallsum_{i\not\in I}s_i
  = d\hat{t} + \lambda_1 - \smallsum_{i\in I}s_i\text{,}
\end{displaymath}
thus
\begin{displaymath}
  \scp{r}{s - r}
  = \smallsum_{i\in I}r_i\left(s_i - r_i\right)
  = \smallsum_{i\in I}\left(s_i - \hat{t}\right)\hat{t}
  = \hat{t}\left(\smallsum_{i\in I}s_i - d\hat{t}\right)
  = \lambda_1\hat{t}\text{.}
\end{displaymath}

\ref{lem:splx_e}
$\scp{m_I}{s - r} = \nicefrac{\lambda_1}{d}\cdot\sum_{i\in I}\left(s_i - r_i\right) = \nicefrac{\lambda_1}{d}\cdot\sum_{i\in I}\hat{t} = \lambda_1\hat{t}$ with~\ref{lem:splx_c}, and the claim follows with~\ref{lem:splx_d}.

\ref{lem:splx_f}
$r\in C_I$ by definition of $I$.
Using this, \ref{lem:splx_c} and $\hat{t}\geq 0$ from~\ref{lem:splx_a} yields
\begin{align*}
  \lambda_2^2
  &= \norm{s}_2^2
   = \smallsum_{i\in I}s_i^2 + \smallsum_{i\not\in I}s_i^2
   = \smallsum_{i\in I}\left(r_i + \hat{t}\right)^2 + \smallsum_{i\not\in I}s_i^2\\
  &> \smallsum_{i\in I}r_i^2 + d\hat{t}^2 + 2\hat{t}\smallsum_{i\in I}r_i
  = \norm{r}_2^2 + d\hat{t}^2 + 2\lambda_1\hat{t}
  \geq \norm{r}_2^2\text{,}
\end{align*}
thus the claim holds using Corollary~\ref{cor:subsplx_proj}.

\ref{lem:splx_g}
With~\ref{lem:splx_f} follows $q\in C_I$.
Hence $\scp{q}{s - r} = \sum_{i\in I}q_i\left(s_i - r_i\right) = \hat{t}\sum_{i\in I}q_i = \lambda_1\hat{t}$ with~\ref{lem:splx_c}, and the claims follow with~\ref{lem:splx_d} and Proposition~\ref{prop:sqnrm}.

\ref{lem:splx_h}
Let $p\in\proj_D(s)$ and $q\in\proj_D(r)$.
Then using~\ref{lem:splx_g} and Proposition~\ref{prop:sqnrm} one obtains that $\norm{q - s}_2^2 - \norm{p - s}_2^2 = \norm{q - r}_2^2 - \norm{p - r}_2^2 + 2\scp{p - r}{s - r}$.
With~\ref{lem:splx_c} and~\ref{lem:splx_d} follows
\begin{align*}
  \scp{p - r}{s - r} 
  &= \smallsum_{i\in I}p_i\left(s_i - r_i\right) + \smallsum_{i\not\in I}p_i\left(s_i - r_i\right) - \scp{r}{s - r}
  = \smallsum_{i\in I}p_i\hat{t} + \smallsum_{i\not\in I}p_is_i - \lambda_1\hat{t}\\
  &= \hat{t}\left(\lambda_1 - \smallsum_{i\not\in I}p_i\right) + \smallsum_{i\not\in I}p_is_i - \lambda_1\hat{t}
  = \smallsum_{i\not\in I}p_i\left(s_i - \hat{t}\right) \leq 0\text{.}
\end{align*}
Now $q\in\proj_D(r)$ yields $\norm{q - r}_2^2 \leq \norm{p - r}_2^2$, and hence $\norm{q - s}_2^2 - \norm{p - s}_2^2 \leq 0$, thus $q\in\proj_D(s)$.
Similarly, $\norm{q - s}_2^2 - \norm{p - s}_2\geq 0$ with $p\in\proj_D(s)$, so $\norm{q - r}_2^2 - \norm{p - r}_2^2 \geq -2\scp{p - r}{s - r}\geq 0$, therefore $p\in\proj_D(r)$.
\end{proof}
The following corollary states a similar result as in \cite{Theis2005}.
However, the proof here uses the notion of simplex projections instead of relying on pure analytical statements.
The result presented here is stronger, as multiple entries of the vector can be set to zero simultaneously, while in \cite{Theis2005} at most one entry can be zeroed out in a single iteration.
\begin{corollary}
\label{cor:splx}
Let $s\in L\setminus C$ and $p\in\proj_D(s)$.
Then $p_i = 0$ for all $i\inint{1}{n}$ with $s_i \leq 0$.
\end{corollary}
\begin{proof}
Let $i\inint{1}{n}$ with $s_i \leq 0$.
Let $r\in\proj_C(s)$.
With Lemma~\ref{lem:splx}\ref{lem:splx_a} follows $r_i = 0$ because $\hat{t}\geq 0$, and the claim follows with Lemma~\ref{lem:splx}\ref{lem:splx_h}.
\end{proof}
The final step is to meet the hypersphere constraint again.
For this, the simplex projection $r$ is projected onto the target hypersphere, simultaneously keeping already vanished entries at zero, yielding a point $u$.
Lemma~\ref{lem:splxrcsn} gives an explicit formulation of this projection and shows that the solution set with respect to the projection onto $D$ stays the same.
Refer to Figure~\ref{fig:simplex_proj_rec} for a sketch of the construction of $u$.
\begin{lemma}
\label{lem:splxrcsn}
Let $s\in L\setminus C$, $r := \proj_C(s) = \max\left(s - \hat{t}\cdot e,\ 0\right)$ with $\hat{t}\in\R_{\geq 0}$ using Lemma~\ref{lem:splx}.
Let $I := \set{i\inint{1}{n} | r_i\neq 0}$ and $d := \abs{I}$.
Let $u := m_I + \delta\left(r - m_I\right)$ where $\delta := \nicefrac{\sqrt{\rho_I}}{\norm{r - m_I}_2}$.
Then:
\begin{enumerate}
\item \label{lem:splxrcsn_a}
$u\in L$ and $u_i = 0$ for all $i\not\in I$, hence $u\in L_I$.

\item \label{lem:splxrcsn_b}
$\proj_D(u)\subseteq C_I$.

\item \label{lem:splxrcsn_c}
$u = \proj_{L_I}(r)$.

\item \label{lem:splxrcsn_d}
$\proj_D(u) = \proj_D(r) = \proj_D(s) \subseteq C_I$.
\end{enumerate}
\end{lemma}
\begin{proof}
\ref{lem:splxrcsn_a}
Clearly $u\in H$.
With $m_I,r\in H$ and Remark~\ref{rem:point_m} follows that $\scp{m}{m_I} = \scp{m}{r} = \nicefrac{\lambda_1^2}{n}$.
Moreover, $\scp{m_I}{m_I} = \nicefrac{\lambda_1^2}{d}$ and $\scp{r}{m_I} = \sum_{i\in I}r_i\cdot\nicefrac{\lambda_1}{d} = \nicefrac{\lambda_1^2}{d}$, therefore $\scp{r}{m_I - m} = \scp{m_I}{m_I - m}$.
With $u = \left(1 - \delta\right)m_I + \delta r$ it is $\scp{u}{m_I - m} = \scp{m_I}{m_I - m} = \lambda_1^2\left(\nicefrac{1}{d} - \nicefrac{1}{n}\right)$.
Hence $\scp{u - m_I}{m_I - m} = 0$.
Further, $\norm{m - m_I}_2^2 = \norm{m}_2^2 + \norm{m_I}_2^2 - 2\scp{m}{m_I} = \lambda_1^2\left(\nicefrac{1}{d} - \nicefrac{1}{n}\right)$.

Thus with Proposition~\ref{prop:sqnrm}, $\norm{u - m}_2^2 = \norm{u - m_I}_2^2 + \norm{m_I - m}_2^2 = \rho_I + \lambda_1^2\left(\nicefrac{1}{d} - \nicefrac{1}{n}\right) = \rho$, and with Lemma~\ref{lem:intersection_sphere_plane} follows $u\in L$.
For $i\not\in I$, one obtains $u_i = \left(1 - \delta\right)e_i\transp m_I + \delta r_i = 0$, hence $u\in L_I$.

\ref{lem:splxrcsn_b}
If $u\in C$, then $u\in D$ because of $u\in L$ with~\ref{lem:splxrcsn_a}, and hence $u = \proj_D(u)$.
The claim then follows with $u_i = 0$ for all $i\not\in I$.

If $u\not\in C$, then let $q\in\proj_D(u)$.
With Corollary~\ref{cor:splx} applied to $u$ follows that $q_i = 0$ for all $i$ with $u_i \leq 0$, especially for all $i\not\in I$.
Hence the claim follows.

\ref{lem:splxrcsn_c}
Write $I = \discint{i_1}{i_d}$ and consider $\varphi\colon\R^n\to\R^d$, $\left(x_1,\ \dotsc,\ x_n\right)\transp\mapsto\left(x_{i_1},\ \dotsc,\ x_{i_d}\right)\transp$.
Further, let $\tilde{H} := \set{a\in\R^d | e\transp a = \lambda_1}$, $\tilde{K} := \set{q\in\R^d | \norm{q}_2 = \lambda_2}$, $\tilde{L} := \tilde{H}\cap \tilde{K}$ and $\tilde{D} := \R_{\geq 0}^d \cap \tilde{H} \cap \tilde{K}$.
Clearly, when $x_i = 0$ for all $i\not\in I$, then $e\transp x = e\transp\varphi(x)$ and $\norm{x}_2 = \norm{\varphi(x)}_2$.
Thus in this case membership of $x$ in one of $H$, $K$, $L$ or $D$ implies membership of $\varphi(x)$ in $\tilde{H}$, $\tilde{K}$, $\tilde{L}$ or $\tilde{D}$, respectively.

Application of Lemma~\ref{lem:sphere} to $\varphi(r)$ and $\varphi(u)$ implies that $\varphi(u) = \proj_{\tilde{L}}(\varphi(r))$.
Let $q\in L_I$, then $\varphi(q)\in\tilde{L}$, hence $\norm{\varphi(u) - \varphi(r)}_2 \leq \norm{\varphi(q) - \varphi(r)}_2$.
From $i\not\in I$ follows $r_i = u_i = q_i = 0$, hence $\norm{u - r}_2 = \norm{\varphi(u) - \varphi(r)}_2$ and $\norm{q - r}_2 = \norm{\varphi(q) - \varphi(r)}_2$, and the claim follows.

\ref{lem:splxrcsn_d}
For the converse of $\varphi$, let $\psi\colon\R^d\to\R^n$, $\tilde{x}\mapsto x$ where $x_i = 0$ for all $i\not\in I$ and $x_i = \tilde{x}_j$ when there is a $j\inint{1}{d}$ with $i = i_j$.
Analogous to the above, membership of $\tilde{y}$ in one of $\tilde{H}$, $\tilde{K}$, $\tilde{L}$ or $\tilde{D}$ implies membership of $\psi(\tilde{y})$ in $H$, $K$, $L$ or $D$, respectively.

With Lemma~\ref{lem:splx}\ref{lem:splx_f} and Lemma~\ref{lem:splx}\ref{lem:splx_h} it is enough to show $\proj_D(u) = \proj_D(r)$.
Like in~\ref{lem:splxrcsn_c}, from Lemma~\ref{lem:sphere} follows as well that $\proj_{\tilde{D}}(\varphi(r)) = \proj_{\tilde{D}}(\varphi(u))$.
Let $p\in\proj_D(u)$ and $q\in\proj_D(r)$, then $p\in C_I$ with~\ref{lem:splxrcsn_b} and $q\in C_I$ with Lemma~\ref{lem:splx}\ref{lem:splx_f}, and thus $\varphi(p),\varphi(q)\in\tilde{D}$.
Assume $\varphi(p)\not\in\proj_{\tilde{D}}(\varphi(u))$, then there exists an $a\in\tilde{D}$ with $\norm{\psi(a) - u}_2 = \norm{a - \varphi(u)}_2 < \norm{\varphi(p) - \varphi(u)}_2 = \norm{p - u}_2$, violating the minimality of $p$.
Hence $\varphi(p)\in\proj_{\tilde{D}}(\varphi(u))$, and analogously follows $\varphi(q)\in\proj_{\tilde{D}}(\varphi(r))$.
Now $\proj_{\tilde{D}}(\varphi(r)) = \proj_{\tilde{D}}(\varphi(u))$ implies that $\varphi(p)\in\proj_{\tilde{D}}(\varphi(r))$ and $\varphi(q)\in\proj_{\tilde{D}}(\varphi(u))$.
Thus, $\norm{p - r}_2 = \norm{\varphi(p) - \varphi(r)}_2 = \norm{\varphi(q) - \varphi(r)}_2 = \norm{q - r}_2$, so $p\in\proj_D(r)$, and one obtains analogously that $q\in\proj_D(u)$.
Therefore $\proj_D(u) = \proj_D(r)$.
\end{proof}
With Lemma~\ref{lem:splxrcsn} a point $u$ is constructed.
If $u\in C$, then it is already the solution for the projection onto $D$.
Otherwise, Lemma~\ref{lem:splx} and Lemma~\ref{lem:splxrcsn} can be applied once more, gaining a new point $u$.
Lemma~\ref{lem:splx}\ref{lem:splx_b} states that the amount of nonzero entries of $u$ must decrease, hence this process can be repeated for at most $n$ iterations.
If a point with only two non-vanishing entries results, it is guaranteed to be a solution by Proposition~\ref{prop:splx_inradius}.

\subsection{Proof of Theorem~\ref{thm:projfunc} and Theorem~\ref{thm:projfunc_improved}}
\label{sect:proof_projfunc_thms}
Using the previous results it can now be shown that the proposed Algorithm~\ref{alg:projfunc} actually computes a correct solution, and that the algorithm always terminates in finite time.

\begin{proofof}[Proof of Theorem~\ref{thm:projfunc}]
For proving partial correctness, let $x\in\R^n$ be arbitrary.
Lemma~\ref{lem:plane} yields $\proj_D(x) = \proj_D(r)$ after line~\ref{algl:projH}, and with Lemma~\ref{lem:sphere} follows $\proj_D(x) = \proj_D(s)$ after line~\ref{algl:projL}.
There is a pre-test loop in line~\ref{algl:projwhile}, and it has to be shown that the loop-invariant is $\proj_D(x) = \proj_D(s)$.
At the beginning of the loop, $s\not\in\R_{\geq 0}^n$ must hold, thus $s\in L\setminus C$.
After line~\ref{algl:projC}, $\proj_D(x) = \proj_D(r)$ holds with Lemma~\ref{lem:splx}.
Then with Lemma~\ref{lem:splxrcsn}, $\proj_D(x) = \proj_D(s)$ is ensured after line~\ref{algl:projLI}, hence the loop-invariant holds.
Thus, after the loop it is $\proj_D(x) = \proj_D(s)$ and $s\in D$, so $\proj_D(x) = s$.
If $r = m$ in line~\ref{algl:projL} or $r = m_I$ in line~\ref{algl:projLI}, $s$ can be chosen to be any point from $L$ or $L_I$, respectively, for example the point given in Remark~\ref{rem:projmontoL}.
In this case, the projection is not unique, but a valid representative is found.

To prove total correctness, it has to be shown that the loop in line~\ref{algl:projwhile} terminates.
Remark~\ref{rem:splx_proj_boundary} applied to $C_I$ guarantees that the number of nonzero entries in $s$ is strictly less at the end of the loop than the number of nonzero entries upon entering the loop.
Hence, at most $n$ iterations of the loop can be carried out, and when $\abs{I} = 2$ the solution is already in $D$ with Proposition~\ref{prop:splx_inradius}.
Thus the algorithm terminates in finite time.
\end{proofof}
It remains to be shown that the optimized variant is also correct.

\begin{proofof}[Proof of Theorem~\ref{thm:projfunc_improved}]
First note that Algorithm~\ref{alg:projfunc_explicit} consists of a procedure {\tt proj\_L} carrying out projections onto $L$ and $L_I$ in-place, and a main body.
A function {\tt proj\_C} is called to obtain the information on how to perform projections onto $C$.
This is carried out by Algorithm~\ref{alg:projsplx_general}.
Upon entry of the main body, the input vector $x$ is sorted in descending order, yielding a vector $y$.
The algorithm then operates on the sorted vector $y$, and undoes the sorting permutation at the end.
Because $H$, $L$ and $C$ are permutation-invariant, the projections onto the respective sets are guaranteed to remain sorted with Lemma~\ref{lem:proj_props}.

Therefore, $y$ has not to be sorted again for the simplex projection, as Algorithm~\ref{alg:projsplx_general} would require.
Also note from Lemma~\ref{lem:splx} that in the simplex projection the smallest elements are set to zero, and the original Algorithm~\ref{alg:projfunc} continues working on the $d$ non-vanishing entries.
Because of the order-preservation, entries $d+1,\ \dotsc,\ n$ of $y$ are zero, and all relevant information is concentrated in $y_1,\ \dotsc,\ y_d$.
Therefore, Algorithm~\ref{alg:projfunc_explicit} can continue working on these first $d$ entries only, and the index set of non-vanishing entries is always $I = \discint{1}{d}$.
As the nonzero elements are stored contiguously in memory, access to $y$ can be realized as a small unit-stride array.
This is more efficient than working on a large and sparsely populated vector.
Therefore, the loop starting at line~\ref{algl:projfunc_explicit_while} corresponds to the loop starting at line~\ref{algl:projwhile} in Algorithm~\ref{alg:projfunc}.
At the end of the main body, the sorting permutation $\tau$ is inverted and the entries from the sorted result vector $y$ are stored in a new vector $s$.
Because $y_{d+1},\ \dotsc,\ y_n = 0$, these entries can be ignored by setting the entire vector $s$ to zero before-hand.
\end{proofof}
The proposed optimizations hence lead to the same solution which the original algorithm computes.

\section{Analytical Properties of the Sparseness-Enforcing Projection Operator}
\label{sect:analytical_properties}
In this appendix, it is studied in which situations $\pi_{\geq 0}$ and $\pi$ as defined in Section~\ref{sect:projfunc_differentiability} are differentiable, and hence continuous.
Further, an explicit expression for their gradient is sought.
It is clear by Theorem~\ref{thm:projfunc} that the projection of any point onto $D$ can be written as finite composition of projections onto $H$, $L$, $C$ and $L_I$, respectively.
In other words, for all points $x\in\R^n\setminus R$ there exists a finite sequence of index sets $I_1,\dotsc,I_h\subseteq\discint{1}{n}$ with $I_j\supsetneq I_{j+1}$ for $j\inint{1}{h-1}$ such that
\begin{displaymath}
  \pi_{\geq 0}(x) = \left[\medcirc^{j=h}_1\big(\proj_{L_{I_j}}\circ\proj_C\big)\circ\proj_L\circ\proj_H\right](x)\text{,}
\end{displaymath}
where $\medcirc^{j=h}_1$ denotes iterated composition of functions, starting with $j = h$ and decreasing until $j = 1$, that is
$\pi_{\geq 0}(x) = \proj_{L_{I_h}}(\proj_C(\cdots\proj_{L_{I_1}}(\proj_C(\proj_L(\proj_H(x))))\cdots))$.
The sequence $I_1,\dotsc,I_h$ here depends on $x$.
The intermediate goal is to show that this sequence remains fixed in a neighborhood of $x$, and that each projection in the chain is differentiable almost everywhere.
This then implies differentiability of $\pi_{\geq 0}$ except for a null set.
Because of the close relationship of $\pi$ with $\pi_{\geq 0}$, $\pi$ is then also differentiable almost everywhere as shown in the end of this appendix.

The projection onto $H$ is differentiable everywhere, as is clear from its explicit formula given in Lemma~\ref{lem:plane}.
Considering $L$ and $L_I$ for $I\subseteq\discint{1}{n}$, the projection is unique and can be cast as function $\R^n\to\R^n$ unless the point to be projected is equal to the barycenters $m$ and $m_I$, respectively.
By considering the explicit formulas given in Lemma~\ref{lem:sphere} and Lemma~\ref{lem:splxrcsn}, it is clear that these functions are differentiable as composition of differentiable functions.
Thus only the projection onto the simplex $C$ demands attention.
Note that the number $\hat{t}$ from Proposition~\ref{prop:projsplx} is equal to the mean value of the entries of the argument, that survive the projection, modulo an additive constant:
\begin{proposition}
\label{prop:splx_proj_subset}
Let $x\in\R^n\setminus C$ and $p := \proj_C(x)$.
Then there is a set $I\subseteq\discint{1}{n}$ such that $p = \max(x - \hat{t}\cdot e,\ 0)$ where $\hat{t} = \nicefrac{1}{\abs{I}}\cdot\left(\sum_{i\in I}x_i - \lambda_1\right)$.
\end{proposition}
\begin{proof}
Follows directly from Proposition~\ref{prop:projsplx} and Algorithm~\ref{alg:projsplx_general} by undoing the permutation $\tau$.
\end{proof}
Note that when $I = \discint{1}{n}$, this is very similar to the projection onto $H$, see Lemma~\ref{lem:plane}.
The next result states a condition under which $I$ is locally constant, and hence identifies points where the projection onto $C$ is differentiable with a closed form expression:
\begin{lemma}
\label{lem:splx_proj_analytics}
Let $x\in\R^n\setminus C$, and let $p := \proj_C(x)$, $I\subseteq\discint{1}{n}$ and $\hat{t}\in\R$ be given as in Proposition~\ref{prop:splx_proj_subset}.
When $x_i\neq\hat{t}$ for all $i\inint{1}{n}$, then the following holds where $u := \sum_{i\in I}e_i\in\R^n$ and $v := e - u\in\R^n$ are the indicator vectors of $I$ and $I^C$, respectively:
\begin{enumerate}
\item \label{lem:splx_proj_analytics_a}
$p_i > 0$ if and only if $i\in I$.

\item \label{lem:splx_proj_analytics_b}
$p = x + \nicefrac{1}{d}\cdot\left(\lambda_1 - u\transp x\right)u - v\hada x$, where $\hada$ denotes the Hadamard product and $d := \abs{I}$.

\item \label{lem:splx_proj_analytics_c}
There exists a constant $\epsilon > 0$ and a neighborhood $U := \set{s\in\R^n | \norm{x - s}_2 < \epsilon}$ of $x$, such that $\sgn(\proj_C(s)) = \sgn(p)$ for all $s\in U$.

\item \label{lem:splx_proj_analytics_d}
$\proj_C(s) = s + \nicefrac{1}{d}\cdot\left(\lambda_1 - u\transp s\right)u - v\hada s$ for all $s\in U$.

\item \label{lem:splx_proj_analytics_e}
$s\mapsto\proj_C(s)$ is differentiable in $x$.
\end{enumerate}
\end{lemma}
\begin{proof}
Only a sketch of a proof is presented here.

\ref{lem:splx_proj_analytics_a}
Follows from the characterization of $\hat{t}$ given in Proposition~\ref{prop:splx_proj_subset}.

\ref{lem:splx_proj_analytics_b}
The identity can be validated directly using~\ref{lem:splx_proj_analytics_a} and Proposition~\ref{prop:splx_proj_subset}.

\ref{lem:splx_proj_analytics_c}
Follows by choosing $\epsilon := \nicefrac{1}{2}\cdot\min_{i\inint{1}{n}}\abs{x_i - \hat{t}}$, which is positive by requirement on $x$.

\ref{lem:splx_proj_analytics_d}
Validation follows like in~\ref{lem:splx_proj_analytics_b} using~\ref{lem:splx_proj_analytics_c}.

\ref{lem:splx_proj_analytics_e}
The projection onto $C$ can be written locally in closed form using~\ref{lem:splx_proj_analytics_d}.
In the same neighborhood, the index set of vanishing entries of the projected points does not change.
Hence, the projection is differentiable as a composition of differentiable functions.
\end{proof}
It is clear that there are points in which $s\mapsto\proj_C(s)$ is continuous but not differentiable, for example points that are projected onto one of the vertices of $C$.
The structure in this situation is locally equivalent to that of the absolute value function.
However, for every point where the projection onto $C$ is not differentiable, a subtle change is sufficient to find a point where the projection is differentiable:
\begin{lemma}
\label{lem:splx_proj_critical}
Consider the function $p\colon\R^n\setminus C\to C$, $s\mapsto\proj_C(s)$ and let $x\in\R^n\setminus C$ be a point such that $p$ is not differentiable in $x$.
Then for all $\epsilon > 0$ there exists a point $y\in\R^n$ with $\norm{x - y}_2 < \epsilon$ such that $p$ is differentiable in $y$.
\end{lemma}
\begin{proof}
Let $\hat{t}\in\R$ be the separator from Proposition~\ref{prop:splx_proj_subset} for the projection onto $C$.
Let the index set of all collisions with $\hat{t}$ be denoted by $J := \set{j\inint{1}{n} | x_j = \hat{t}}$, which is nonempty with Lemma~\ref{lem:splx_proj_analytics} because $p$ is not differentiable in $x$.
Define $\delta := \nicefrac{\epsilon}{\sqrt{4\abs{J}}} > 0$ and consider $y := x - \delta\sum_{j\in J}e_j\in\R^n$.
Clearly, $\norm{x - y}_2 = \nicefrac{\epsilon}{2}$.
From Proposition~\ref{prop:splx_proj_subset} follows that the separating $\hat{t}$ for the projection onto $C$ is independent of the entries of $x$ with indices in $J$, as long as they are less than or equal to $\hat{t}$.
Because $\delta > 0$, these entries in $y$ are strictly smaller than $\hat{t}$, hence $p$ is differentiable in $y$ with Lemma~\ref{lem:splx_proj_analytics}.
\end{proof}
Therefore the set on which $s\mapsto\proj_C(s)$ is not differentiable forms a null set.
The next result gathers the gradients of the individual projections involved in the computation of the sparseness-enforcing projection operator with respect to $\sigma$.
Using the chain rule, the gradient of $\pi_{\geq 0}$ can be derived afterwards as multiplication of the individual gradients.

\begin{lemma}
\label{lem:projfuncgrad}
The individual projections for $\pi_{\geq 0}$ are differentiable almost everywhere. Their gradients are given as follows:
\begin{enumerate}
\item \label{lem:projfuncgrad_a}
$\frac{\partial\proj_H(x)}{\partial x} = E_n - \nicefrac{1}{n}\cdot ee\transp$, where $E_n\in\R^{n\times n}$ is the identity matrix.

\item \label{lem:projfuncgrad_b}
$\frac{\partial\proj_L(x)}{\partial x} = \frac{\sqrt{\rho}}{\norm{x-m}_2}\left(E_n - \nicefrac{1}{\norm{x-m}_2^2}\cdot (x-m)(x-m)\transp\right)$.

\item \label{lem:projfuncgrad_c}
$\frac{\partial\proj_C(x)}{\partial x} = E_n - \nicefrac{1}{d}\cdot uu\transp - \diag(v)$.
Here, $I := \set{i\inint{1}{n} | e_i\transp\proj_C(x) \neq 0}$ is the index set of nonzero entries of the projection onto $C$, $d := \abs{I}$, $u := \sum_{i\in I}e_i\in\R^n$ and $v := e - u\in\R^n$.

\item \label{lem:projfuncgrad_d}
$\frac{\partial\proj_{L_I}(x)}{\partial x} = \frac{\sqrt{\rho_I}}{\norm{x-m_I}_2}\left(E_n - \nicefrac{1}{\norm{x-m_I}_2^2}\cdot (x-m_I)(x-m_I)\transp\right)$.
\end{enumerate}
\end{lemma}
\begin{proof}
\ref{lem:projfuncgrad_a}
Follows from the closed form expression in Lemma~\ref{lem:plane}.

\ref{lem:projfuncgrad_b}
Lemma~\ref{lem:sphere} yields $\proj_L(x) = m + \delta(x)\cdot(x-m)$ with $\delta(x) = \nicefrac{\sqrt{\rho}}{\norm{x - m}_2}$.
With the quotient rule follows $\nicefrac{\partial\delta(x)}{\partial x} = -\nicefrac{\sqrt{\rho}}{\norm{x-m}_2^3}\cdot\left(x - m\right)\transp$, as $\rho$ does not depend on $x$.
The claim then follows by application of the product rule.

\ref{lem:projfuncgrad_c}
Follows from Lemma~\ref{lem:splx_proj_analytics}, similar to~\ref{lem:projfuncgrad_a}, using $v\hada x = \diag(v)x$.

\ref{lem:projfuncgrad_d}
Follows exactly as in~\ref{lem:projfuncgrad_b}.
\end{proof}
Clearly, the gradients for $\proj_H$ and for $\proj_{L}$ are special cases of the gradients of $\proj_C$ and $\proj_{L_I}$, respectively.
Therefore, they need no separate handling in the computation of the overall gradient.
Exploiting the special structure of the matrices involved and the readily sorted input as in Algorithm~\ref{alg:projfunc_explicit}, the gradient computation can be further optimized.
For the remainder of this appendix, let $\0_{a\times b}\in\set{0}^{a\times b}$ and $J_{a\times b}\in\set{1}^{a\times b}$ denote the matrices with $a$ rows and $b$ columns where all entries equal zero and unity, respectively.

\begin{theorem}
\label{thm:projfuncblockgrad}
Let $x\in\R^n$ be sorted in descending order and $\pi_{\geq 0}$ be differentiable in $x$.
Let $h\in\N$ denote the number of iterations Algorithm~\ref{alg:projfunc_explicit} needs to terminate.
In every iteration of the algorithm, store the following values for $i\inint{1}{h}$, where line numbers reference Algorithm~\ref{alg:projfunc_explicit}:
\begin{itemize}
  \item $d_i\in\N$ denoting the current dimensionality as determined by lines~\ref{algl:projfunc_explicit_d_init} and~\ref{algl:projfunc_explicit_projCcall}.
  \item $\delta_i := \sqrt{\nicefrac{\rho}{\varphi}}\in\R$ where $\rho$ and $\varphi$ are determined in lines~\ref{algl:projfunc_explicit_projL_rho} and~\ref{algl:projfunc_explicit_projL_phi}, respectively.
  \item $r(i) := y - m_I\in\R^{d_i}$ as computed in line~\ref{algl:projfunc_explicit_projL_r}.
\end{itemize}
Let $N := d_h = \norm{\pi_{\geq 0}(x)}_0$ denote the number of nonzero entries in the projection onto $D$, and define $s(i) := \tdvect{e_1\transp r(i)}{e_N\transp r(i)}\in\R^N$ as the first $N$ entries of each $r(i)$.
For $i\inint{1}{h}$ let
\begin{displaymath}
  A_i :=
  \delta_iE_N - \nicefrac{\delta_i}{d_i}\cdot J_{N\times N} -\alpha_is(i)s(i)\transp + \nicefrac{\alpha_i}{d_i}\cdot s(i)s(i)\transp J_{N\times N}\in\R^{N\times N}
  \text{ where }\alpha_i := \nicefrac{\delta_i}{\norm{r(i)}_2^2}\text{,}
\end{displaymath}
and let $A := \prod^{i=h}_1 A_i = A_h\cdots A_1\in\R^{N\times N}$.
Then the gradient of $\pi_{\geq 0}$ in $x$ is $\diag\big(A,\ \0_{(n-N)\times(n-N)}\big)$, that is a block diagonal matrix where the quadratic submatrix with row and column indices from $1$ to $N$ is given by $A$, and where all other entries vanish.
\end{theorem}
\begin{proof}
The gradient of projections onto $H$ is merely a special case of projections onto $C$, which also applies to the respective projections onto $L$ and $L_I$, see Lemma~\ref{lem:projfuncgrad}.
Hence, the very first iteration is a special case of iterations with $i > 1$.
Consider one single iteration $i\inint{1}{h}$ of Algorithm~\ref{alg:projfunc_explicit}, that is the computation of $\proj_{L_I}\circ\proj_C$ for some $I\subseteq\discint{1}{n}$.
Write $d := d_i$, $\delta := \delta_i$, $\alpha := \alpha_i$ and $r := r(i)$ for short.
Because the input vector $x$ is sorted by requirement, all intermediate vectors that are projected are sorted as well using Lemma~\ref{lem:proj_props}.
Thus $I = \discint{1}{d}$ holds.

With Lemma~\ref{lem:projfuncgrad}, the gradient $G_C\in\R^{n\times n}$ of the projection onto both $H$ and $C$ is of the form $G_C := E_n - \nicefrac{1}{d}\cdot uu\transp - \diag(v)$, where $u := \sum_{i=1}^de_i$ and $v := e - u$.
Let $q := \left(r_1,\dotsc,r_d,0,\dotsc,0\right)\transp\in\R^n$ be a copy of $r$ padded with zeros to achieve full dimensionality $n$.
The gradient of the projection onto $L$, and in general $L_I$, is given by $G_L := \delta E_n - \alpha qq\in\R^{n\times n}$ using Lemma~\ref{lem:projfuncgrad}.
The gradient of the whole iteration is then given by the chain rule, yielding
\begin{displaymath}
  G := G_LG_C = \delta E_n - \nicefrac{\delta}{d}\cdot uu\transp - \delta\diag(v) - \alpha qq\transp + \nicefrac{\alpha}{d}\cdot qq\transp uu\transp + \alpha qq\transp\diag(v)\in\R^{n\times n}\text{.}
\end{displaymath}
Write $\0 := \0_{(n-d)\times(n-d)}$, then $G$ is a block diagonal matrix of a matrix from $\R^{d\times d}$ and $\0$:
Note that $E_n - \diag(v) = \diag(E_d, \0)$, $uu\transp = \diag(J_{d\times d}, \0)$, and $qq\transp = \diag(rr\transp, \0)$.
Therefore, $qq\transp\diag(v) = \diag(rr\transp, \0)\cdot\big(\begin{smallmatrix}0&0\\0&E_{n-d}\end{smallmatrix}\big) = 0$, and $qq\transp uu\transp = \diag(rr\transp J_{d\times d}, \0)$.
Thus
\begin{align*}
  G &= \delta \diag(E_d, \0) - \nicefrac{\delta}{d}\cdot\diag(J_{d\times d}, \0) - \alpha\diag(rr\transp, \0) + \nicefrac{\alpha}{d}\cdot\diag(rr\transp J_{d\times d}, \0)\\
  &= \diag\left(\delta E_d - \nicefrac{\delta}{d}\cdot J_{d\times d} - \alpha rr\transp + \nicefrac{\alpha}{d}\cdot rr\transp J_{d\times d},\ \0\right)\text{.}
\end{align*}
By denoting the gradient of iteration $i$ by matrix $G_i\in\R^{n\times n}$ for $i\inint{1}{h}$ and by application of the chain rule follows that the gradient of all iterations is given by $\prod^{i=h}_1 G_i$.
In this matrix, all entries but the top left submatrix of dimensionality $N \times N$ are vanishing, where $N = d_h$.
This is because the according statement applies to $G_i$ for the top left submatrix of dimensionality $d_i\times d_i$, and $d_1 > \cdots > d_h$ holds, and only the according entries survive the matrix multiplication.
Therefore it is sufficient to compute only the top left $N\times N$ entries of the gradients of the individual iterations, as the remaining entries are not relevant for the final gradient.
This is reflected by the definition of the matrices $A_i$ for $i\inint{1}{h}$ from the claim.
\end{proof}
The gradient can thus be computed using matrix-matrix multiplications, where the matrices are square and the edge length is the number of nonzero entries in the result of the projection.
This computation is more efficient than using the $n\times n$ matrices of the individual projections.
However, when the target degree of sparseness is low, and thus the amount of nonzero entries $N$ in the result of the projection is large, gradient computation can become very inefficient.
In practice, often only the product of the gradient with an arbitrary vector is required.
In this case, the procedure can be sped up by exploiting the special structure of the gradient of $\pi_{\geq 0}$:
\begin{corollary}
\label{cor:projfuncgraddgemv}
Let $x\in\R^n$ be sorted in descending order and $\pi_{\geq 0}$ be differentiable in $x$.
The product of the gradient of $\pi_{\geq 0}$ in $x$ with an arbitrary vector can be computed using vector operations only.
\end{corollary}
\begin{proof}
Note that because of the associativity of the matrix product it is enough to consider the product of the gradient $G\in\R^{n\times n}$ of one iteration of Algorithm~\ref{alg:projfunc_explicit} with one vector $y\in\R^n$.
Because of the statements of Theorem~\ref{thm:projfuncblockgrad}, it suffices to consider the top left $N\times N$ entries of $G$ and the first $N$ entries of $y$, as all other entries vanish.
Therefore let $A := \delta E_N - \nicefrac{\delta}{d}\cdot J_{N\times N} -\alpha s s\transp + \nicefrac{\alpha}{d}\cdot s s\transp J_{N\times N}\in\R^{N\times N}$ be the non-vanishing block of $G$ as given by Theorem~\ref{thm:projfuncblockgrad}, let $u := J_{N\times 1}\in\R^N$ be the vector of ones such that $u u\transp = J_{N\times N}$, and let $z := \tdvect{y_1}{y_N}\in\R^N$ denote the vector with the first entries of $y$.
Using matrix product associativity and distributivity over multiplication with a scalar yields
\begin{displaymath}
  Az = \delta\left(z - \nicefrac{1}{d}\cdot\scp{z}{u} \cdot u\right) + \alpha\left(\nicefrac{1}{d}\cdot\scp{s}{u}\scp{z}{u}- \scp{s}{z}\right) s\text{,}
\end{displaymath}
where $\scp{z}{u} = \sum_{i=1}^N z_i$ and $\scp{s}{u} = \sum_{i=1}^N s_i$.
Hence $Az$ can be computed in-place from $z$ by subtraction of a scalar value from all entries, rescaling by $\delta$, and adding a scaled version of vector $s$.
\end{proof}
Although in Theorem~\ref{thm:projfuncblockgrad} and Corollary~\ref{cor:projfuncgraddgemv} it was necessary that the input vector is sorted, the general case can easily be recovered:
\begin{proposition}
\label{prop:projgrad_sorted}
Let $x\in\R^n$ be a point, $\tau\in S_n$ such that $y := P_\tau x\in\R^n$ is sorted in descending order and $\pi_{\geq 0}$ be differentiable in $y$ with gradient $G\in\R^{n\times n}$.
Then $\pi_{\geq 0}$ is also differentiable in $x$, and the gradient is $P_\tau\transp GP_\tau$.
\end{proposition}
\begin{proof}
Follows with $P_{\tau}\transp = P_{\tau^{-1}} = P_{\tau}^{-1}$, $\pi_{\geq 0}(x) = P_\tau\transp \pi_{\geq 0}(P_\tau x) = P_\tau \pi_{\geq 0}(y)$ and the chain rule.
\end{proof}
Likewise, the gradient for the unrestricted projection $\pi$ can be computed from the gradient for $\pi_{\geq 0}$:
\begin{proposition}
\label{prop:projgrad_reflective}
Let $x\in\R^n$ be a point such that $\pi_{\geq 0}$ is differentiable in $\abs{x}$ with gradient $G\in\R^{n\times n}$.
Let $s\in\set{\pm 1}^n$ be given such that $\pi(x) = s\hada\pi_{\geq 0}\left(\abs{x}\right)$.
Then $\pi$ is differentiable in $x$, and the gradient is $\diag(s) G \diag(s)$.
\end{proposition}
\begin{proof}
Follows analogously to Proposition~\ref{prop:projgrad_sorted}, using $\abs{x} = s\hada x = \diag(s)x$.
\end{proof}
Summing up, the gradient of the projection onto $S_{\geq 0}^{(\lambda_1,\lambda_2)}$ and $S^{(\lambda_1,\lambda_2)}$ can be computed efficiently by bookkeeping a few values as discussed in Theorem~\ref{thm:projfuncblockgrad}, and applying simple operations to recover the general case.
When only the product of the gradient with a vector is required, the computation can be made more efficient as stated in Corollary~\ref{cor:projfuncgraddgemv}.
Direct application of Theorem~\ref{thm:projfuncblockgrad} should be avoided in this situation because of the high computational complexity.

\section{Gradients for SOAE Learning}
\label{sect:soae_gradients}
The objective function $E_{\SOAE}$ is a convex combination of two similarity measures $s_R$ and $s_C$.
The degrees of freedom $W$, $W_{\out}$ and $\theta_{\out}$ of the SOAE architecture should be tuned by gradient-based methods to minimize these functions.
This appendix reports the gradient information needed for reproduction of the experiments.
The first statement addresses the reconstruction module.
\begin{proposition}
\label{prop:SOAE-grad-reconst}
It is $\left(\nicefrac{\partial s_R\left(\tilde{x},\ x\right)}{\partial W}\right)\transp = xgW f'(u) + g\transp h\transp\in\R^{d\times n}$ where $g := \nicefrac{\partial s_R(\tilde{x}, x)}{\partial \tilde{x}}\in\R^{1\times d}$ is the gradient of the similarity measure with respect to its first argument.
Additionally, $\left(\nicefrac{\partial s_R\left(\tilde{x},\ x\right)}{\partial W_{\out}}\right)\transp = 0\in\R^{n\times c}$ and $\nicefrac{\partial s_R\left(\tilde{x},\ x\right)}{\partial \theta_{\out}} = 0\in\R^{1\times c}$.
\end{proposition}
\begin{proof}
As $s_R$ does not depend on $W_{\out}$ or $\theta_{\out}$, the respective gradients vanish.
The symmetry between encoding and decoding yields $\tilde{x} := W\cdot f\left(W\transp x\right)$.
The gradient for $W$ follows using the chain rule and the product rule for matrix calculus, see \citet{Neudecker1969} and \citet{Vetter1970}.
\end{proof}
The correlation coefficient is the recommended choice for the similarity measure of the reconstruction module because it is normed and invariant to affine-linear transformations.
It is also differentiable almost everywhere:
\begin{proposition}
\label{prop:corrcoeff_gradient}
If $s_R$ is the correlation coefficient and $x,\tilde{x}\in\R^d\setminus\set{0}$, then
\begin{displaymath}
  \left(\tfrac{\partial s_R(\tilde{x}, x)}{\partial \tilde{x}}\right)\transp
  = \tfrac{1}{\sqrt{\lambda\mu}}\left(x - \tfrac{\scp{e}{x}}{d}e\right) - \tfrac{s_R(\tilde{x}, x)}{\lambda}\left(\tilde{x} - \tfrac{\scp{e}{\tilde{x}}}{d}e\right)\in\R^d\text{,}
\end{displaymath}
where all entries of $e\in\R^d$ are unity, $\lambda := \norm{\tilde{x}}_2^2 - \nicefrac{1}{d}\cdot\scp{e}{\tilde{x}}^2\in\R$ and $\mu := \norm{x}_2^2 - \nicefrac{1}{d}\cdot\scp{e}{x}^2\in\R$.
\end{proposition}
\begin{proof}
One obtains $\sqrt{\lambda\mu} \cdot s_R(\tilde{x}, x) = \scp{x}{\tilde{x}} - \nicefrac{1}{d}\cdot\scp{e}{\tilde{x}}\scp{e}{x}$ because $s_R$ is the correlation coefficient.
The claim then follows with the quotient rule.
\end{proof}
The gradients of the similarity measure for classification capabilities are essentially equal to those of an ordinary two-layer neural network, and can be computed using the back-propagation algorithm \citep{Rumelhart1986}.
However, the pairing of the softmax transfer function with the cross-entropy error function provides a particularly simple structure of the gradient \citep{Dunne1997}.
For completeness, the gradients of the classification module of SOAE are summarized:
\begin{proposition}
\label{prop:SOAE-grad-classf}
If $s_C$ is the cross-entropy error function, $g$ is the softmax transfer function and the target vector for classification $t$ is a one-of-$c$ code, then
$\left(\nicefrac{\partial s_C\left(y,\ t\right)}{\partial W_{\out}}\right)\transp = h\cdot (y - t)\transp\in\R^{n\times c}$,
$\nicefrac{\partial s_C\left(y,\ t\right)}{\partial \theta_{\out}} = (y - t)\transp\in\R^{1\times c}$ and
$\left(\nicefrac{\partial s_C\left(y,\ t\right)}{\partial W}\right)\transp = x \cdot \left( (y - t)\transp W_{\out}\transp f'(u)\right)\in\R^{d\times n}$.
\end{proposition}
\begin{proof}
Basic matrix calculus \citep{Neudecker1969,Vetter1970} yields
$\nicefrac{\partial s_C\left(y,\ t\right)}{\partial \theta_{\out}} = \left(\nicefrac{\partial s_C\left(y,\ t\right)}{\partial y}\right)\cdot g'(y)$,
$\left(\nicefrac{\partial s_C\left(y,\ t\right)}{\partial W_{\out}}\right)\transp = h\cdot \left(\nicefrac{\partial s_C\left(y,\ t\right)}{\partial \theta_{\out}}\right)$ and
$\left(\nicefrac{\partial s_C\left(y,\ t\right)}{\partial W}\right)\transp = x \cdot \left(\left(\nicefrac{\partial s_C\left(y,\ t\right)}{\partial \theta_{\out}}\right)\cdot W_{\out}\transp f'(u)\right)$.
By requirement $\nicefrac{\partial s_C\left(y,\ t\right)}{\partial y} = -(t\hadadiff y)\transp$, where $\hadadiff$ denotes the element-wise quotient, $g'(y) = \diag(y) - yy\transp$ and $\sum_{i=1}^c t_i = 1$.
Therefore $\left(\nicefrac{\partial s_C\left(y,\ t\right)}{\partial \theta_{\out}}\right)\transp = \left(yy\transp - \diag(y)\right)\cdot (t\hadadiff y) = y\cdot\scp{y}{t\hadadiff y} - y\hada t\hadadiff y = y - t$ using $\scp{y}{t\hadadiff y} = \sum\nolimits_{i=1}^c y_i\cdot\nicefrac{t_i}{y_i} = 1$, and the claim follows.
\end{proof}
As $E_{\SOAE}$ is a convex combination of the reconstruction error and the classification error, its overall gradient follows immediately from Proposition~\ref{prop:SOAE-grad-reconst} and Proposition~\ref{prop:SOAE-grad-classf}.
Proposition~\ref{prop:corrcoeff_gradient}, the results from Appendix~\ref{sect:analytical_properties}, and the gradient of the $L_0$ projection as described in Section~\ref{sect:projfunc_differentiability} can then be used to compute the explicit gradients for the procedure proposed in this paper.

\bibliography{the}
\end{document}